\pdfoutput=1
\documentclass[twoside,11pt]{article}

\usepackage{jmlr2e}
\usepackage{cs}
\usepackage{graphicx}
\usepackage{amsmath,amssymb} 
\usepackage{color}
\usepackage{eucal}
\usepackage[tight,rm]{subfigure}

\usepackage{mathsymb}
\usepackage{verbatim}
\usepackage{subfigure}
\usepackage{url}
\graphicspath{{./}{./Fig/}{./Figs/}{./CVPR2011/figs/}{./CVPR2011/caltech/}{./CVPR2012/Figs/}}

\usepackage{algorithm2e}
\let\savedalgorithm\algorithm
\let\savedendalgorithm\endalgorithm
\newenvironment{algorithmic}{%

\savedalgorithm
}{%
\savedendalgorithm
}

\usepackage{algorithm}

\newtheorem{thm}{Theorem}      

\usepackage{amsmath}
\usepackage{url}
\usepackage{multirow}
\usepackage{pifont}

\def\x{{\boldsymbol x}}
\def\w{{\boldsymbol w}}
\def\yi{{  y_i  }}
\newcommand{\mbh}{ {Multi\-Boost}-\-hinge\xspace}
\newcommand{\mbe}{ {Multi\-Boost}-\-exp\xspace}
\renewcommand{\fnorm}[2][2]{\ensuremath{ \left\| #2 \right\|_{ \mathrm{#1} } } }

\def\multiboost{{\rm MultiBoost$^{\,\ell_1}$}\xspace}
\def\multistruct{{\rm MultiBoost$^{\,\rm group}$}\xspace}
\def\multistructonevsall{{\rm MultiBoost$^{\rm group}_{\text{\sc fast} }$}}

\def\ova{{\sc fast}\xspace}

\def\ADMM{{ADMM}\xspace}

\def\yi{{y_i}}

\def\bxi{{\bx_i}}

\def\brho{{\boldsymbol \rho}}

\newcommand{\CScomment}[1]{}

\jmlrheading{09}{2012}{}{}{}
{Chunhua Shen, 
Sakrapee Paisitkriangkrai,
and Anton van den Hengel}

\ShortHeadings{A direct approach to multi-class boosting}{Shen, Paisitkriangkrai, and van~den~Hengel}
\firstpageno{1}

\begin{document}

\title{A Direct Approach to Multi-class Boosting and Extensions}

\author{\name Chunhua Shen \email chunhua.shen@adelaide.edu.au
       \\
       \addr
       The University of Adelaide\\
       Adelaide, SA 5005, Australia
       \AND
       \name Sakrapee Paisitkriangkrai \email paul.paisitkriangkrai@adelaide.edu.au
       \\
       \addr
       The University of Adelaide\\
       Adelaide, SA 5005, Australia
       \AND
       \name Anton van den Hengel \email  anton.vandenhengel@adelaide.edu.au
       \\
       \addr
       The University of Adelaide\\
       Adelaide, SA 5005, Australia}

\editor{Oct. 2012}

\maketitle

\begin{abstract} 

        Boosting methods combine a set of moderately accurate weak
        learners to form a highly accurate predictor.  Despite the
        practical importance of multi-class boosting,  it has received
        far less attention than its binary counterpart.
        In this work, we propose a fully-corrective multi-class
        boosting formulation which directly solves the multi-class
        problem without dividing it into multiple binary
        classification problems.  
        In contrast, most previous multi-class boosting algorithms
        decompose a multi-boost problem into multiple 
        binary boosting problems.  
        By explicitly deriving the Lagrange dual of the primal
        optimization problem, we are able to construct a column
        generation-based fully-corrective approach to boosting which
        directly optimizes multi-class classification performance.
        The new approach not only updates all weak learners'
        coefficients at every iteration, but does so in a manner
        flexible enough to accommodate various loss functions and
        regularizations. For example, it enables us to introduce
        structural sparsity through mixed-norm regularization to
        promote group sparsity and feature sharing.  Boosting with
        shared features is particularly beneficial  in complex
        prediction problems where features can be expensive to
        compute.  Our experiments on various data sets demonstrate
        that our direct multi-class boosting generalizes as well as,
        or better than, a range of competing multi-class boosting
        methods.  The end result is a highly effective and compact
        ensemble classifier which can be trained in a distributed
        fashion.

\end{abstract}


\section{Introduction}
\label{sec:intro}

            A significant proportion of the most important practical classification problems
            inherently involve making a selection between a large number of classes.
            Such problems demand effective and efficient multi-class
            classification techniques.
            Unlike binary classification, which has been well researched,
            multi-class classification
            has received relatively little attention due to the inherent
            complexity of the problem.
            Some important steps have been 
            (see  \cite{Wu2004Probability, Crammer2001Algorithmic,
            Guruswami1999Multiclass} for instance),  but the primary approach
            thus far has exploited large numbers of independent
            binary classifiers.
            An example of this approach is the extension of a binary
            classification algorithm to the multi-class case by considering
            the problem as a set of one-vs-all binary classification
            problems.

            Boosting has recently attracted much research interest in
            many scientific fields due to its huge success in
            classification and regression tasks, especially
            in the first real-time face detection application \citep{Viola2004}.
            Both theoretical and empirical results show that
            boosting methods have competitive generalization
            performance compared with
            many existing classifiers in the literature.
            To explain why boosting works,
            \cite{Schapire98}
            introduced 
            an appropriate
            margin theory, which was inspired by the
            margin theory in support vector machines,
            and concluded that boosting is also an effective classifier
            which maximizes the minimum margin over the training data.
            Extending 
            this idea, LPBoost~\citep{Demiriz2002LPBoost} seeks to
            maximize the relaxed minimum margin (soft margin) using hinge loss.
            The proposed boosting algorithm is fully corrective in the
            sense that all the coefficients of learned weak
            classifiers are updated at each iteration.  Such
            fully-corrective boosting algorithms typically require
            fewer iterations to achieve convergence.
            
            Despite the significant attention that boosting-based
            binary classification methods have attracted, multi-class
            boosting has been much less well studied.
            As with multi-class classification in general, the
            most natural strategy for multi-class boosting is to
            partition the problem into a set of independent binary classification
            problems. In this scenario each binary classifier is charged with distinguishing 
            a subset of the classes against all others.  Methods such as one-vs-all, all-vs-all and output
            code-based methods belong
            to this category.
						Although such partitioning strategies greatly simplify the problem, 
						they inevitably impact upon  the final solution.  In many cases the 
						partitioning strategy changes the cost function to be optimized, 
						and thus delivers a sub-optimal solution. 
                        The all-vs-all approach, {\em a.k.a.}\
                        one-vs-one, however, has been shown to achieve
                        excellent classification accuracy. 
                        In this approach $ k (k-1) /2 $ two-way
                        pairwise classifiers are trained, with $ k$
                        the number of classes. The computation of both
                        training and testing can be 
                        prohibitively expensive even when $ k $ is
                        of medium size.   
                        More importantly, however, 
                        almost all of these strategies do not directly
                        optimize the multi-class decision function
                        that they seek to exploit.

            In this work, we proffer a direct approach to 
            fully-corrective multi-class boosting.
            In order to achieve this result, we
            generalize the concept of 
            the separating 
            hyperplane and margin in binary boosting
            to multi-class problems. 
            This allows the development of a single, fully-corrective, multi-class boosting 
            classifier which directly optimizes multi-class classification performance.
            Similar ideas have been used in multi-class support vector machines
            \citep{Crammer2001Algorithmic,Weston1999SVM,Elisseeff01akernel}.
            To our knowledge, it has not been employed to design {\it
            fully-corrective}
            multi-class boosting.
            As shown in \citep{Shen2010PAMI} fully-corrective boosting
            in general leads to more compact models. Here for the
            first time, we develop fully-corrective multi-class
            boosting.  
            
            In deriving out direct formulation we also 
            generalize the fully-corrective $\ell_1$ regularized
            boosting algorithms to arbitrary mixed-norm regularization terms.
            Mixed-norm regularization, also known as group sparsity,
            has been used when there exists a structure that separates the
            model into disjoint groups of parameters.
            For $\ell_{1,2}$-norm regularized boosting, for example, 
            each such group of parameters
            is subject to a common 
            $\ell_2$-norm regularizer.
            The key intuition behind structural sparsity is that
            informative features are commonly shared between multiple classes.
            For example, traffic warning signs have a common triangular
            shape with various symbols inside.
            These basic shared features should be used to help differentiate warning
            signs from other traffic signs while the symbols inside can be used to
            differentiate different warning signs.
            In this work, we aim to 
            enable the selection of a 
            common subset of features
            which are informative in identifying a wide range of classes.

            The key idea behind our column generation-based boosting
            approach is that, given 
            an example ${\boldsymbol x}$, with true label $y$, 
            the output of the decision function for the correct label
            must be larger than the output of
            the decision function for all incorrect labels,
            \[ F_y ( \boldsymbol x ) > F_r ( \boldsymbol x ), \,\, \forall r \neq y.\]
            We then formulate a convex optimization problem, which 
            maximizes $ F_y ( \boldsymbol x ) - F_r ( \boldsymbol x ) $ 
            subject to the selected regularization term.
            This leads to a constrained semi-infinite convex optimization problem,
            which may have infinitely many variables. In order to design a boosting
            algorithm, we explicitly derive the Lagrange dual of this problem and
            apply an iterative convex optimization technique known as column generation.
            When the hinge loss is used, our formulation can be viewed as a direct
            extension of LPBoost \citep{Demiriz2002LPBoost} to the multi-class case.
            We also discuss the
            use of the exponential and logistic loss functions.
            In theory, any convex loss function can be employed, as in
            the binary classification case.
            Note that the AnyBoost framework of \cite{anyboost} can not be adopted here
            since AnyBoost cannot cope with multiple constraints.
            In summary, our main contributions are as follows.
             \begin{itemize}
                \item
                    We propose the first direct approach to fully-corrective
                    multi-class boosting based on the generalization of
                    the conventional   ``margin'' in binary classification.
                \item
                    Within this direct, fully corrective boosting framework, we design new
                    boosting methods that
                    promote feature sharing across classes by enforcing group
                    sparsity regularization
                    (referred to as \multistruct).
                    We empirically show
                    that by enforcing group sparsity, the proposed multi-class
                    boosting converges faster
                    while achieving better or comparable generalization performance.
                    The fact that the algorithm converges fast means that fewer
                    features are required for a given classification accuracy and
                    there is a significant improvement in run-time performance.
                    Our derivation for designing multi-class boosting methods
                    is applicable to {\em arbitrary} convex loss functions with
                    general $ \ell_{1,p}$ $(p \geq 1)$ mixed-norms.
                    {\em To our knowledge, this is the first
                    fully-corrective multi-class boosting approach that
                    promotes feature sharing using group sparsity regularization}.
                    Moreover, we propose the use of
                    the alternating direction method
                    of multipliers (\ADMM) \citep{Boyd2011Distributed} to
                    efficiently solve the involved optimization problems,
                    which is much faster than using standard
                    interior-point solvers.
                \item
                    Further, a new family of multi-class boosting algorithms
                    based on a simplified formulation is proposed in order to further reduce training times.
                    This new formulation not only enables us
                    to share features and encourages structural sparsity in the learning
                    procedure of multi-class boosting, but also allows us to take
                    advantage of parallelism in \ADMM to speed up the
                    training time by a factor
                    proportional to the number of classes.
                    The training time required is thus similar to that required to
                    train multiple independent binary classifiers in parallel.
                    The proposed formulation converges significantly faster, 
                    while still enforcing
                    group sparsity.
            \end{itemize}
            Since multi-class classification can be
            seen as an instance of structured learning problems of \cite{StructSVM},
            the proposed formulation may also be applicable to other structured
            prediction problems.

            We briefly review most relevant work on multi-class boosting before
            we present our algorithms.

            \subsection{Related work}

            AdaBoost, 
            proposed in 
            \citep{Freund1997},
            was the first practical {\em binary} boosting algorithm.
            One of the limitations of 
            binary AdaBoost is that
            each weak classifier's accuracy must be higher than $ 0.5 $.
            That is, a weak classifier must exhibit classification capability superior to that of random guessing.
            AdaBoost.M1, directly extended  AdaBoost to multi-class classification
            using multi-class weak classifiers.
            Multi-class weak classifiers, such as decision trees for example, represent a restricted set 
            of weak classifiers able to give predictions on all $ k $ possible labels at
            each call.
            The fact that only multi-class weak classifiers
            can be used represents a significant restriction, as multi-class weak classifiers
            are complicated and require time-consuming training when compared
            with their simple binary counterparts. 
            The higher complexity of the assembled classifier also
            implies a higher risk of
            over-fitting the training data.
            In addition, the requirement that a
            weak classifier's weighted error must be better than $0.5$ 
            can be hard to achieve for problems with many classes.
            Note that, for a problem with $ k $ classes, random guessing can only guarantee an accuracy
            of $ 1/ k $.
            
            The SAMME algorithm of
            \cite{Zhu2009Multi},
						addressed this last issue, and 
            requires only that the multi-class weak classifiers 
            achieve an error rate better than uniform random guessing for multiple labels
            ($ 1/ k $ for $ k $ labels).
            When $ k = 2$, SAMME reduces to the standard AdaBoost, but is still 
            subject to all of the other limitations associated with the use of multi-class weak classifiers.
            
            To alleviate these difficulties
            one solution is to decompose a multi-class boosting problem into
            a set of binary classification problems.
            To this end strategies such as 
            ``one-vs-all'' and ``one-vs-one''
            have been developed.  Such
            approaches can be viewed as special
            cases of error-correcting output coding (ECOC)
            \citep{ECOC_AI,crammer2002learnability}.
            By introducing a coding matrix, AdaBoost.MO
            \citep{Schapire1999Improved}
            is a typical
            example of ECOC based multi-class boosting.
            In this approach a set of binary classifiers is used, with each 
            trained so as to recognise a subset of the classes.  By comparison of the 
            responses of all of the binary classifiers multi-class classification is achieved.
            Algorithms in this category include AdaBoost.MO
            \citep{Schapire1999Improved}, AdaBoost.OC
            and AdaBoost.ECC \citep{Guruswami1999Multiclass}.
            AdaBoost.OC can be seen as a variant of AdaBoost.MO
            which also combines boosting
            and ECOC. However, unlike AdaBoost.MO, AdaBoost.OC
            uses a collection of randomly generated codewords. For
            more details 
            see~\citep{multiboost}.

            The attraction of 
            transforming
            a multi-class classification problem into
            a set of binary classification problems
            is
            that each of the weak classifiers
            need only be a simple binary classifier.
            This approach has its limitations, however, including the fact that 
            the required optimisation problem is typically compromised by the partition, 
            and that it becomes increasingly difficult to ensure that each binary 
            classifier sees a representative sample of the data as the number of classes increases.
            An additional limitation of all partitioning algorithms we have discussed is that
            they are incapable of effectively exploiting the inevitable similarity
            between classes, and thus to efficiently share features between classifiers.
            Since
            binary classifiers are trained independently, the resulting strong
            classifier can be highly unbalanced and often dependent on an
            excessive number of features/weak classifiers.

            Several approaches have been developed which aim to enable feature-sharing within
            multi-class boosting.
            JointBoost, proposed by
            \cite{Torralba2007Sharing},
            finds common features that can be
            shared across classes using heuristics.
            Weak learners are then trained jointly using standard boosting.
            In order to reduce the number of binary classifiers which need to be
            trained for multi-class problems, the authors
            proposed an approximate search procedure based on greedy forward
            selection.  The drawback of greedy approach, however, is that it is
            short-sighted and 
            cannot recover if an error is made.
            The fact that the weak learner selected at each boosting iteration cannot be guaranteed to be globally optimal
            means that the final ensemble is highly likely to be sub-optimal.
            Zhang \etal proposed training multi-class
            boosting with sharable information patterns
            \citep{Zhang2009Finding}.  As a pre-processing step,
            they generate sharable
            patterns using data mining techniques
            and then train a multi-class boosting-based classifier using these patterns.
            The process of identifying sharable features and the 
            training procedure are thus
            de-coupled, and therefore unlikely to reach the optimal solution.
            In comparison to JointBoost and Zhang \etal's work, the method we propose selects
            weak learners systematically on the basis of structural sparsity during 
            the training process and thus, at least asymptotically, will reach the globally optimal solution.

            A related approach,
            termed GradBoost \citep{Duchi2009Boosting}, 
            also exploits a mixed-norm in order to achieve group sparsity,
            but
            does not directly
            optimize the boosting objective function. Instead,
            the algorithm updates a block of variables for optimizing a
            quadratic surrogate function in a fashion
            similar to gradient-based coordinate descent.
            It is not clear how well the surrogate approximates
            the original objective function,
            and no proof is given.
            Since the mixed-norm regularization term is not directly optimized
            either, {group sparsity is achieved heuristically
            by a combination of forward selection and backward elimination}.
            Our work fundamentally differs from \citep{Duchi2009Boosting} in that
            we directly optimize the group sparsity regularized objective by
            following the column generation based boosting
            \citep{Shen2010PAMI} without deferring to heuristics. 

            Our work here can also be seen as an extension of the general binary
            fully-corrective
            boosting framework of \cite{Shen2010PAMI} to the multi-class case.
            As in~\citep{Shen2010PAMI}, we design a feature-sharing boosting method using a direct formulation, 
            but for multi-class problems and using a more sophisticated
            group sparsity regularization.
            Note that the general boosting framework of
            \cite{Shen2010PAMI} is not directly applicable
            in our problem setting.

\subsection{Notation}
				A bold lowercase letter ($ \boldsymbol u $) denotes
        a column vector, and an uppercase letter ($ U$) a matrix.
        $ \trace(U)$ represents the trace of a symmetric matrix.
        An element-wise inequality between two vectors or matrices such as
        ${\boldsymbol u} \geq {\boldsymbol v}$ implies that 
        $ u_i \geq  v_i $ for all $i$.

        Let $  ( \x_i; y_i ) \in \Real^d \times
        \{ 1,\ldots, k \}, i = 1\ldots m $,
        be a set of $ m $
        multi-class training examples, where $k$ denotes the number of classes.
        We denote by ${\cal H}$ a set of weak classifiers (or dictionary);
        note that the size of ${\cal H}$ can be infinite.
        Each $ h_j ( \cdot ) \in {\cal H}, j = 1 \dots n$,
        is a function that maps an input $ \x $ to
        $\{-1, +1 \} $. Although our discussion applies equally in the general case where
        $ h( \cdot) $ make take  any real value, we use binary weak classifiers
        in this work.
        The matrix $ H \in \Real^{m \times n } $ 
        captures
        the weak classifiers' responses to the whole of the training data;
        that is $ H_{ij} = h_j ( \x_i) $.
        Each column $ H_{:j} $ thus represents the output of the weak
        classifier $ h_j(\cdot) $ when applied to the entire training set and
        each row $ H_{ i:} $ the responses of all of the weak classifiers
        to the $ i$th training datum $ \x_i $.

        Boosting algorithms learn a strong classifier of the form
        $ F( \x ) = \sum_{ j = 1} ^ n w_j h_j ( \x )  $
        which is parameterized by a vector $ \w \in \Real^n $.
        In 
        our formulation of the problem
        we need to learn a classifier for
        each class. So for class $ r $ (where $r = 1, \ldots, k $), 
        the learned strong classifier is
        $ F_r( \x ) = \sum_{ j = 1} ^ n w_{r,j} h_j ( \x )  $
        and has  parameter vector $ \w_r $.
        We define $ W =  [ \w_1, \w_2, \ldots, \w_k ] \in \Real^{ n \times k}$
        and let
        $ \fnorm[1]{W} = \sum_{ij} | W_{ij} | $ represent the $ \ell_1 $ norm.
        The $ \ell_{1,2}$ norm of a matrix is defined as
        $ \| W \|_{1,2}  =  \sum_j \|  W_{j:} \|_2 $ with $ \| \cdot \|_2$ being
        the $ \ell_2 $ norm.
        The $ \ell_{1, \infty}$ norm of $ W $ is 
        $ \| W \|_{1, \infty}  =  \sum_j \max ( W_{j:} ) $.

        Here we assume that the weak classifier dictionary for each class
        is the same.
        The final strong classifier is a weighted average of multiple
        weak classifiers, and the estimated classification for a
        test datum $ \bx $ is
        $
        F(\bx) = \argmax_{ r = 1, \dots k } \, \sum_{j=1}^{n} w_{ r ,j} h_j(\bx).
        $

        The remaining content is structured as follows.
        Section \ref{sec:2} presents the main algorithm of our work.
        In particular, we beginning by deriving our algorithm with
        $\ell_1$ penalty for
        the piece-wise linear hinge loss and exponential loss functions.
        Then we discuss group sparsity and derive our algorithm
        with the new structural sparsity for both hinge loss
        and logistic loss.
        We present our experimental results
        in Section \ref{sec:exp} and conclude 
        in Section \ref{sec:conc}.

\section{A direct formulation for multi-class boosting}
\label{sec:2}

        In binary classification, the margin is defined as
        $  y F ( \boldsymbol x )   $ with $ y \in \{-1, +1 \} $.
        In the framework of maximum margin learning,
        one tries to maximize the margin $  y F ( \boldsymbol x )   $
        as much as possible. A large margin implies the learned classifier
        confidently classifies the corresponding
        training example.
        We show how  this idea can be generalized to multi-class
        problems in this section.

\subsection{MultiBoost with $\ell_1$-norm regularization (\multiboost) }

\paragraph{The hinge loss}
        Let us consider the hinge loss case, which is piecewise linear and
        therefore makes it easy to derive our formulation. As we will show,
        both the primal and dual problems are linear programs (LPs), which
        can be globally solved in polynomial time.
        The basic idea is to learn classifiers by pairwise comparison.
        For a training example $ (\x, y) $, if we have a perfect classification
        rule, then the following holds
        \[
            F_y(\x) > F_r(\x), \text{ for any } r \neq y.
        \]
        In the large margin framework with the hinge loss,
        ideally
        \begin{equation}
            \label{EQCVPR11:1}
             F_y(\x) \geq 1 + F_r(\x), \text{ for any } r \neq y,
        \end{equation}
        should be satisfied.
        This means that the correct label is supposed to have
        a classification confidence that is larger by at least a unit than
        any of the confidences for the other predictions.
        This extension of ``margin'' to the multi-class case
        has been introduced in support vector machines
        \citep{Weston1999SVM,Elisseeff01akernel}.
        As pointed out in \cite{Weston1999SVM},
        to formulate multi-class problems as a pairwise ranking
        problem in a single optimization can be more powerful
        than to solve a bunch of one-vs-all  binary classifications.
        The argument is that we may generate a multi-class data set
        that can be classified perfectly,
        but for which the training
        data cannot be separated with no error by one-vs-all.
        Recent work in \citep{Multiclass2012NIPS} theoretically proved
        that the direct approach to multi-class classification
        essentially contains the hypothesis classes of one-vs-all.
        Also because the estimation errors of these two methods are roughly
        the same, the direct approach dominates one-vs-all
        in terms of achievable classification performance.

        By introducing the indication operator $ \delta_{ s , t} $
        such that $ \delta_{s, t} = 1 $ if $ s = t $
        and $ \delta_{s, t}= 0 $ otherwise, the above equation can be simplified as
        \begin{equation}
            \label{EQCVPR11:2}
             \delta_{r, y} +  F_y(\x) \geq 1 + F_r(\x), \forall r=1,2,\ldots, k.
        \end{equation}
        We generalize this idea to the entire training set and introduce slack
        variables $ {\boldsymbol \xi} $ to enable soft-margin.
        The primal problem that we want to optimize can then be written as
        \begin{align}
            \label{EQCVPR11:P1}
            \min_{ W, {\boldsymbol \xi} } \; 
            &
            \sum_{i=1}^m \xi_i  + \nu \fnorm[1]{W}
                      \notag
                      \\
                \; \st \;  &
                \delta_{r,\yi} +  H_{i:} \w_\yi \geq 1 + H_{i:} \w_r - \xi_i, \forall i, r,
                \;
                W \geq 0. 
        \end{align}
        Here $ \nu > 0 $ is the regularization parameter.
        $ {\boldsymbol \xi} \geq 0 $ always holds.
        If for a particular $ \x_i $, $ \xi_i $ is negative, then
        one of the constraint in \eqref{EQCVPR11:P1} that corresponds to the case $ r = y_i $
        will be violated. In other words, the constraint corresponding to the case
        $ r = y_i $ ensures the non-negativeness of $ \boldsymbol \xi $.
        Note that we have one slack variable for each training example.
        It is also possible to assign a slack variable to each constraint in \eqref{EQCVPR11:P1}.
        We derive its Lagrange dual, similar to case of LPBoost \citep{Demiriz2002LPBoost}.
    The Lagrangian of problem \eqref{EQCVPR11:P1} can be written as
    \begin{align*}
        L
        = \sum_i \xi_i + \nu \sum_{j,r} W_{jr}
        -  \sum_{i,r}U_{ir} \cdot
        \bigl( \delta_{r, y_i}  +
        H_{i:} \w_{y_i} - 1 - H_{i:} \w_r + \xi_i
        \bigr) - \trace(V^\T W),
    \end{align*}
    with $ U \geq 0$, $ V \geq 0$.
    At optimum, the first derivative of the Lagrangian w.r.t. the primal variables
    must vanish,
    \begin{equation}
        \label{EQCVPR11:A1}
         \frac{\partial L}{\partial \xi_i} = 0
       \longrightarrow  \sum_r U_{ir} = 1, \forall i.
    \end{equation}
    Also,
    \begin{align}
        \label{EQCVPR11:A2}
        \frac{\partial L}{\partial \w_r } &= {\boldsymbol 0}
        \longrightarrow
        \nu 1^\T + \sum_r U_{ir} H_{i:} - \sum_{i,r=y_i}
        \underbrace{\Bigl( \sum_{l} U_{il} \Bigr)}_{=1, \text{ due to } \eqref{EQCVPR11:A1} }
        H_{i:} = V_{ r: },
    \end{align}
    which leads to
    $
     \sum_i U_{ir} H_{i:} - \sum_i \delta_{  r, y_i } H_{i:} \geq
              -  \nu {\boldsymbol 1}^\transpose, \forall  r.
    $
So the Lagrange dual can be written as:\footnote{Strictly speaking,
    this is one of the Lagrange duals of the original primal because
    some transformations from the standard form have been performed.
}
        \begin{align}
            \label{EQCVPR11:D1}
            \min_{ U } \;& \sum_{r=1}^k \sum_{i=1}^m  \delta_{r ,y_i} U_{ir}
                      \notag
                      \\
                \; \st \; &
                 \sum_i \left( \delta_{  r,  y_i } - U_{ir} \right) H_{i:}
                        \leq
                              \nu {\boldsymbol 1}^\transpose, \forall  r,
                      \sum_r U_{ir} = 1, \forall i; \;
                      U \geq 0. 
        \end{align}
        Each row of the matrix $ U $ is normalized.
        The first set of constraints
        can be infinitely many:
        \begin{equation}
            \label{EQCVPR11:100AA}
            \sum_i \left( \delta_{  r,  y_i } - U_{ir} \right) h ( {\boldsymbol x}_i )
                        \leq
                        \nu, \forall  r, \text{ and } \forall h(\cdot) \in {\mathcal H}.
        \end{equation}
        We can now use column generation to solve the problem, similar to the
        LPBoost \citep{Demiriz2002LPBoost}.
        The subproblem for generating weak classifiers is
        \begin{equation}
            \label{EQCVPR11:sub1}
        h^*( \cdot )
        =  \argmax_{ h( \cdot ) \in \mathcal{H}, r }
            \sum_{ i = 1}^m \left(
                             \delta_{r , y_i} - U_{ir}
                            \right)
                            h( \x_i).
        \end{equation}
        The matrix $ U \in \Real^{m \times k }$
        plays the role of measuring importance of a training example.
        The following algorithm can be used to implement our hinge loss based
        \multiboost.

\SetKwInput{KwInit}{Initilaize}
\SetVline
\linesnumbered

\begin{algorithm}[t]
\caption{\multiboost with the hinge loss.
}
\begin{algorithmic}
   \KwIn{

     1)    A set of examples $\{\bx_i,y_i\}$, $i=1 \cdots m$;
     $
     \;
     $
     \\2)    The maximum number of weak classifiers, $T$;
   }

   \KwOut{
      A multi-class classifier $F(\bx) = \argmax_r \sum_{j=1}^{T} w_{r,j} h_j(\bx)$;
}

\KwInit {

   1)      $t \leftarrow 0$;
   \\2)      Initialize sample weights, $U = 1/k$;
}

\While{ $t < T$ }
{
  1) Find the weak classifier by solving the subproblem \eqref{EQCVPR11:sub1};
  \\2) If the stopping criterion has been met, we exit the loop.\\
   \If{ 
   $ \sum_{i=1}^m \left[ \delta_{r,y_i} - U_{ir} \right] h(\bx_i) 
    < \nu  + \epsilon $ } { break; }
  3) Add the best weak classifier, $h_t(\cdot)$, into the primal problem;
  \\4) Solve the primal problem \eqref{EQCVPR11:P1} using
       a primal-dual interior-point LP solver such as  \citet{mosek},
       such that
       the dual solution is also available.
  \\5) $t \leftarrow t + 1$;
}

\end{algorithmic}
\label{ALGCVPR11:alg1}
\end{algorithm}

\paragraph{The exponential loss}

    Now let us consider the exponential loss in the section.
    In the case of the exponential loss,
    We may write the primal optimization problem as
      \begin{align}
            \label{EQCVPR11:PExp}
            \min_{ W } \;&
            \sum_{i=1}^m
            \sum_{r=1}^{k}
           \exp
           \left[ - \bigl(
                    H_{i:} \w_\yi -  H_{i:} \w_r
                    \bigr)
           \right]
                    + \nu' \fnorm[1]{W}, 
                \; \st \;
                W \geq 0.
        \end{align}

        We define a set of margins associated with a training example as
    \begin{equation}
        \label{EQCVPR11:margin}
        \rho_{i,r} =
                   H_{i:} \w_\yi -  H_{i:} \w_r,
                   \; r = 1,\dots, k.
    \end{equation}
    Clearly only when $ \rho_{i,r} \geq 0 $, will the training example
    $ \x_i $ be correctly classified.
    We consider the logarithmic version of the original
    cost function, which does not change the problem because
    $ \log( \cdot ) $ is strictly monotonically increasing.
    So we write \eqref{EQCVPR11:PExp} into
    \begin{align}
        \label{EQCVPR11:PExp2}
            \min_{ W, \boldsymbol \rho } \;
            \log \Bigl(
            \sum_{i=1}^m
            \sum_{r=1}^k
            \exp \left[
            -  \rho_{i,r}      \right]
            \Bigr)
                    + \nu \fnorm[1]{W}
                \; \st \;
                \;
                \rho_{i,r}   =
                    H_{i:} \w_\yi -  H_{i:} \w_r,
                    \forall i, \forall r,
                \; W \geq 0. 
    \end{align}
    The dual problem can be easily derived:
      \begin{align}
            \label{EQCVPR11:D2}
            \min_{ U } \; &
            \sum_{r=1}^k \sum_{i=1}^m U_{ir} \log U_{ir}
                      \notag
                      \\
                 \; \st \;&
                 \sum_i \left[
                 \delta_{r, y_i }
                 \bigl(
                 {\textstyle \sum_{l=1}^k }  U_{il}
                                  \bigr)
                        - U_{ir} \right] H_{i:}
                        \leq
                              \nu {\boldsymbol 1}^\transpose, \forall
                              r; \;
                      \sum_{i,r} U_{ir} = 1, U \geq 0. 
        \end{align}

    We can see that the dual problem is a Shanon entropy maximization problem.
    The objective function of the dual encourages the weights $ U $ to be uniform.
    The KKT condition gives the relationship between the optimal primal and
    dual variables:
    \begin{align}
        \label{EQCVPR11:K2}
        U_{ir}^\star
        =  \frac{\exp( - \rho_{i,r}^\star ) }{ \sum_{i,r} \exp( - \rho_{i,r}^\star ) },\
        \forall i, r.
    \end{align}
    Different from the case of the hinge loss, here $ U $ is normalized as an entire
    matrix. Also we can solve the primal problem using simple (Quasi-)Newton, which is much
    faster than to solve the dual problem using convex optimization solvers.
    Note that the scale of the primal problem is usually smaller than the dual
    problem. After obtaining the primal variable, we can use the KKT condition to
    get the dual variable.
    The subproblem that we need to solve for generating weak classifiers
    also slightly differs from \eqref{EQCVPR11:sub1}:
    \begin{equation}
        \label{EQCVPR11:sub2}
         h^*( \cdot )
            =  \argmax_{ h(\cdot) \in \mathcal{H}, r }
            \sum_{ i = 1}^m \Bigl(
                             \delta_{r , y_i}
                             \bigl(
                 {\textstyle
                        \sum_{l=1}^k }  U_{il}
                      \bigr)
                 - U_{ir}
                 \Bigr)
                 h( \x_i).
    \end{equation}

\paragraph{General convex loss}

\def\lambda{ { \Theta } } 

    We generalize the presented idea to {\em any smooth convex} loss functions in
    this section.
    Suppose $ \lambda( \cdot ) $ is a smooth convex function defined in
    $\Real$. For classification problems, $  \lambda( \cdot )  $ is usually
    a convex surrogate of the non-convex zero-one loss.
    As in the exponential loss case, we introduce a set auxiliary variables
    that define the margin as the pairwise difference of prediction scores.
    This auxiliary variable is the key to lead to the important Lagrange dual,
    on which the fully-corrective boosting algorithms rely.

    The optimization problem can be formulated as
    \begin{align}
        \label{EQCVPR11:PGenera}
        \min_{W, \boldsymbol \rho}
        \;
        \sum_{i=1}^m \sum_{r=1}^k \lambda( - \rho_{i,r}  )
        + \nu \fnorm[1]{W}
        \;
                \st \;&
                \eqref{EQCVPR11:margin}, {\rm and }\;\,
                W \geq 0.
    \end{align}
    The Lagrangian is
    \begin{align*}
        L = 
        \sum_{i,r} \lambda( - \rho_{ i, r }  )
        -  \trace(V^\T W)
        +  \sum_{i,r} U_{ir}  (  H_{i:}
        {\boldsymbol w}_{y_i} - H_{i:} {\boldsymbol w}_r  )
        -  \sum_{i,r} U_{ir} \rho_{i,r} + \nu \sum_{j,r }  W_{jr}.
    \end{align*}
    We can again write its Lagrange dual as
     \begin{align}
            \label{EQCVPR11:General2}
            \min_{ U } \;
            \sum_{r=1}^k \sum_{i=1}^m  \lambda^* ( - U_{ir} )
                 \; \st \;
                 \sum_i \left[
                 \delta_{r, y_i }
                 \bigl(
                 {\textstyle \sum_{l=1}^k }  U_{il}
                                  \bigr)
                        - U_{ir} \right] H_{i:}
                        \leq
                              \nu {\boldsymbol 1}^\transpose, \forall  r,
        \end{align}
    where
    $ \lambda^* ( \cdot ) $ is the Fenchel dual function of $ \lambda ( \cdot ) $
    \citep{Boyd2004Convex}.
    Note that $ \lambda^* ( \cdot ) $ is always convex even if the original
    loss function $ \lambda ( \cdot) $ is non-convex. The difference is that the duality
    gap is not zero when $ \lambda ( \cdot) $ is non-convex.
    The KKT condition establishes the connection between the dual variable $ U $
    and the primal variable at optimality:
    \begin{align}
        \label{EQCVPR11:KKT100}
        U_{ir}^\star = -   \nabla    \lambda  ( \rho_{i, r}^\star ).
    \end{align}
    So we can actually solve the primal problem and then recover the
    dual solution from the primal.
    From \eqref{EQCVPR11:KKT100}, we know that the weight $ U $ is typically
    non-negative for classification problems because the classification loss
    function $ \lambda ( \cdot ) $ is  monotonically decreasing
    and its gradient is non-positive.

\def\lambda{\theta} 

In the next section we formulate the multi-class boosting algorithm using mixed norm regularization.
We maximize the same margin defined in the previous section.

\subsection{MultiBoost with group sparsity (\multistruct)}

\paragraph{The hinge loss with $\ell_{1,2}$-norm regularization}

    Given training samples, our goal is to minimize the multi-class hinge
    loss with $ \ell_{1,2}$ mixed-norm regularization.
    The primal problem can be written as
    \begin{align}
    \label{EQCVPR12:hinge1}
    \min_{ W, \bslack }   \;
    &
    \sum_{i=1}^{m} \xi_i + \nu  \| W  \|_{1,2},    
    \; \st \; 
    \delta_{r,\yi} + H_{i:} \bw_{\yi} \geq
    1 + H_{i:} \bw_{r} - \xi_i, \forall i, \forall r;
    \; 
    W \geq 0; \; \bslack \geq 0.
    \end{align}
    Here $\nu > 0$ is the regularization parameter.
    We rewrite \eqref{EQCVPR12:hinge1}
    by introducing an auxiliary variable $V$:
    \begin{align}
    \label{EQCVPR12:hinge2}
    \min_{ W, V, \bslack }   \;
    &
    \sum_{i=1}^{m} \xi_i + \nu  \| V  \|_{1,2}    \\ \notag
    \st \; &
    \delta_{r,\yi} + H_{i:} \bw_{\yi} \geq
    1 + H_{i:} \bw_{r} - \xi_i, \forall i, r; 
    \; 
    V = W; \; W \geq 0; \; \bslack \geq 0.
    \end{align}
    This auxiliary variable $ V $ splits the regularization term from the
    classification loss, and plays a critical role in deriving the
    meaningful dual problem.
    Actually $\bslack \geq 0 $ is automatically satisfied since the
    constraint, corresponding to the case $r = \yi$, ensures the
    non-negativeness of $\bslack$.
    The Lagrangian can then be written as
    \begin{align*}
    L = &
    \sum_{i=1}^m \xi_i + \nu \|  V \|_{1,2}
    -\sum_{i,r} U_{ir}   ( \delta_{r,\yi} + H_{i:} \bw_{\yi} - 1
    - H_{i:} \bw_{r} + \xi_i)
    \\ \notag
    &
    -
    \trace( Q^\T (\nu W - \nu V) ) - \trace( P^\T W ),
    \notag
    \end{align*}
    where $W$, $V$ and $\bslack$ are primal variables and
    $U$, $P$ and $Q$ are dual variables (with $U \geq 0$ and $P \geq 0$).
    At optimum, the first derivative of the Lagrangian w.r.t.\ the primal variables, $\bslack$, must vanish,
    $   {\partial L}/{\partial \xi_i} = 0 \rightarrow
    {\textstyle \sum_{r}} U_{ir} = 1, \forall i.
    $
    The first derivative w.r.t.\ each column of $W$ must also be zeros:
    \begin{align}
    \label{EQCVPR12:hinge5}
    \frac{\partial L}{\partial \bw_r }  =  {\bf 0}  
    & \to
    \sum_{i} U_{ir} H_{i:} - \sum_{i,r=y_{i}}
    \underbrace{\Bigl[ {\textstyle \sum_{l} } U_{il}\Bigr]  }_{= 1} H_{i:} = P_{:r} - \nu Q_{:r}
    \\ \notag
    & \to
    {\textstyle \sum_{i}} U_{ir} H_{i:} - {\textstyle \sum_{i}} \delta_{r,\yi} H_{i:} \geq -\nu
    Q_{:r}.
    \end{align}
    The infimum over the primal variables $V$ can be expressed as
    \begin{align}
    \label{EQCVPR12:hinge6}
    \inf_{V} L &= \inf_{V} -\nu \left<Q,V \right>
    + \nu \| V \|_{1,2}
     \\ \notag
    &= -\nu {\textstyle \sum_{j}} \sup_{V_{j:}} \left<Q_{j:},V_{j:} \right>
          + \nu {\textstyle \sum_{j}} \| V_{j:} \|_{2} \\ \notag
    &= -\nu {\textstyle \sum_{j}}
                           \Bigl[ \sup_{V_{j:}} Q_{j:}^{\T} V_{j:}
        - \| V_{j:} \|_{2} \Bigr]
        \\ \notag &
    = -\nu {\textstyle \sum_{j}}
    \begin{cases}
    0 & \mbox{if } \| Q_{j:} \|_{2} \leq 1, \forall j,  \\
    \infty & \mbox{otherwise}.
    \end{cases}
    \end{align}
    Note that we use the fact that the convex conjugate of $\|
    V_{j:} \|_{2}$ is the indicator function
    of the dual norm unit ball \citep{Boyd2004Convex}.
    Hence the Lagrange dual can be written as
    \begin{align}
    \label{EQCVPR12:hinge7}
    \min_{ U, Q }   \;
    &
    \sum_{i,r} U_{ir} \delta_{r,\yi}    \\ \notag
    \st \; &
    {\textstyle \sum_i} (\delta_{r,\yi} - U_{ir} ) H_{i:}  \leq \nu Q_{:r}, \forall r; 
    \;
    {\textstyle \sum_r} U_{ir} = 1, \forall i;
    \, U \geq  0;
    \; \| Q_{j:} \|_{2} \leq 1, \forall j.
    \end{align}
    Since there can be infinitely many constraints, we need to use column generation
    to solve \eqref{EQCVPR12:hinge7} \citep{Demiriz2002LPBoost}.
    The subproblem for generating weak classifiers is
    \begin{align}
    \label{EQCVPR12:hinge8}
    h^{\ast}(\cdot) = \argmax_{h(\cdot)\in \mathcal{H}, r}
    \;
    {\textstyle \sum}_{i=1}^m ( \delta_{r,\yi} - U_{ir} ) h(\bxi). 
    \end{align}
    $ h^{\ast}(\cdot) $ is the one that most violates the first constraint
    in the dual \eqref{EQCVPR12:hinge7}.
    The idea of column generation is that instead of solving the
    original problem with prohibitively large number of constraints,
    we consider instead a small subset of entire variable sets.  The
    algorithm begins by finding a variable that most violates the dual
    constraints, \ie, the solution to \eqref{EQCVPR12:hinge8}, which
    corresponds inserting a
    primal variable into \eqref{EQCVPR12:hinge1} or \eqref{EQCVPR12:hinge2}.
    The process continues as
    long as there exists at least one constraint that is violated
    for \eqref{EQCVPR12:hinge7}.
    The algorithm terminates when we cannot find such a violated
    constraint.
    As in AdaBoost, the matrix $U \in {\mathbb R}^{m
    \times k}$ plays the role of measuring the importance of the training
    samples.  The weak classifier which maximizes \eqref{EQCVPR12:hinge8}
    is selected in each iteration.

\paragraph{The hinge loss with $\ell_{1,\infty}$-norm regularization}

    Similarly, the $\ell_{1,\infty}$-norm regularized primal can be written as
    \begin{align}
    \label{EQCVPR12:hinge9}
    \min_{ W, V, \bslack }   \;
    &
    \sum_{i=1}^{m} \xi_i + \nu  \| V  \|_{1,\infty}    \\ \notag
    \st \; &
    \delta_{r,\yi} + H_{i:} \bw_{\yi} \geq
    1 + H_{i:} \bw_{r} - \xi_i, \forall i, \forall r; 
    \; 
    V = W;  \; W \geq 0;  \; \bslack \geq 0.
    \end{align}
    The Lagrangian of \eqref{EQCVPR12:hinge9} can be written as
    \begin{align}
    L = & \sum_{i=1}^m \xi_i + \nu \|  V \|_{1,\infty}
    -\sum_{i,r} U_{ir}   ( \delta_{r,\yi} + H_{i:} \bw_{\yi} - 1
    - H_{i:} \bw_{r} + \xi_i)  -
          \trace( Q^\T (\nu W - \nu V) ) 
          \notag 
          \\
          &
          - \trace(P^\T W),
    \notag
    \end{align}
    with $U \geq 0$ and $P \geq 0$.
    For $\ell_{1,\infty}$-norm, the infimum over the primal variables $V$ can be expressed as,
    \begin{align}
    \inf_{V} L &= \inf_{V} -\nu \left<Q,V \right>
           + \nu \| V \|_{1,\infty}
             \\ \notag
            &= -\nu {\textstyle \sum_{j}} \sup_{V_{j:}} \left<Q_{j:},V_{j:} \right>
                 + \nu {\textstyle \sum_{j}} \| V_{j:} \|_{\infty}
           \\ \notag
           &= -\nu {\textstyle \sum_{j}}
                                   \Bigl[ \sup_{V_{j:}} Q_{j:}^{\T} V_{j:}
                - \| V_{j:} \|_{\infty} \Bigr]
                \\ \notag
            &= -\nu {\textstyle \sum_{j}}
            \begin{cases}
            0 & \mbox{if } \| Q_{j:} \|_{1} \leq 1, \forall j,  \\
            \infty & \mbox{otherwise}.
            \end{cases}
    \end{align}
    Here we make use of the fact that,
    \begin{align}
    f^{\ast}( \by) = \sup_\bx ( \by^{\T} \bx - \| \bx \| ) =
    \begin{cases}
            0 & \mbox{if } \|  \by \|_{\ast} \leq 1,  \\
            \infty & \mbox{otherwise},
            \end{cases}
    \end{align}
    where $\| \cdot \|_{\ast}$ is dual norm of $\| \cdot \|$.\footnote{
    We note here that $\ell_{p}$ norm in primal corresponds to
    $\ell_{q}$ norm in dual with ${1}/{p} + {1}/{q} = 1$.
    For example, the Euclidean norm, $\| \cdot \|_{2}$ is dual to itself and
    the $\ell_1$-norm, $ \| \cdot \|_{1}$ is dual to the $\ell_{\infty}$-norm,
    $ \| \cdot \|_{\infty}$.}
    Hence we can derive its corresponding dual as,
    \begin{align}
    \min_{ U,Q }   \;
    &
    \sum_{i,r} U_{ir} \delta_{r,\yi}
    \\ \notag
    \st  &
    {\textstyle \sum_i} (\delta_{r,\yi} - U_{ir} ) H_{ij}
        \leq \nu Q_{:r}, \forall r;
    \;
    {\textstyle \sum_r} U_{ir} = 1, \forall i; \; U \geq 0;
    \;
    \| Q_{j:} \|_{1} \leq 1, \forall j.
    \end{align}
    From the dual problem we see that the only difference between $\ell_{1,2}$-norms and
    $\ell_{1,\infty}$-norms is in the norm of the last constraint.
    This is not surprising since $\ell_{p}$ norm in primal corresponds to
    $\ell_{q}$ norm in dual with ${1}/{p} + {1}/{q} = 1$.

\paragraph{The logistic loss with $\ell_{1,2}$-norm and $\ell_{1,\infty}$-norm regularization}

    In this secion, we consider the logistic loss with a mixed-norm regularization.
    The learning problem for the logistic loss in an $\ell_{1,2}$ regularization
    framework can be expressed as
    \begin{align}
        \label{EQCVPR12:logPrimalL12}
            \min_{ W, V, \brho } \;
            &
            \frac{1}{mk} \sum_{i=1}^{m} \sum_{r=1}^{k}
                \log \bigl( 1 + \exp \left( -\rho_{ir} \right) \bigr) +
                \nu  \| V  \|_{1,2}   \\ \notag
            \st \; &
            \rho_{ir} = H_{i:} \bw_{\yi} - H_{i:} \bw_{r}, \forall i,
            \forall r;  \; 
            V = W; \; W \geq 0.
            \notag
    \end{align}
    The Lagrangian of \eqref{EQCVPR12:logPrimalL12} can be written as
    \begin{align}
        \label{EQCVPR12:log3}
        L =& \frac{1}{mk} \sum_{i=1}^m \sum_{r=1}^k
        \log\bigl( 1 + \exp\left( -\rho_{ir} \right) \bigr) +
                \nu \| V  \|_{1,2}
            -\sum_{i,r} U_{ir}
            ( \rho_{ir} - H_{i:} \bw_{\yi} + H_{i:} \bw_{r})
            \\ \notag
            &- \trace( Q^\T (\nu W - \nu V) ) - \trace( P^\T W ),
    \end{align}
    with $U \geq 0$ and $P \geq 0$.
    At optimum the first derivative of the Lagrangian w.r.t.\ each row of $W$ must be zeros
    \begin{align}
        \label{EQCVPR12:log4}
        \frac{\partial L}{\partial \bw_r } =  {\bf 0}
        &\rightarrow
            \sum_{i,r=\yi}
            \left( \textstyle \sum_{l} U_{il} \right) H_{i:}
            {\textstyle - \sum_{i}} U_{ir} H_{i:}
            = P_{:r} - \nu Q_{:r} \\ \notag
        &\rightarrow
        {\textstyle \sum_{i}}
            \left[ \delta_{r,\yi} \left( \textstyle \sum_{l} U_{il} \right)
                  - U_{ir} \right] H_{i:} \geq -\nu Q_{:r}
    \end{align}
    for  $ \forany r $.
    Take infimum over the primal variable, $\rho_{ir}$,
    \begin{align}
        \label{EQCVPR12:log6}
        \frac{\partial L}{\partial \rho_{ir}} = 0 \rightarrow
            \rho_{ir} ^\star = -\log\left( \frac{-m k U_{ir} ^\star }
            {m k U_{ir}^\star -1 } \right), \forall i, \forall r.
    \end{align}
    and
    \begin{align}
        \label{EQCVPR12:log7}
        \inf_{\rho_{ir}} L =
            \frac{1}{mk}
            \sum_{ir} - (1 + mkU_{ir}) & \log\left(1+mkU_{ir}\right)
                      -mkU_{ir} \log \left( -mkU_{ir} \right).
    \end{align}
    By reversing the sign of $U$,
    the Lagrange dual can be written as
    \begin{align}
        \label{EQCVPR12:logDualL12}
            \max_{ U,Q }   \;
            &
            -\frac{1}{mk} \sum_{i=1}^{m} \sum_{r=1}^{k}
                \Bigl[mkU_{ir} \log \left( mkU_{ir} \right)
                + 
                \left(1 - mkU_{ir}\right) \log\left( 1 - mkU_{ir} \right) \Bigr]
            \\ \notag
            \st \; &
            {\textstyle \sum_{i}}
                \bigl[ \delta_{r,\yi} \left( \textstyle \sum_{l} U_{il} \right)
                  - U_{ir} \bigr]
                        H_{i:} \leq \nu Q_{:r}, \forall r;
             \;
            \| Q_{j:} \|_{2} \leq 1, \forall j.
    \end{align}
    Through the KKT optimality condition, the
    gradient of Lagrangian   
    over primal variables $\brho$
    and dual variables $U$ must vanish at
    the optimum.
    The solutions of \eqref{EQCVPR12:logPrimalL12} and \eqref{EQCVPR12:logDualL12}
    coincide since both problems are feasible and satisfy Slater's
    condition.
    One can find the solution by solving either problem.
    The relationship between the optimal values of $\brho$ and $U$ can be expressed as
    \begin{align}
    \label{EQCVPR12:log9}
    U_{ir}^\star = \frac{ \exp(-\rho_{ir}^\star) }{ mk \bigl(
    1+\exp(-\rho_{ir}^\star) \bigr) }.
    \end{align}
    As is the case for the hinge loss,
    the dual of the $\ell_{1,\infty}$-norm regularized logistic loss can be written as
    \begin{align}
    \label{EQCVPR12:logDualL12_B}
    \max_{ U,Q }   \;
    &
    -\frac{1}{mk} \sum_{i=1}^{m} \sum_{r=1}^{k}
    \Bigl[mkU_{ir} \log \left( mkU_{ir} \right)
    + 
    \left(1 - mkU_{ir}\right) \log\left( 1 - mkU_{ir} \right) \Bigr]
    \\ \notag
    \st \; &
    {\textstyle \sum_{i}}
    \bigl[ \delta_{r,\yi} \left( \textstyle \sum_{l} U_{il} \right)
    - U_{ir} \bigr] H_{i:} \leq \nu Q_{:r}, \forall r;
    \; \| Q_{j:} \|_{1} \leq 1, \forall j.
    \end{align}

\paragraph{General convex loss with arbitrary $\ell_{1,p}$-norm regularization}

\def\Cost{ {   \Theta } } 

    In this section, we generalize our idea to any convex loss functions
    with any mixed-norm regularizers.
    As before, we define $ 
    \Cost(\cdot)$ as a smooth convex function and $\Omega(\cdot)$
    as any well-established regularization functions\footnote{
    Here we assume a non-overlapping group structure.
    This assumption is always valid since
    $ W =  [ \w_1, \w_2, \ldots, \w_k ]$ and $\bigcap_{j=1}^{n} W_{:j}
    = \emptyset$.
    }.
    We define the margin as the pairwise difference of prediction scores.
    The general mixed-norm regularized optimization problem that 
    we want to solve is,
    \begin{align}
        \label{EQCVPR12:sup_gen1}
            \min_{ W, V, \brho } \;
            &
            \sum_{i=1}^{m} \sum_{r=1}^{k}   \Cost ( \rho_{ir} ) +
                \nu \sum_{j} \Omega( V_{j:} )
            \\ \notag
            \st \; &
            \rho_{ir} = H_{i:} \bw_{\yi} - H_{i:} \bw_{r}, \forall i,
            \forall r;       \ \text{and} \
            V = W; W \geq 0.
    \end{align}
    The Lagrangian of \eqref{EQCVPR12:sup_gen1} is
    \begin{align}
        L =& \sum_{i=1}^m \sum_{r=1}^k  \Cost  \left( \rho_{ir} \right) +
                \nu \sum_{j} \Omega( V_{j:} )    \
            -\sum_{i,r} U_{ir}
            ( \rho_{ir} - H_{i:} \bw_{\yi} + H_{i:} \bw_{r})
            \\ \notag
            \; &- \trace( Q^\T (\nu W - \nu V)  ) - \trace ( P^\T W ),
    \end{align}
    Following our derivation for multi-class logistic loss, the Lagrange dual can be written as,
    \begin{align}
            \min_{ U,Q }   \;
            &
            \sum_{i=1}^{m} \sum_{r=1}^{k}
               \Cost ^ {\ast}(-U_{ir})
             \\ \notag
            \st \; &
            {\textstyle \sum_{i}}
                \bigl[ \delta_{r,\yi} \left( \textstyle \sum_{l} U_{il} \right)
                  - U_{ir} \bigr]
                        H_{i:} \leq \nu Q_{:r}, \forall r;
             \ \text{and} \
            \Omega^{\ast}( Q_{j:} ) \leq 1, \forall j.
    \end{align}
    where $ \Cost^{\ast}(\cdot)$ is the Fenchel dual function of $l(\cdot)$
    and $\Omega^{\ast}(\cdot)$ is the Fenchel conjugate of $\Omega(\cdot)$.
    Through the KKT condition, the relationship between the dual variable $U$ and the primal variable $\rho$,
    \begin{align}
        U_{ir} ^\star = - \nabla  \Cost ( \rho_{ir}  ^\star  ),
    \end{align}
    holds at optimality.
    It is important to note here the difference between \multiboost (having $\ell_1$ penalty) and
    \multistruct (having mixed-norm penalty).
    Although both dual variables, $U$, have the same expression, \ie,
    each dual variable is defined as the negative gradient of the loss
    at $\rho_{ir}$, the solution to the primal variables, $W$, are different.
    \multiboost does not enforce group sparsity and
    is unable to exploit the existence of structural features.
    The details of our boosting algorithm are given in
    Algorithm~\ref{ALGCVPR12:alg1}.

\SetKwInput{KwInit}{Initilaize}
\SetVline
\linesnumbered

\begin{algorithm}[t]
\caption{MultiBoost with shared weak classifiers via group sparsity.
}
\begin{algorithmic}
   \KwIn{

     1)    A set of examples $\{\bx_i,y_i\}$, $i=1 \cdots m$;
     $
     \;
     $
     \\2)    The maximum number of weak classifiers, $T$;
   }

   \KwOut{
      A multi-class classifier $F(\bx) = \argmax_r \sum_{j=1}^{T} w_{r,j} h_j(\bx)$;
}

\KwInit {

   1)      $t \leftarrow 0$;
   \\2)      Initialize sample weights, $U_{ir} = 1/(mk)$;
}

\While{ $t < T$ }
{
  1) Train a weak learner,
  $h_t(\cdot) = $
  $$
  \left\{ \begin{array}{ll}
  \argmax_{h(\cdot) \in \mathcal{H},r}  \sum_{i=1}^m
  \left[
  \delta_{r,y_i}  - U_{ir}
   \right] h(\bx_i) , &\mbox{{hinge loss}} \\
  \argmax_{h(\cdot),r}  \sum_{i=1}^m
  \left[
  \delta_{r,y_i} \left( \textstyle \sum_{l} U_{il} \right) - U_{ir}
   \right] h(\bx_i) , &\mbox{{logistic}}
       \end{array} \right.
  $$
  \\2) If the stopping criterion has been met, we exit the loop.\\
   \If{ $\Bigl\| \sum_{i=1}^m \left[
   \delta_{r,y_i} - U_{ir}
   \right] h(\bx_i) \Bigr\|_2
    < \nu  + \epsilon $ } { break; \quad {(hinge loss)}}
     \If{ $\Bigl\| \sum_{i=1}^m \left[
   \delta_{r,y_i} \left( \textstyle \sum_{l} U_{il} \right) - U_{ir}
   \right] h(\bx_i) \Bigr\|_2
    < \nu  + \epsilon $ } { break; \quad {(logistic loss)}}
  3) Add the best weak learner, $h_t(\cdot)$, into the current set;
  \\4) Solve either the primal or the dual problem (we solve the dual \eqref{EQCVPR12:hinge7})
  for the { hinge loss case};
       or solve the primal problem \eqref{EQCVPR12:logPrimalL12} using \ADMM
  for the logistic loss case;
  \\5) Update sample weights (dual variables);
  \\6) $t \leftarrow t + 1$;
}

\end{algorithmic}
\label{ALGCVPR12:alg1}
\end{algorithm}

    \begin{thm}[Convergence property] Both $\ell_1$-norm and $\ell_{1,p}$-norm
    regularized boosting algorithms are guaranteed to converge to an optimum
    of any convex loss functions
    provided that both algorithms makes progress at each boosting iteration.
    In other words, as long as the objective value decreases,
    both algorithms optimize \eqref{EQCVPR12:sup_gen1}
    globally to a desired accuracy.
    \end{thm}

    \begin{proof}
    Here we consider \multiboost. The proof of \multistruct would follow
    the same discussion.
    Our proof relies on the fact that the $\ell_1$ regularizer forces the
    set of possible solutions to be sparse and
    each column generation iteration guarantees the objective value to be smaller.
    We first assume that the current solution is a finite subset of
    weak learners, $\{ h_j( \cdot) \}_{j=1}^{n-1}$.
    If we add a weak learner, $h_{n} (\cdot)$, that is not in the current subset,
    and the corresponding coefficient, $w_{r,n} = 0, \forall r$,
    the solution must remain unchanged.
    We can simply conclude that the current set of weak learners,
    $\{ h_j (\cdot) \}_{j=1}^{n-1}$, and their coefficients, $\bw_r, \forall r$,
    are already at the optimal solution.

    Next, we consider the case when the optimality condition is violated.
    We need to show that we can find a weak learner, $h_{n} (\cdot)$,
    which is not in the current set and $\exists r: w_{r,n} > 0$.
    Let us assume that $h_{n} (\cdot)$ is the base learner found by solving
    Step~$1$ in Algorithm~\ref{ALGCVPR11:alg1},
    and the stopping criterion (Step~$2$) has not been met.
    Hence
    $\exists r: \sum_{i=1}^m \left[
     \delta_{r,y_i} - U_{ir} \right] h_{n} (\bxi)
     > \nu$.
    If after the weak learner $h_{n} (\cdot)$ is added into the primal problem,
    the primal solution remains unchanged, that is,
    $w_{r,n} = 0, \forall r$.
    Based on the optimality condition:
    \begin{align}
     \inf_{\bw_r} L^{\prime} = \inf_{\bw_r} \left( \nu - \sum_{i=1}^m \left[
     \delta_{r,y_i} - U_{ir} \right] H_{i:} - 
     \bv_r^\T \right) \bw_r. \notag
    \end{align}
    At optimum, the first derivative of $L^{\prime}$ w.r.t.\ the 
    primal variables must vanish,
    \ie, $L^{\prime}$ must be ${\mathbf 0}$. 
    But $\exists r: v_{r,n} = \nu - \sum_{i=1}^m \left[
     \delta_{r,y_i} - U_{ir} \right] h_{n} (\cdot) < 0$.
     This contradicts the fact the Lagrange multiplier, $v_{r,n}$,
     must be greater than or equal to zero.

    We can conclude that after the base learner $h_{n} (\cdot)$ is added into
    the primal problem, $\exists r: w_{r,n} > 0$.
    Since one more primal variable is added into the problem, the objective value
    of the primal problem must decrease.
    A decreasing in the objective value guarantees that the algorithm makes
    progress at each iteration.
    Since all optimization problems are convex,
    there exists no local optimal solution.
    Therefore the proposed column generation based boosting
    is guaranteed to converge
    to the global optimal solution.

    \end{proof}

\subsection{Implementation}

    Note that the dual problem of hinge loss, \eqref{EQCVPR12:hinge7}, is a
    conic quadratic optimization problem
    involving
    several linear constraints and quadratic cones.
    We use the Mosek optimization solver to solve \eqref{EQCVPR12:hinge7}
    which
    provides solutions for both primal and dual problems simultaneously
    using the interior-point method.
    For the logistic loss formulation the primal problem has $nk$
    variables and $mk$ simple constraints \eqref{EQCVPR12:logPrimalL12}.  The
    dual problem has $mk$ variables\footnote{Here we ignore the equality
    constraints since they can be put back into the original cost
    function.} and $nk$ constraints.  In boosting, we often have more
    training samples than final weak classifiers ($m \gg n$).
    However, the $\ell_{1,2}$-norm is not differentiable everywhere,
    and thus to solve \eqref{EQCVPR12:logPrimalL12} we apply the
    \ADMM method \citep{Boyd2011Distributed}.
    \ADMM decouples the regularization term from the logistic loss by
    introducing additional auxiliary variables.  The algorithm then
    solves \eqref{EQCVPR12:logPrimalL12} by using an alternating
    minimization approach.
    \ADMM formulates the original problem as the following,
    \begin{align}
        \label{EQCVPR12:ADMM1}
            \min_{ W,Z }
            f(W) + g(Z)
            \; \st \;
            W = Z.
    \end{align}
    Here $f(W)$ is any convex loss functions \eqref{EQCVPR12:sup_gen1}
    and $g(Z) = \nu \cdot \Omega(Z) $ is any regularization functions.
    As in the method of multipliers, we form the augmented Lagrangian,
    \begin{align}
        \label{EQCVPR12:ADMM2}
        L_{\lambda} =&\ f(W) + g(Z) +
        \left< U,W-Z \right> + \frac{\lambda}{2} \| W - Z \|_2.
    \end{align}
    Here $\lambda$ is the augmented Lagrangian parameter ($\lambda > 0$).
    The method of multipliers for \eqref{EQCVPR12:ADMM1} has the form,
    \begin{align}
        \label{EQCVPR12:ADMM3}
        (W^{s+1}, Z^{s+1}) =&\ \argmin_{W,Z} L_{\lambda} (W, Z, U^s) \\
        U^{s+1} =&\ U^s + \lambda ( W^{s+1} - Z^{s+1} ).
    \end{align}
    Here the Lagrangian is minimized jointly with respect to both $W$ and $Z$ variables.
    Since it is expensive to solve a joint minimization in \eqref{EQCVPR12:ADMM3},
    both primal variables ($W$ and $Z$) are updated in an alternating fashion.
    This alternate update scheme is known as \ADMM.
    \ADMM consists of the following iterations,
    \begin{align}
        \label{EQCVPR12:ADMM4}
        W^{s+1} =&\ \argmin_{W} L_{\lambda} (W, Z^s, U^s) \\
        \label{EQCVPR12:ADMM5}
        Z^{s+1} =&\ \argmin_{Z} L_{\lambda} (W^{s+1}, Z, U^s) \\
        U^{s+1} =&\ U^s + \lambda ( W^{s+1} - Z^{s+1} ).
    \end{align}
    As an example, we regularize the above logistic loss with a mixed-norm
    $\ell_{1,2}$ regularizer.
    We can rewrite \eqref{EQCVPR12:ADMM4} and  \eqref{EQCVPR12:ADMM5} as,
    \begin{align}
        \label{EQCVPR12:ADMM6}
        W^{s+1} =&\ \argmin_{W} \ \frac{1}{mk} \sum_{i=1}^m
            \sum_{r=1}^k \log\left(1+\exp\left( -\rho_{ir} \right)\right)+ 
            (U^s)^\T W + \frac{\lambda}{2} \| W - Z^s \|_2^2 \\
        \label{EQCVPR12:ADMM7}
        Z^{s+1} =&\ \argmin_{Z} \ \nu \| Z \|_{1,2} - (U^s)^\T Z +
            \frac{\lambda}{2} \| W^{s+1} - Z \|_2^2.
    \end{align}
    Here $\rho_{ir} = H_{i:} \bw_{r} - H_{i:} \bw_{y_i}$.
    Since \eqref{EQCVPR12:ADMM6} is now smooth and differentiable everywhere,
    a quasi-Newton method such as L-BFGS-B can be used to efficiently solve \eqref{EQCVPR12:ADMM6}.
    For \eqref{EQCVPR12:ADMM7}, a closed-form solution exists and it can be
    computed through sub-differential calculus \citep{Boyd2011Distributed}.
    The solution is known as a block soft thresholding,
    \begin{align}
        \label{EQCVPR12:ADMM8}
        Z_{j:}^{s+1} =& \ \mathcal{S}_{\nu / \lambda}(W_{j:}^{s+1} + U_{j:}^{s}), \forall j,
    \end{align}
    where $\mathcal{S} $ is a vector soft thresholding operator defined as
    \begin{equation}
        \label{EQ:STVec}
    \mathcal{S}_\kappa( \bx ) = (1 - \kappa/\| \bx  \|_2)_{+} \bx,
    \end{equation}
    with $  \mathcal{S} ( 0 ) = 0 $. 
    A brief summary of \ADMM in provided in Algorithm~\ref{ALG:ADMM}.

\SetKwInput{KwInit}{Initilaize}
\SetVline
\linesnumbered

\begin{algorithm}[t]
\caption{ADMM for solving \eqref{EQCVPR12:logPrimalL12}}
\begin{algorithmic}
   \KwIn{

     1)    Outputs of weak classifiers, $H$;
     $ \; $
     \\2)    Augmented Lagrangian parameter, $\lambda$;
     $ \; $
     \\3)    The maximum number of iterations, $s_{\rm max}$;
   }

   \KwOut{
      An optimal $W^{\ast}$;
   }

\KwInit {

   1)      $s \leftarrow 0$;
   $ \; $
   \\2)      $W^{0}$, $Z^{0}$, $U^{0}$;
}

\Repeat{ {\rm convergence}}
{
  $W^{s+1} = \argmin_{W} \frac{1}{mk} \sum_{i=1}^m
        \sum_{r=1}^k \log\bigl(1+\exp\left( - \rho_{ir} \right)\bigr)+
  (U^s)^\T W + \frac{\lambda}{2} \| W - Z^s \|_2^2 $;
  \\ Update $Z_{j:}^{s+1} $ using \eqref{EQCVPR12:ADMM8};
  \\ $U^{s+1} = U^s + \lambda ( W^{s+1} - Z^{s+1} )$;
  \\ $s \leftarrow s + 1$;
  \\   \If{ $s > s_{\rm max}$ } { break; }
}

\end{algorithmic}
\label{ALG:ADMM}
\end{algorithm}

    \def\qmax{ { q_{\rm max} } }

\paragraph{Distributed optimization via \ADMM}
    We describe here how to exploit distributed computing in \ADMM to speed up the
    training time of our proposed approach.
    In order to solve the problem in a distributed fashion, we first
    separate the loss function across $\qmax$ blocks of data.
    We redefine our problem as,
    \begin{align}
        \label{EQCVPR12:distboost1}
            \min_{ W,Z }   \;
            \sum_{q=1}^\qmax    \Cost_q ( W_q ) + \nu \cdot \Omega( Z )
            \;
            \st \;
            W_q - Z = 0, q = 1, \ldots, \qmax,
    \end{align}
    where $\Cost_q$ refers to the loss function for the $q$-th block of data.
    Similar to the previous section, \ADMM considers the following iterations,
    \begin{align}
        \label{EQCVPR12:distboost2}
        W_q^{s+1} =&\ \argmin_{W_q} L_{\lambda} (W_q, Z^s, U^s), \forall q; \\
        \label{EQCVPR12:distboost3}
        Z^{s+1} =&\ \argmin_{Z} L_{\lambda} (W_1^{s+1}, \ldots,
        W_\qmax ^ {s+1}, Z, U_1^s, \ldots, U_\qmax ^ s); \\
        U_q^{s+1} =&\ U_q^s + \lambda ( W_q^{s+1} - Z^{s+1} ), \forall
        q=1\cdots \qmax,
    \end{align}
    where $\lambda$ is the augmented Lagrangian parameter ($\lambda > 0$).
    The resulting \ADMM algorithm for \eqref{EQCVPR12:distboost2} and \eqref{EQCVPR12:distboost3} is
    \begin{align}
        \label{EQCVPR12:distboost4}
        W_q^{s+1} =&\ \argmin_{W_q} \ l_q (W_q) + ( U^s_q )^\T W_q +
        \frac{\lambda}{2} \| W_q -  Z^s \|^2_2, \forall q,  \\
        \label{EQCVPR12:distboost6}
        Z^{s+1} =&\ \argmin_{Z} \ \nu \cdot \Omega(  Z  ) +
                  \sum_{q=1}^ \qmax  \Bigl( -  \left( U^s_q \right)^\T Z +
                 \frac{\lambda}{2} \| W_q^{s+1} -  Z \|^2_2 \Bigr), \\
        \label{EQCVPR12:distboost8}
        U_q^{s+1} =&\ U_q^s + \lambda ( W_q^{s+1} - Z^{s+1} ), \forall q,
    \end{align}
    where $\bar{W}^{s+1} = \frac{1}{ \qmax } \sum_{q=1}^\qmax
    W_q^{s+1}$ and $\bar{U}^{s} = \frac{1}{\qmax} \sum_{q=1}^\qmax  U_q^{s}$.
    For \multiboost, the $Z$-update is a soft threshold operation, \ie,
    \begin{align}
       \label{EQCVPR12:distboost9}
        Z^{s+1} =&\ \argmin_{Z} \ \nu \cdot \Omega(  Z  ) +
                  \sum_{q=1}^ \qmax 
                  \Bigl( -  \left( U^s_q \right)^\T Z +
                 \frac{\lambda}{2} \| W_q^{s+1} -  Z \|^2_2 \Bigr), \\ \notag
                =&\ \argmin_{Z} \ \nu \|  Z  \|_{1} +
                  ( \qmax \lambda /2) \| Z - \bar{W}^{s+1} -
                  (1/\lambda) \bar{U}^s \|_2^2, \\ \notag
                =&\  \mathcal{S}_{\nu / (\lambda  \qmax ) } \left(\bar{W}_{j:}^{s+1} +
                  (1/\lambda) \bar{U}_{j:}^{s} \right), \forall j.
    \end{align}
    The soft thresholding operator $\mathcal{S}$ applied to a scalar is defined as
    \begin{equation}
    \mathcal{S}_\kappa ( x ) = (1 - \kappa / | x | )_{+} x.
    \end{equation} 
    for $ x \neq 0$. 
    For \multistruct, \eg, $\ell_{1,2}$-norm regularizer,
    a closed-form solution exists for
    \eqref{EQCVPR12:distboost6} and it can be computed as
    \begin{align}
        \label{EQCVPR12:distboost10}
        Z^{s+1} =&\ \argmin_{Z} \ \nu \cdot \Omega(  Z  ) +
                  \sum_{q=1}^\qmax \Bigl( -  \left( U^s_q \right)^\T Z +
                 \frac{\lambda}{2} \| W_q^{s+1} -  Z \|^2_2 \Bigr), \\ \notag
                =&\ \argmin_{Z} \ \nu \|  Z  \|_{1,2} +
                  (\qmax \lambda /2) \| Z - \bar{W}^{s+1} -
                  (1/\lambda) \bar{U}^s \|_2^2, \\ \notag
                =&\  \mathcal{S}_{\nu / (\lambda \qmax ) } \left(\bar{W}_{j:}^{s+1} +
                  (1/\lambda) \bar{U}_{j:}^{s} \right), \forall j.
    \end{align}
    In this case, note that $  \mathcal{S}  $ is the vector
    thresholding operator defined in \eqref{EQ:STVec}.  

    Here we assume that $\sum_{q=1}^\qmax  m_q = m$, \ie, the sum of the number of samples in each block is equal to the total number of samples.
    The first step, \eqref{EQCVPR12:distboost4}, can be carried out independently in parallel for each block of data.
    In other words, we distribute  \eqref{EQCVPR12:distboost4} to each thread or processor.
    The second step, \eqref{EQCVPR12:distboost9} or \eqref{EQCVPR12:distboost10}, gathers variables computed in \eqref{EQCVPR12:distboost4} to form the average.
    After the final step, \eqref{EQCVPR12:distboost8}, the value of $U_q^{s+1}$ is then distributed to the subsystems.

    For both hinge loss and logistic regression, we can rewrite \eqref{EQCVPR12:distboost4} as,
    \begin{align}
         W_q^{s+1}   =&\ \argmin_{W_q} \frac{1}{m_q}
                  \sum_{i=1}^{m_q} \xi_i + (U_q^s)^\T W_q + \frac{\lambda}{2} \| W_q - Z^s \|_2^2, \\
         W_q^{s+1}   =&\ \argmin_{W_q} \frac{1}{m_q k}
                  \sum_{i=1}^{m_q} \sum_{r=1}^k \log \Bigl(1+\exp\left( -\rho_{ir} \right)\Bigr)+ (U_q^s)^\T W_q + \frac{\lambda}{2} \| W_q - Z^s \|_2^2;
    \end{align}
    respectively.

    \subsection{Faster training of multi-class boosting}
    Although we have combined \ADMM with
    L-BFGS-B for faster training of multi-class logistic loss,
    the resulting algorithm is still computationally expensive to train.
    The drawback of \eqref{EQCVPR12:logPrimalL12} is that the formulation cannot
    be separated for faster training.  Since real-world data often
    consists of a large number of samples and classes,
    the training procedure can be very slow.

    In order to improve the training efficiency of the classifier we thus
    propose here another variation of the multi-class boosting based
    on the logistic loss.
    This variation is achieved through a simplification of the form of
    $\rho_{ir}$ in \eqref{EQCVPR12:logPrimalL12} to $\rho_{ir} = y_{ir} H_{i:}
    \bw_{r}$ where $y_{ir} = 1$ if $\yi = r$ and $y_{ir} = -1$,
    otherwise.
    Note that this formulation was originally introduced in
    \cite{Chapelle2008Multi} for multi-class as well as multi-label
    support vector machine (SVM) learning and proved to be effective.
    To our knowledge, this formulation of multi-class loss function
    has not been applied to boosting.
    Here we extend it to multi-class boosting.
    The fast training (\ova) formulation is:
    \begin{align}
        \label{EQCVPR12:ova1}
            \min_{ W, \brho }
            \;
            \frac{1}{mk} \sum_{i=1}^{m} \sum_{r=1}^{k} \log \bigl( 1 + \exp \left( -\rho_{ir} \right) \bigr) +
                \nu  \| W  \|_{1,2} \;  
            \;
            \st
            \; 
            \rho_{ir} = y_{ir} H_{i:} \bw_{r}, \forall i, \forall r;
            \;
            W \geq 0.
    \end{align}
    The Lagrange dual can be written as
    \begin{align}
        \label{EQCVPR12:ova2}
        \max_{ U }
        \; &
        -\frac{1}{mk} \sum_{i=1}^{m} \sum_{r=1}^{k}
        \Bigl[mkU_{ir} \log \left( mkU_{ir} \right) + 
        \left(1 - mkU_{ir}\right)
        \log\left( 1 - mkU_{ir} \right) \Bigr] \\ \notag
        \st \; & {\textstyle \sum_{i}} U_{ir} y_{ir}
        H_{i:} \leq \nu Q_{:r}, \forall r;
        \;
        \| Q_{j:} \|_{2} \leq 1, \forall j.
    \end{align}
    The relationship between $\rho$ and $U_{ir}$ is the same as
    \eqref{EQCVPR12:log9}.  We replace steps $1$ and $2$ in
    Algorithm~\ref{ALGCVPR12:alg1} with the constraint in \eqref{EQCVPR12:ova2} and
    step $4$ in Algorithm~\ref{ALGCVPR12:alg1} with the optimization problem in
    \eqref{EQCVPR12:ova1}.
    As in \cite{Chapelle2008Multi}, it is easy to
    apply the above formulation to multi-label classification,
    where each example can have multiple class labels.
    We leave this for future work.

    \paragraph{Parallel optimization for \ova boosting}
        The computational bottleneck of Algorithm~\ref{ALG:ADMM} lies in minimizing
        $W^{s+1}$.  By simplifying the margin as $\rho_{ir} = y_{ir}
        H_{i:} \bw_{r}$, we can solve each $\bw_r, \forall r$
        independently.  This speeds up our training time by a factor
        proportional to the number of classes.  Let us define $W = [
        \bw_1, \bw_2,\ldots, \bw_k ] \in {\Real}^{n \times
    k}$, $Z = [ \bz_1, \bz_2,\ldots, \bz_k ] \in {\Real}^{n \times
    k}$ and $U = [ \bu_1, \bu_2,\ldots, \bu_k ] \in {\Real}^{n \times
    k}$, line $2$ in Algorithm~\ref{ALG:ADMM} can simply be replaced by,
    \begin{align}
        \label{EQCVPR12:ADMMPar}
        \bw^{s+1}_{r} =& \argmin_{\bw} \frac{1}{mk}
        \sum_{i=1}^m
        \sum_{r=1}^k
             \log \bigl( 1+\exp\left( -\rho_{ir} \right) \bigr)+ 
            (\bu_r^s)^\T \bw + \frac{\lambda}{2}
            \| \bw - \bz_r^s \|_2^2, \; \forall r.
    \end{align}
    Even without a multi-core processor, solving a series of
    \eqref{EQCVPR12:ADMMPar} is still faster than solving line $2$ in
    Algorithm~\ref{ALG:ADMM}.  Distributed optimization can also be
    applied to our algorithms to further speed up the training time.  The
    idea is to distribute a subset of training data in \eqref{EQCVPR12:ADMMPar}
    to each processor and gather optimal $\bw^{s+1}_{r}$ to form the
    average.  Interested readers should refer to Chapter $8$ in
    \cite{Boyd2011Distributed}.

\section{Experiments}

\subsection{\multiboost}

    \label{sec:exp}

    \begin{figure}[!ht]
    \begin{center}
    \subfigure[mean images of ``1'', ``6'' and ``9'']{
    \includegraphics[width=0.1602\textwidth,clip]{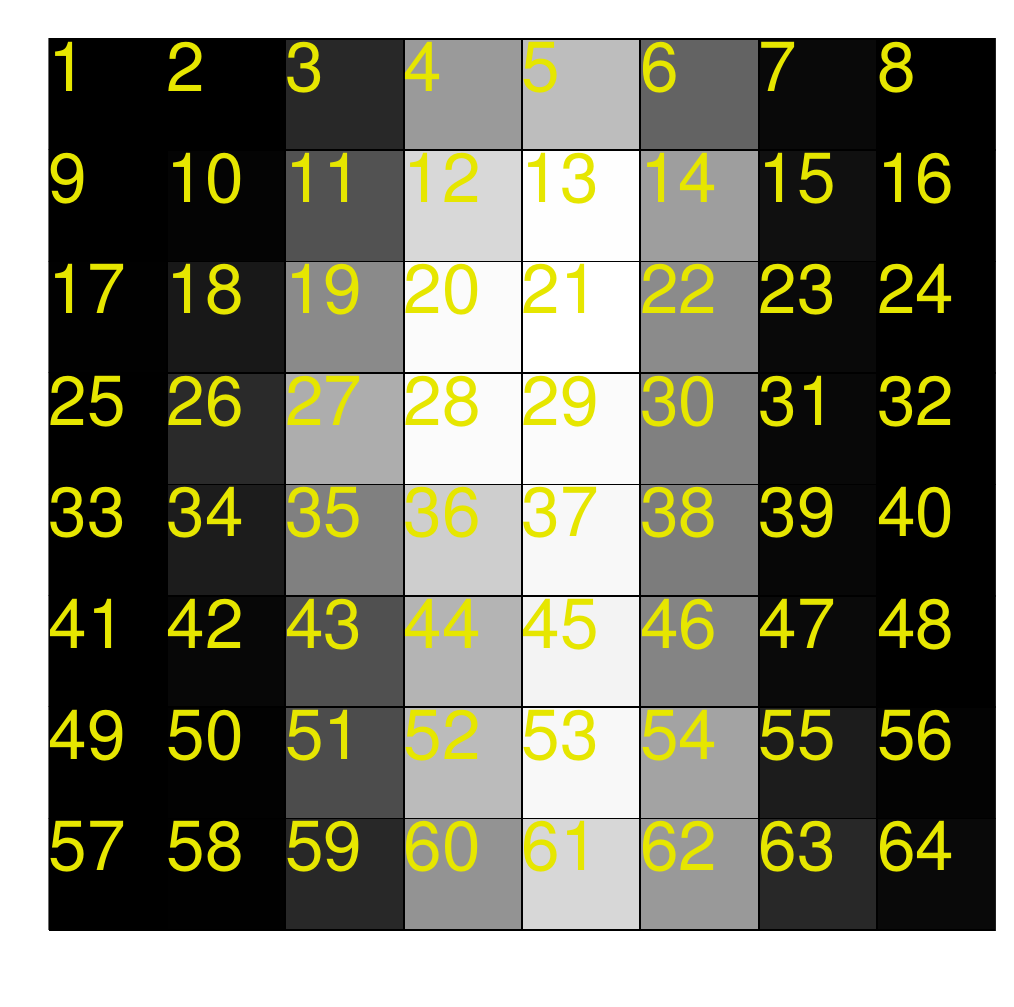}
    \includegraphics[width=0.162\textwidth,clip]{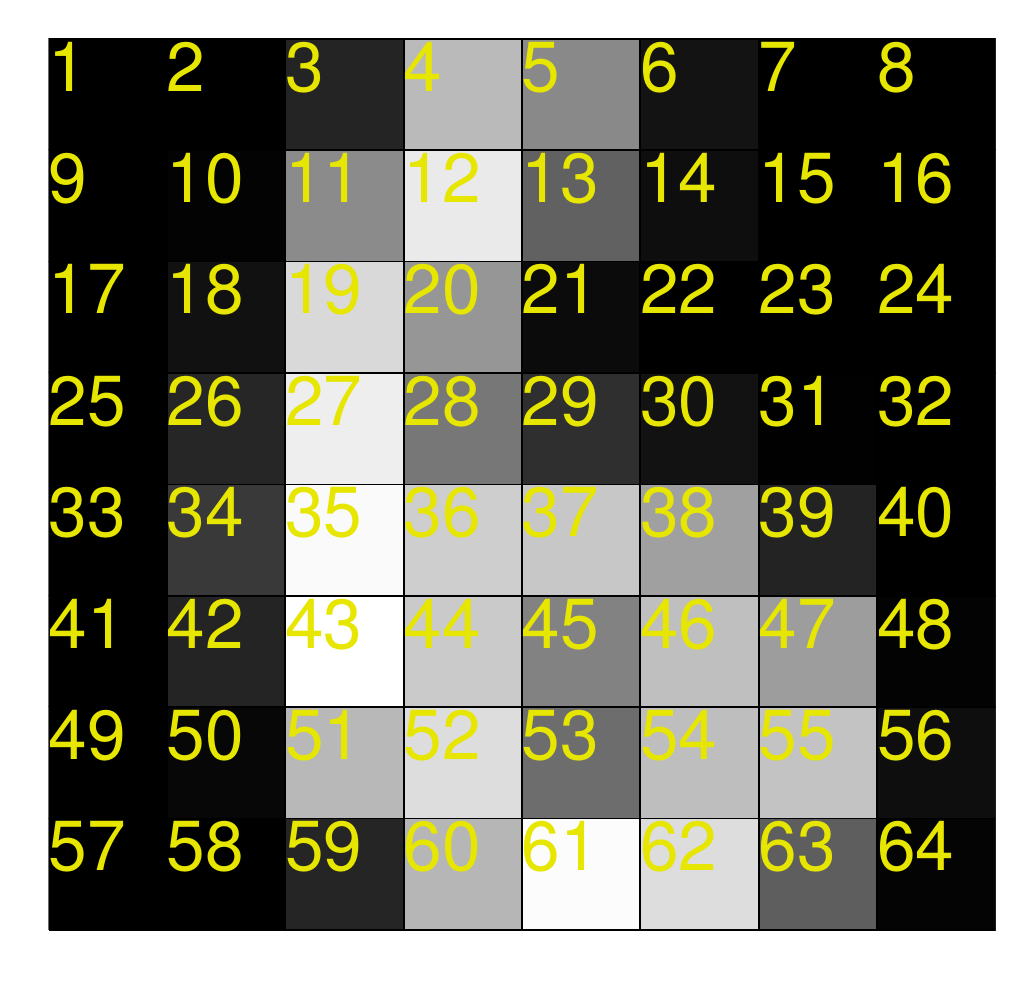}
    \includegraphics[width=0.162\textwidth,clip]{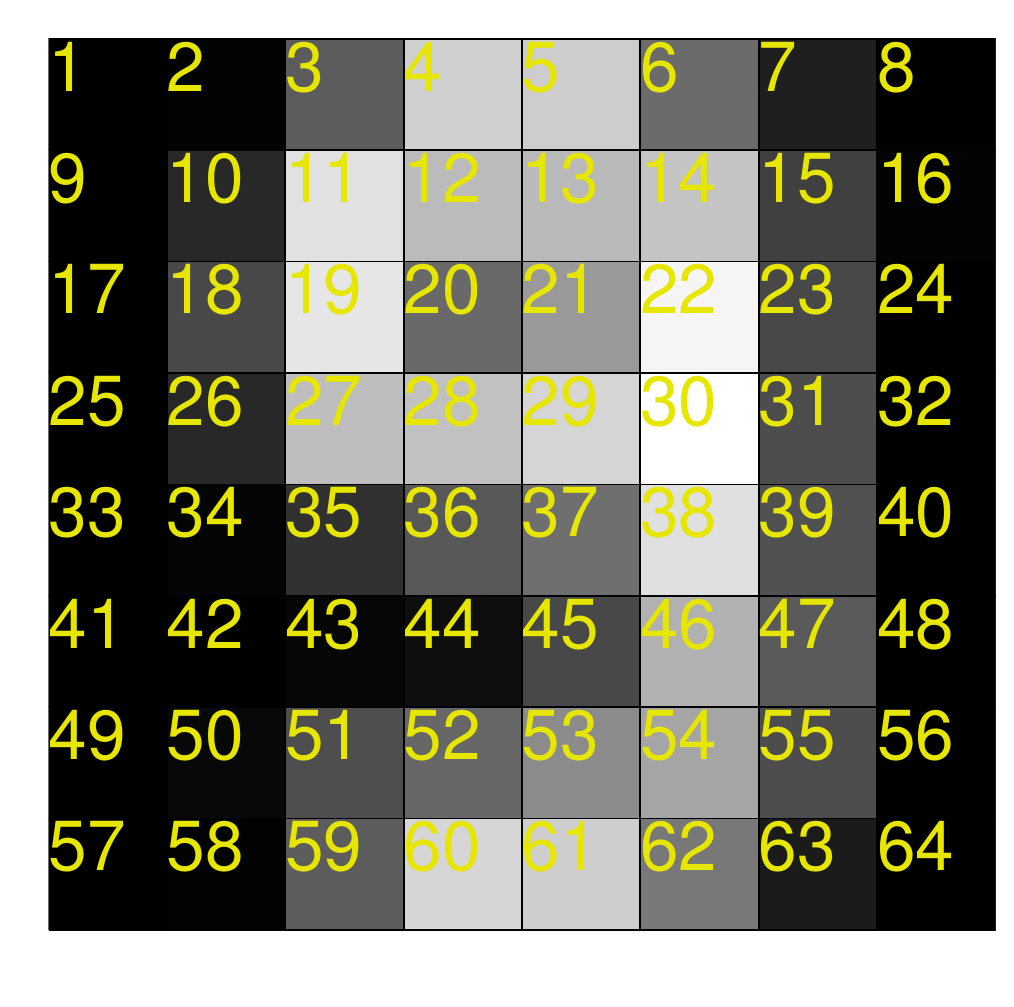}
    }
    \subfigure[AdaBoost.ECC]{
    \includegraphics[width=0.16\textwidth,clip]{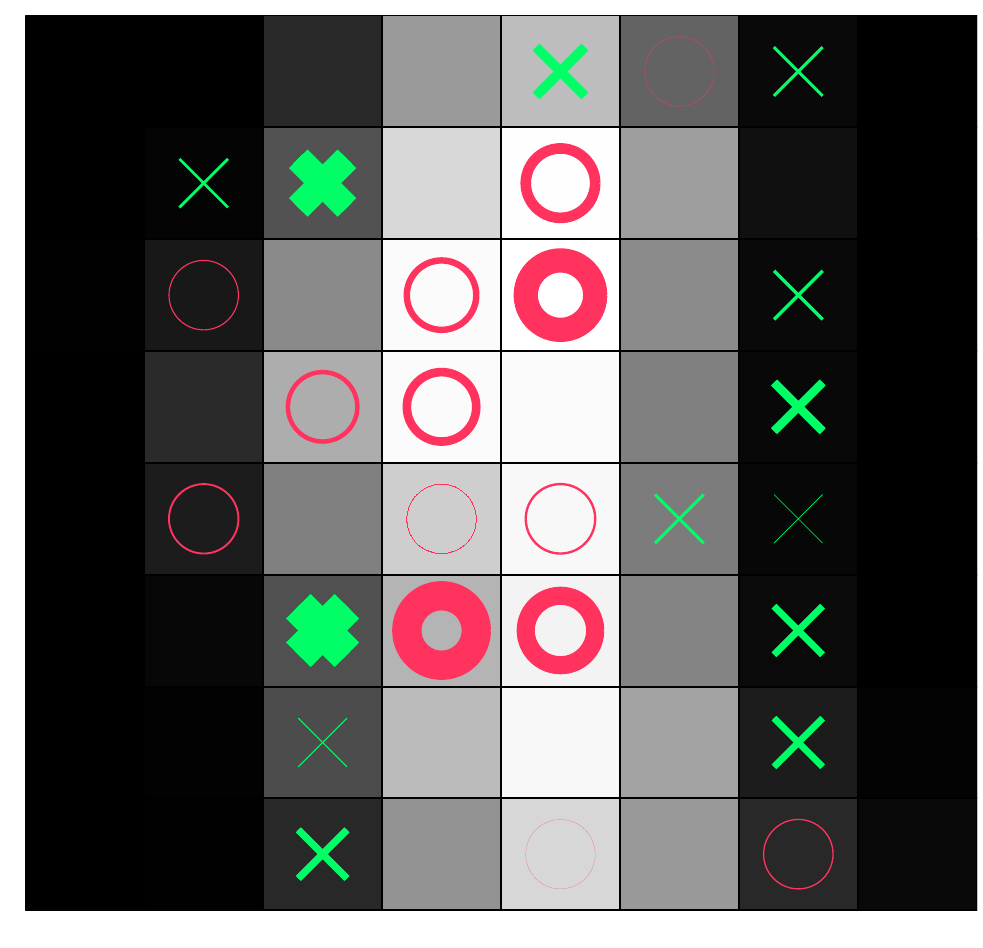}
    \includegraphics[width=0.16\textwidth,clip]{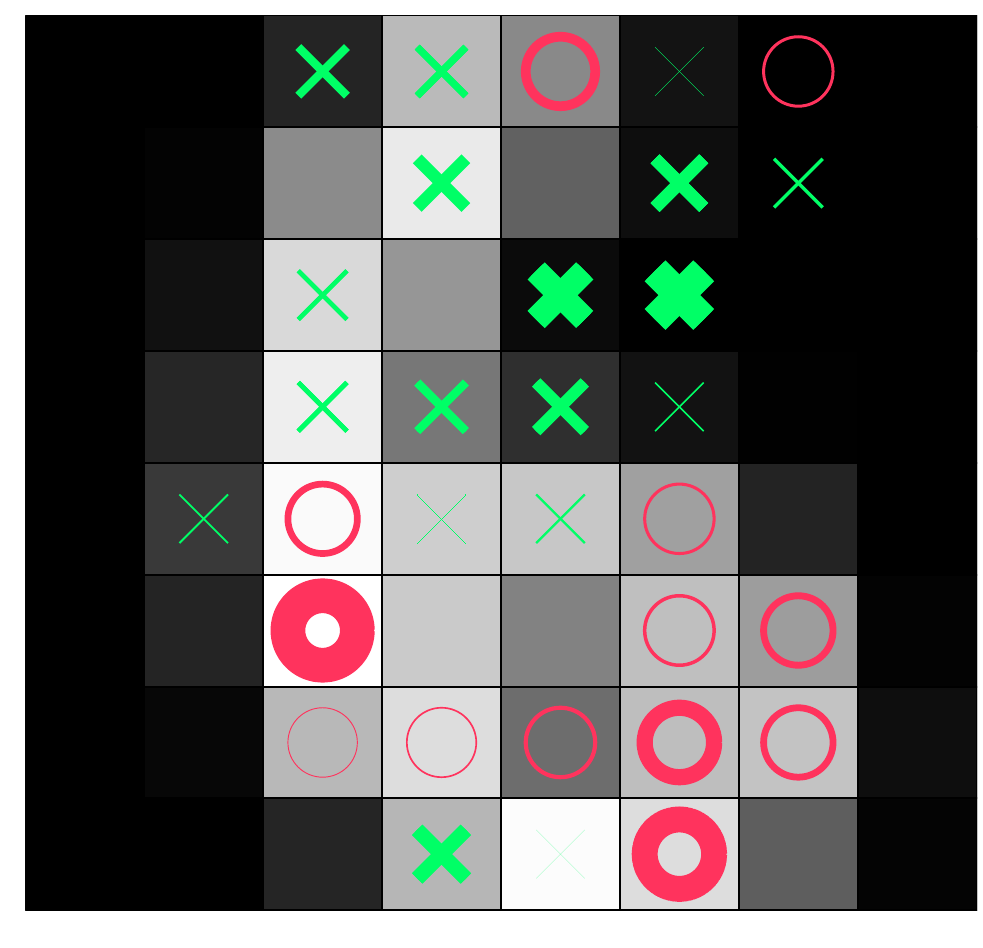}
    \includegraphics[width=0.16\textwidth,clip]{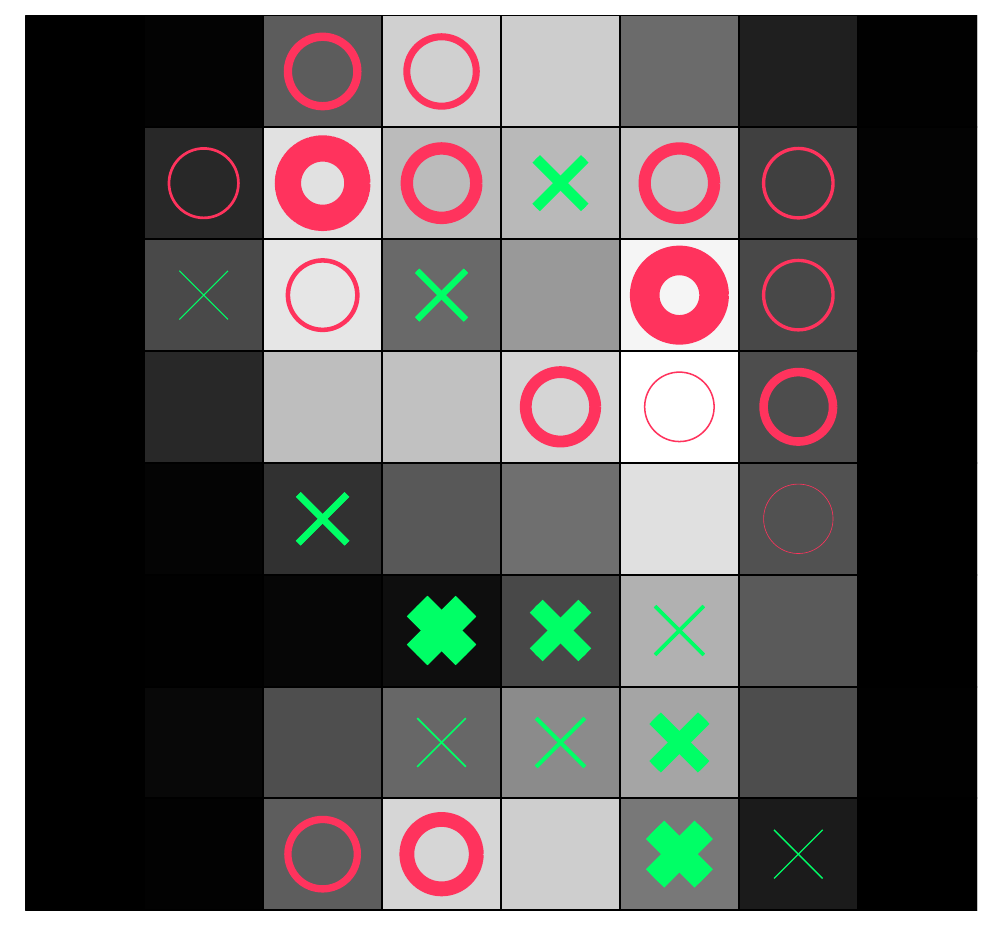}
    }
    \subfigure[\mbh]{
    \includegraphics[width=0.16\textwidth,clip]{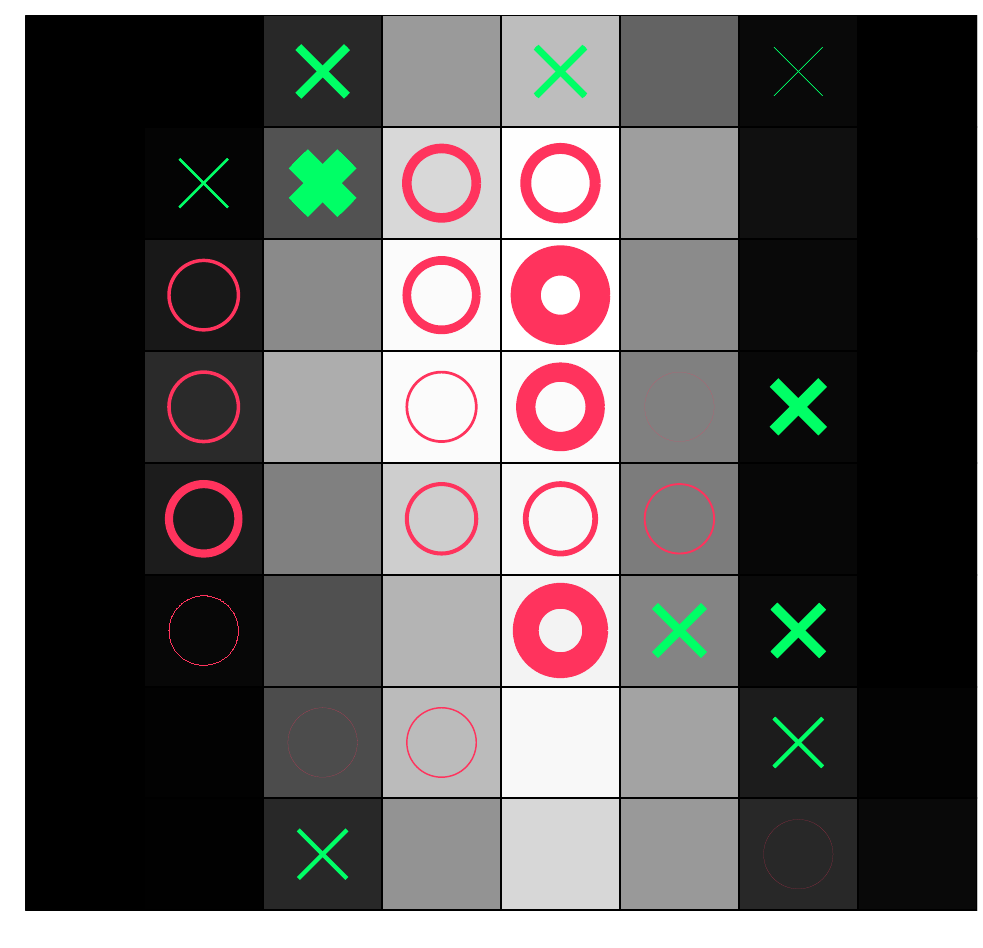}
    \includegraphics[width=0.16\textwidth,clip]{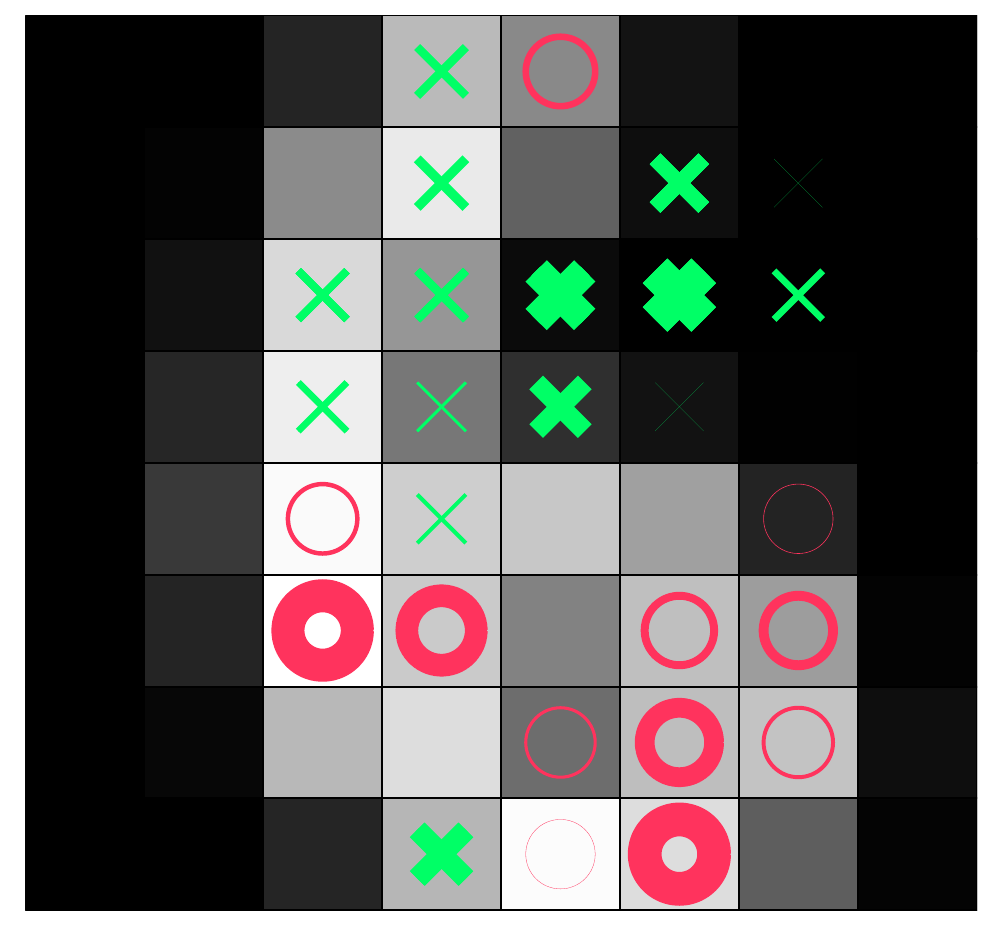}
    \includegraphics[width=0.16\textwidth,clip]{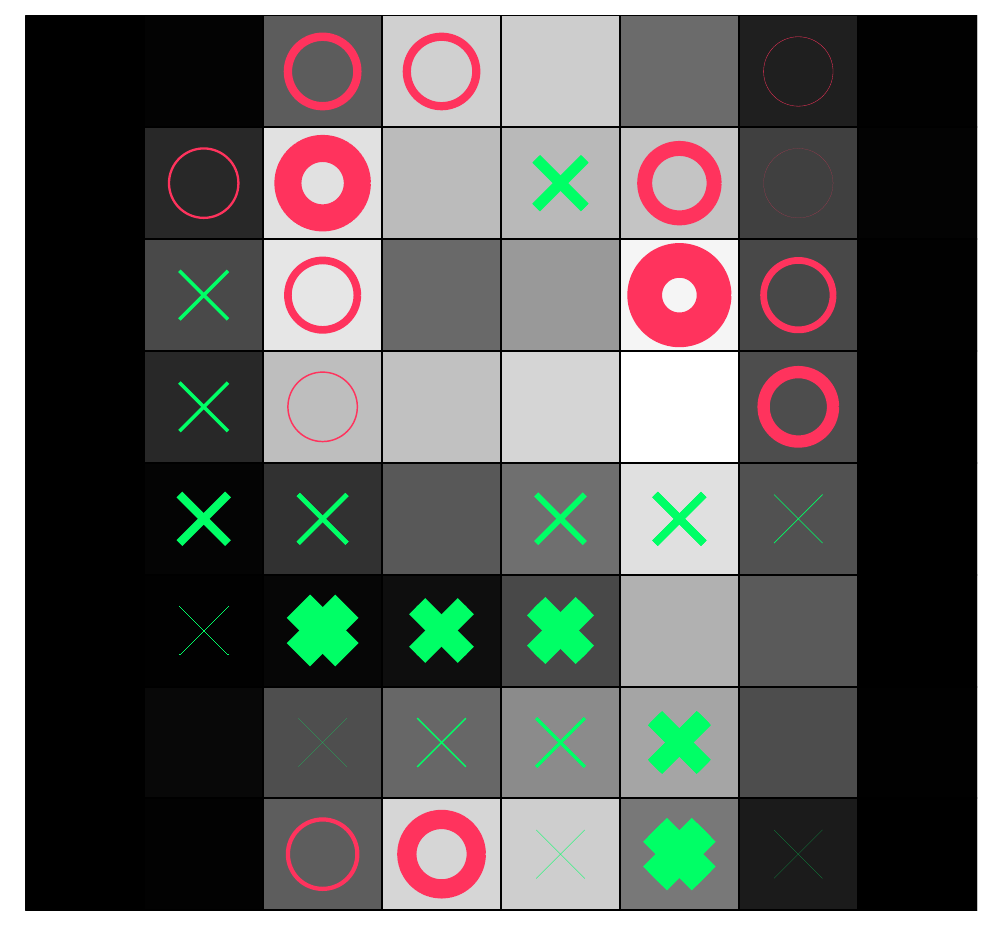}
    }
    \subfigure[\mbe]{
    \includegraphics[width=0.16\textwidth,clip]{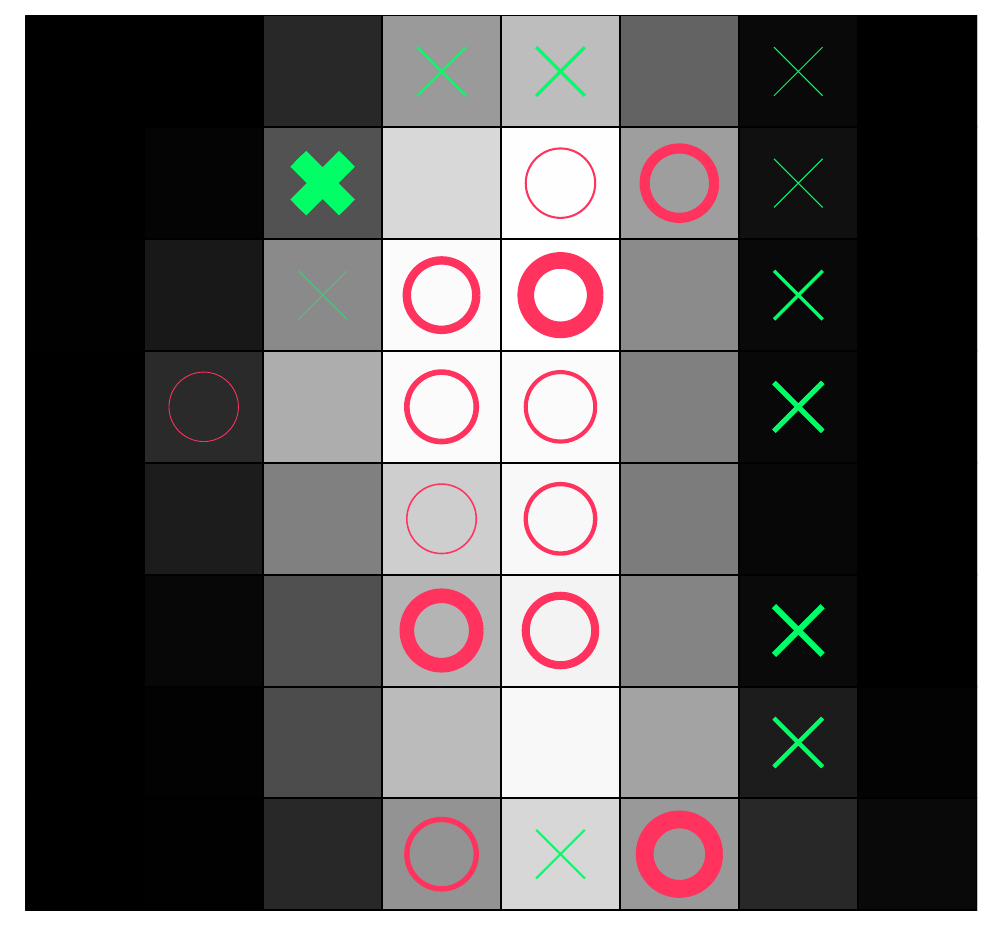}
    \includegraphics[width=0.16\textwidth,clip]{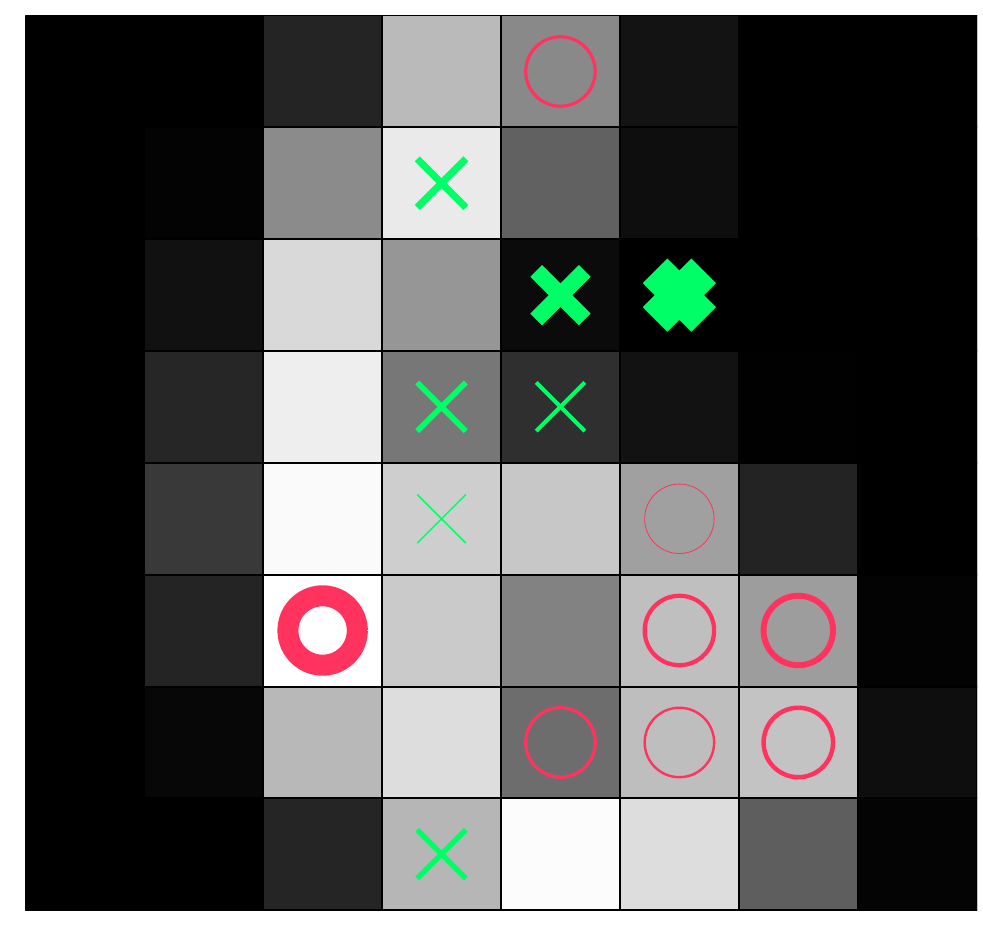}
    \includegraphics[width=0.16\textwidth,clip]{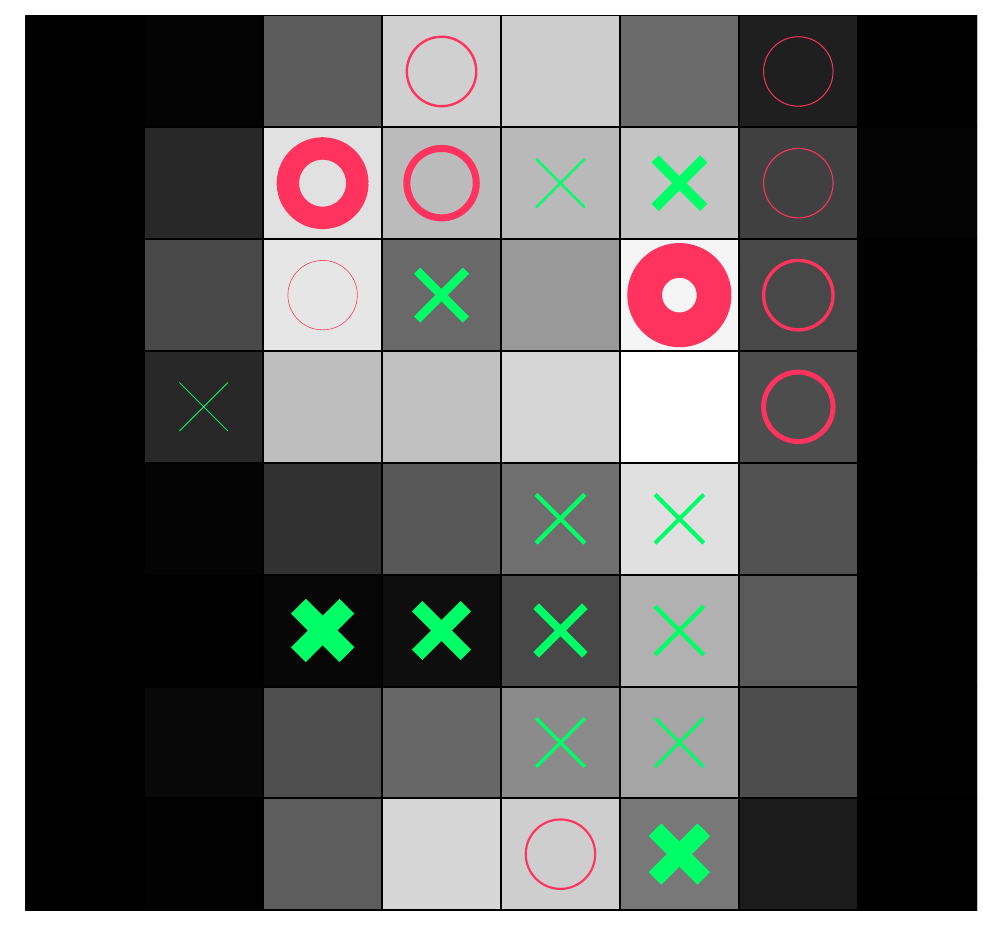}
    }
    \end{center}
       \caption{
       Plot (a) shows the mean images of the samples
       belonging to digits ``1'', ``6'' and ``9''.  Each block is a feature
       and is numerically indexed.
       The remaining plots illustrate the classification models trained on this data set
       by (b) AdaBoost.ECC, (c) \mbh and (d) \mbe.
       Red circles indicate that weak classifiers on these features
       should take large values;
       Green crosses indicate small values should be taken on these
       features in order to make correct classification.
       The width of a mark is proportional to the weight of the stump.
       We can see that \mbh is slightly better than AdaBoost.ECC,
       e.g., on the $43$-th and $21$-th features.
       }
    \label{figCVPR11:opt}
    \end{figure}

    We first performed a few sets of experiments to compare \multiboost
    with previous multi-class boosting algorithms.
    For fair comparison, we focus on the multi-class algorithms
    using binary weak learners, including AdaBoost.MO and AdaBoost.ECC,
    which are still considered as the state-of-the-art.
    For AdaBoost.MO, the error-correcting output codes are introduced
    to reduce the primal problem into multiple binary ones; for AdaBoost.ECC,
    the binary partitioning is made at each iteration by using the ``random-half" method,
    which has been experimentally proven better
    than the optimal ``max-cut'' solution \cite{li2006multiclass}.
    Decision stumps are chosen as the weak classifiers for all boosting algorithms,
    due to its simplicity and the  controlled complexity of the weak learner.

    Convex optimization problems are involved in \mbh and \mbe. To solve them,
    we use the off-the-shelf Mosek convex optimization package, 
    which provides solutions for both primal and dual problems simultaneously with its
    interior-point Newton method.
    We also need to set the regularization parameter $\nu$ for these two algorithms using cross
    validation.
    For each run, a five-fold cross validation is carried out first
    to determine the best $\nu$. Notice that the loss functions in \mbh and \mbe
    may have different scales, we choose the parameter
    from $\{10^{-4}$, $10^{-3}$, $5\cdot 10^{-3}$, $0.01$, $0.02$, $0.04$, $0.05\}$ for the former,
    and the candidate pool $\{10^{-8}$, $10^{-7}$, $5 \cdot 10^{-7}$, $8 \cdot 10^{-7}$,
$10^{-6}$, $ 2 \cdot 10^{-6}$, $4 \cdot 10^{-6}$, $8 \cdot 10^{-6}$, $10^{-5}\}$ for the latter.

    \paragraph{Toy data}
    In the first experiment, we make the comparison on a toy data set,
    which consists of 4 clusters of planar points.
    Each cluster has 50 samples, which are drawn from their respective normal distribution.
    As shown in Figure \ref{figCVPR11:toy}(a), the centers of the circles indicate where their means are,
    and the radii depict the different deviations.
    We run the boosting algorithms on this toy data set
    and plot the decision boundaries on the $x$-$y$ plane.
    Figures \ref{figCVPR11:toy}(b)-(e) illustrate the results
    when the number of training iterations is set to be $100$.
    In this case, it is hard to state which model is better.
    However, if we increase the iteration to $5000$ times, the planes in (f) and (g)
    are apparently over segmented by AdaBoost.MO and AdaBoost.ECC. On the contrary,
    the decision boundaries of (h) \mbh and (i) \mbe seem closer to the
    true decision boundary.
    Unlike the others, models trained by AdaBoost.MO are more complex,
    since this learning method assembles $\ell$ weak classifiers rather than one at each iteration
    if $\ell$-length codewords are used. Empirically we see that
    AdaBoost.ECC also  seems  susceptible to over-fitting.

    \begin{figure*}[!ht]
    \begin{center}
    \subfigure[data]{\includegraphics[width=0.3\textwidth,clip]{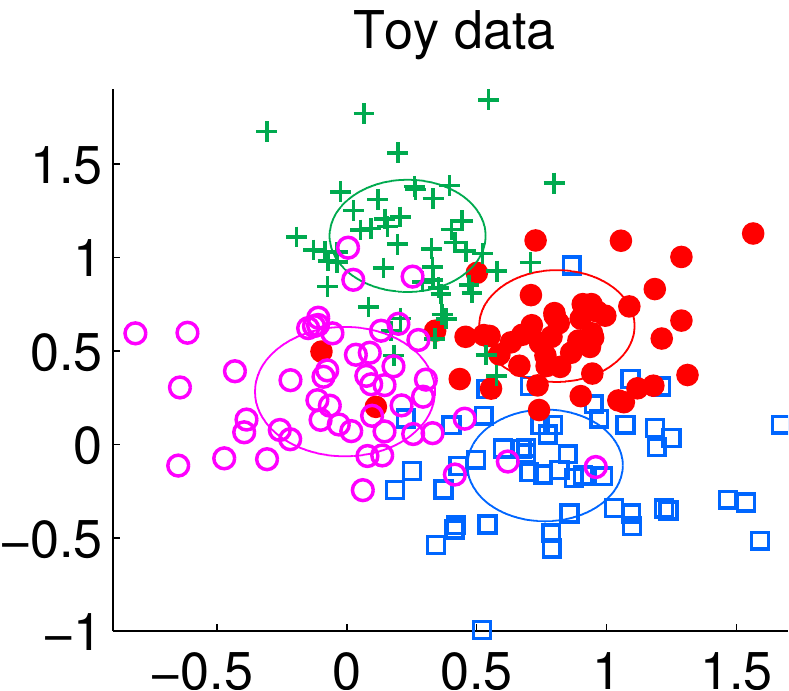}}
    \subfigure[AdaBoost.MO, 100 iterations]{\includegraphics[width=0.3\textwidth,clip]{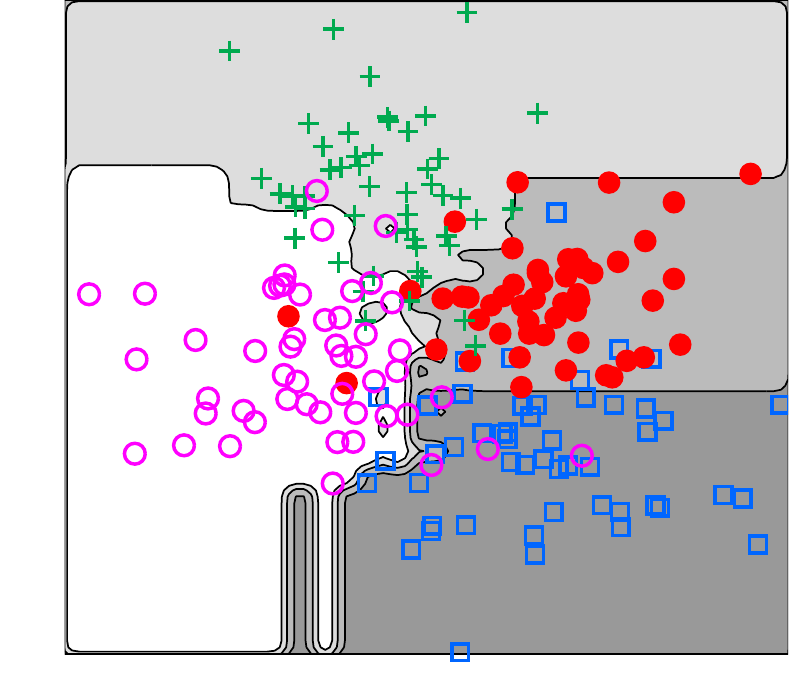}}
    \subfigure[AdaBoost.ECC, 100 iterations]{\includegraphics[width=0.3\textwidth,clip]{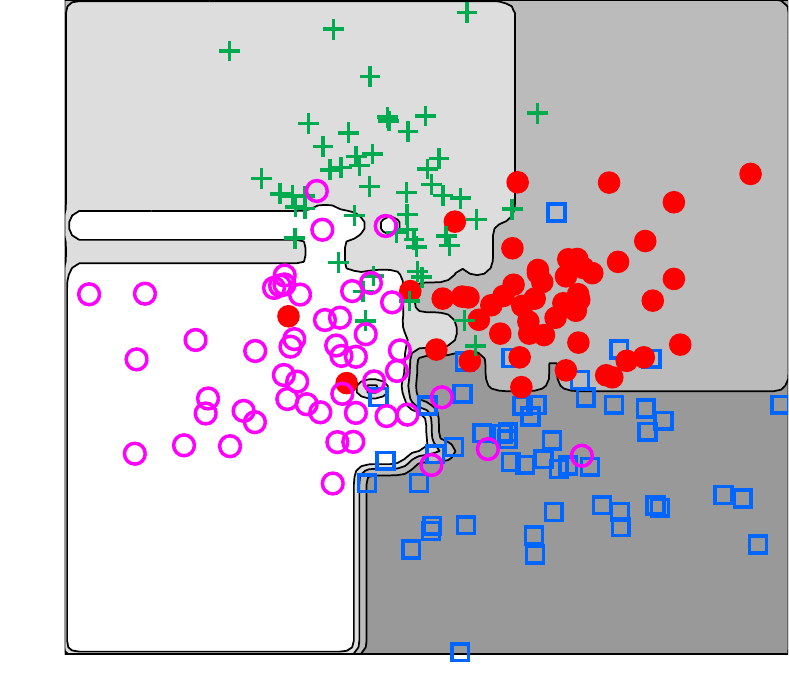}}
    \subfigure[\mbh, 100 iterations]{\includegraphics[width=0.3\textwidth,clip]{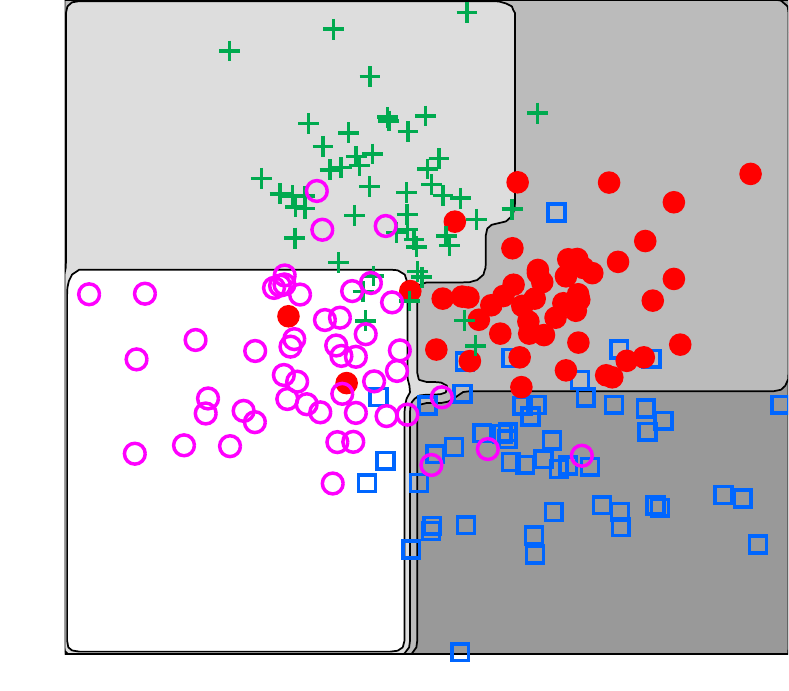}}
    \subfigure[\mbe, 100 iterations]{\includegraphics[width=0.3\textwidth,clip]{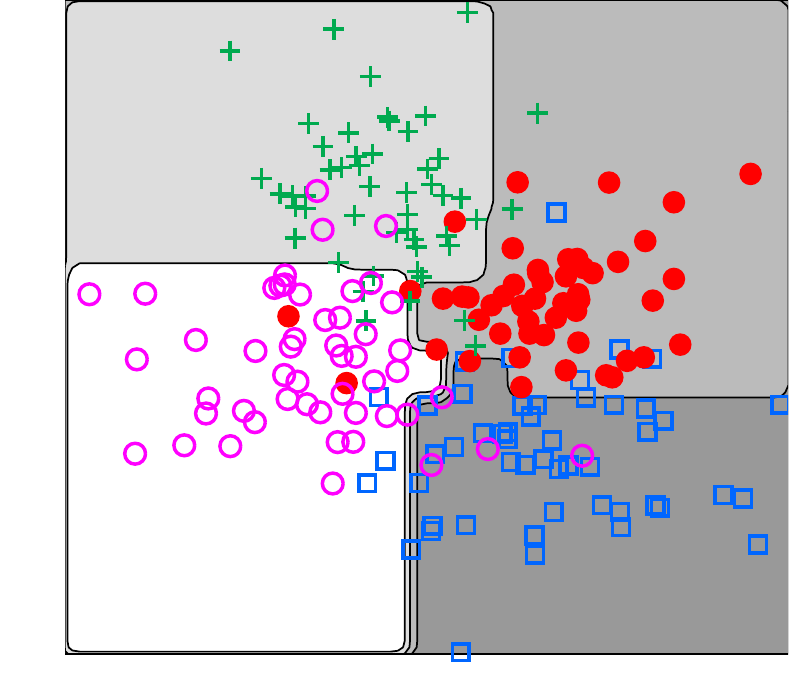}}
    \subfigure[AdaBoost.MO, 5000 iterations]{\includegraphics[width=0.3\textwidth,clip]{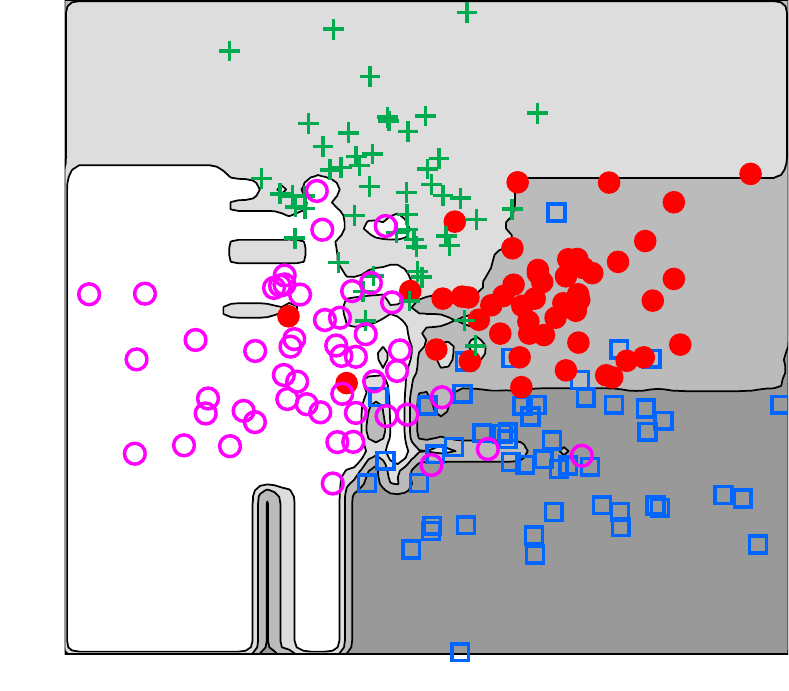}}
    \subfigure[AdaBoost.ECC, 5000 iterations]{\includegraphics[width=0.3\textwidth,clip]{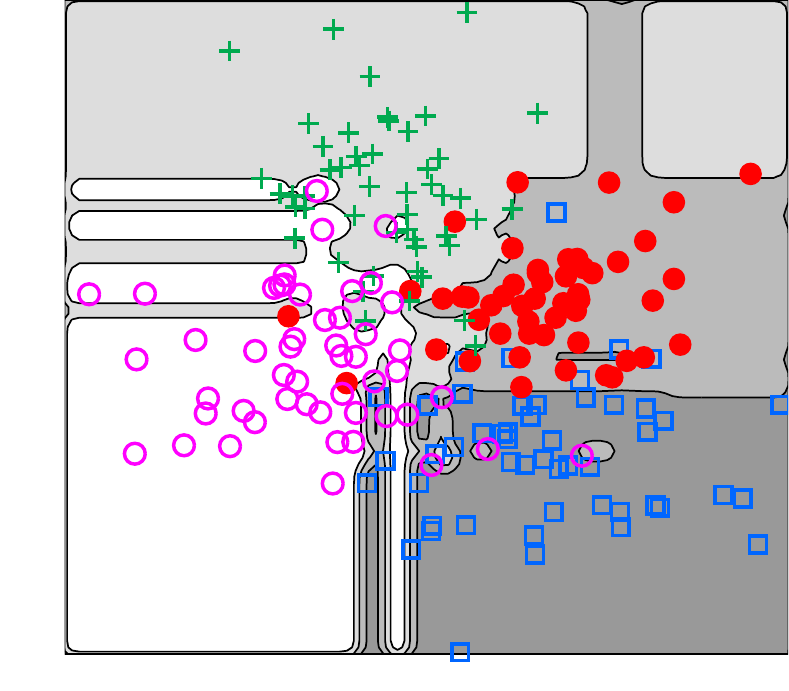}}
    \subfigure[\mbh, 5000 iterations]{\includegraphics[width=0.3\textwidth,clip]{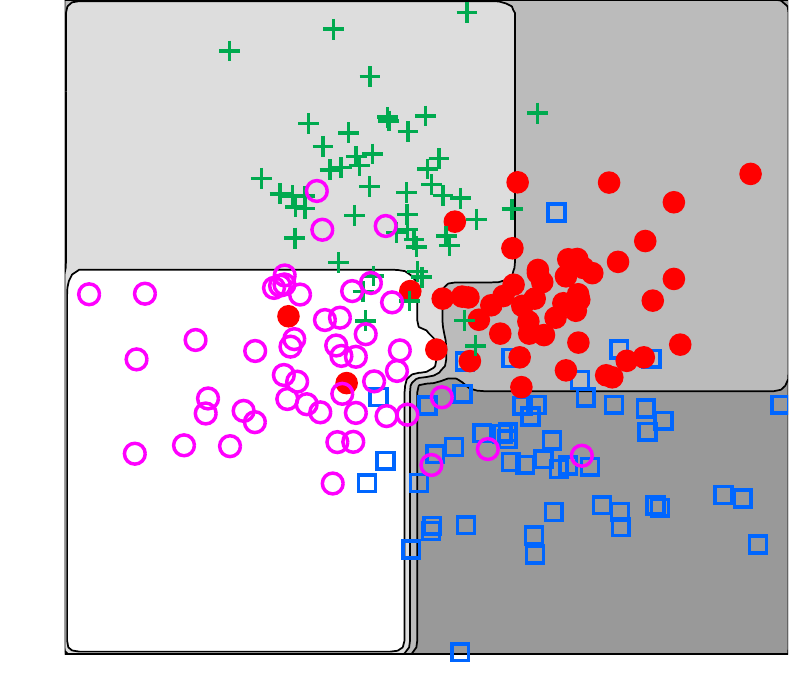}}
    \subfigure[\mbe, 5000 iterations]{\includegraphics[width=0.3\textwidth,clip]{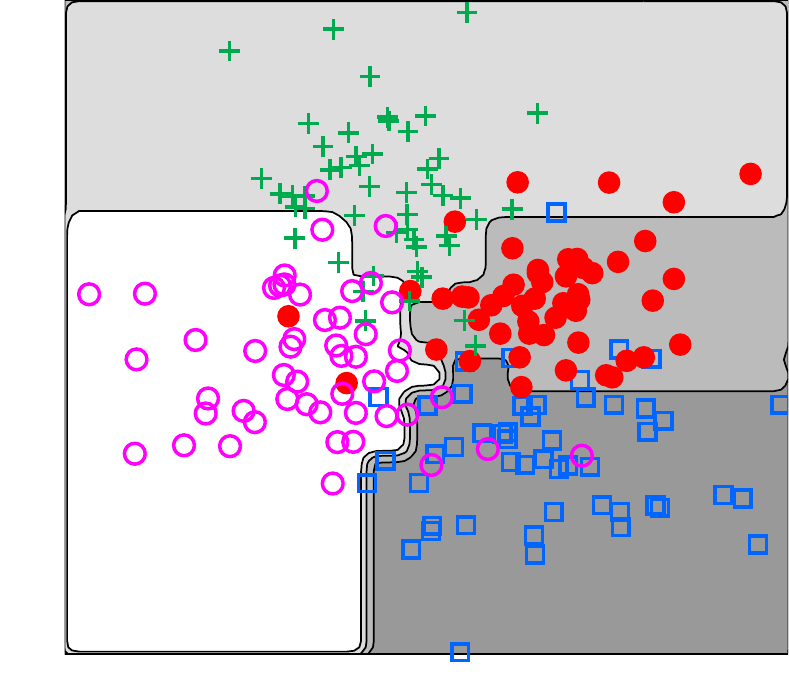}}
    \end{center}
       \caption{
            Figure (a) shows a toy data set, which contains 4 classes
            and a total of 200 sample points.
            Boosting algorithms are trained on this set using decision stumps.
            Plots (b)-(e) illustrate the decision boundaries made by
            (b) AdaBoost.MO, (c) AdaBoost.ECC, (d) \mbh and (e) \mbe
            with the number of training iterations being $100$.
            For comparison, plots (f)-(i) illustrate the decision boundaries of these algorithms,
            respectively,
            when the number of iterations is $5000$.
            (f) and (g) apparently suffer from over-fitting.
        }
    \label{figCVPR11:toy}
    \end{figure*}

    \begin{figure}[t]
    \begin{center}
    \includegraphics[width=0.45\linewidth]{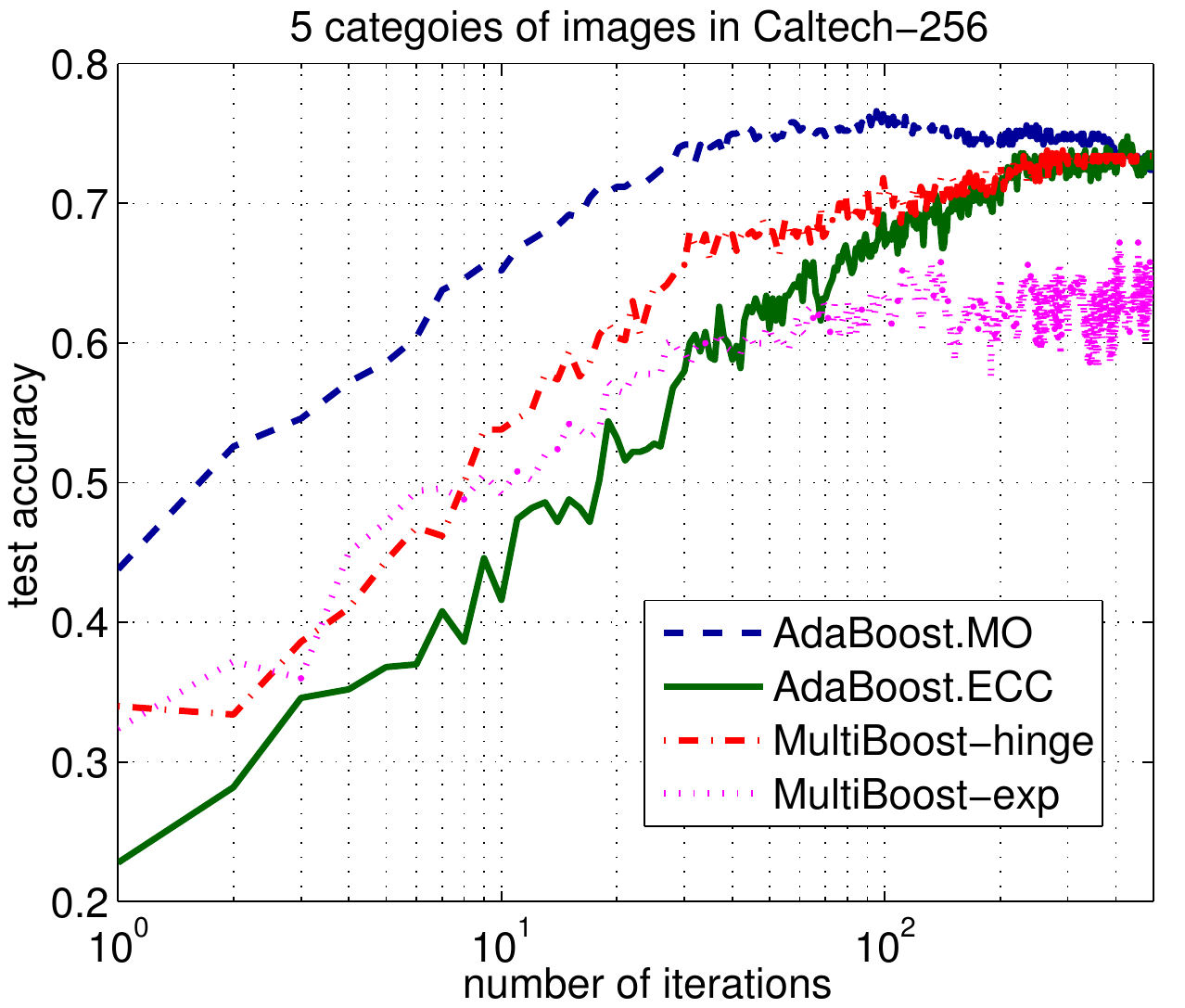}
    \end{center}
       \caption{
        Test accuracy curves of four boosting algorithms on 5 categories of Caltech-256 images.
        The weak classifiers are decision stumps and the number of training iterations is $500$.
        The average results of $10$ runs are reported.
        Each run randomly selects $75\%$ data for training and the other $25\%$ for test.
       }
    \label{figCVPR11:caltech}
    \end{figure}

    \paragraph{UCI data sets}
    Next we test our algorithms on $7$ data sets collected from UCI repository.
    Samples are randomly divided into $75\%$ for training and $25\%$ for test,
    no matter whether there is a pre-specified split or not. Each data set is run $10$ times
    and the average results of test error are reported in Table \ref{table:uci}.
    The maximum number of iterations is set to $500$. Almost all the algorithms converge
    before the maximum iteration.
    Again the regularization parameter is determined by 5-fold cross validation.

    Table \ref{table:uci} reports the results. The conclusion that we can draw on
    this experiment is: 1) Overall, all the algorithms achieve comparable accuracy.
    2) our algorithms are slightly better in terms of generalization ability
    than the other two on $5$ out of $7$ data sets.
    \mbe outperforms others in  $4$ data sets. 3)
    Also note that the performance \mbh is more stable than \mbe, which may be
    due to the fact that the hinge loss is less sensitive to noise than
    the exponential loss.

    \begin{table*}[ht!]
    \begin{center}
    \begin{tabular}{l|c |c | c | c }
    \hline
    dataset &AdaBoost.MO &AdaBoost.ECC & \mbh & \mbe \\
    \hline\hline
        thyroid     &0.005$\pm$0.001 &0.005$\pm$0.001 &0.005$\pm$0.001
				    &\textbf{0.004$\pm$0.001}  \\
	    dna         &0.059$\pm$0.005 &0.064$\pm$0.005 &\textbf{0.057$\pm$0.007}
				    &0.061$\pm$0.004 \\
	    wine        &0.036$\pm$0.025 &0.034$\pm$0.029 &0.032$\pm$0.018
				    &\textbf{0.030$\pm$0.029}  \\
	    iris        &0.062$\pm$0.017 &0.073$\pm$0.021 &0.068$\pm$0.022
				    &\textbf{0.057$\pm$0.022}  \\
	    glass       &\textbf{0.232$\pm$0.047}  &0.242$\pm$0.053 &0.234$\pm$0.046
				    &0.315$\pm$0.086 \\
	    svmguide2   &0.213$\pm$0.039 &0.214$\pm$0.030 &0.222$\pm$0.052
				    &\textbf{0.206$\pm$0.040}  \\
        svmguide4   &0.192$\pm$0.018 &\textbf{0.191$\pm$0.018}  &0.207$\pm$0.018
				    &0.214$\pm$0.027 \\
    \hline
    \end{tabular}
    \end{center}
    \caption{
        Test errors of four boosting algorithms on UCI data sets.
        The average results of 10 repeated tests are reported.
        Weak classifiers are decision stumps.
        \mbe is the best on 4 out of 7 data sets.
        }
    \label{table:uci}
    \end{table*}

\paragraph{Handwriting digits recognition}
    To further examine the effectiveness of our algorithms,
    We have conducted another experiment on a handwritten digits data set,
    which is also from UCI repository. The original data set contains $5620$ digits
    written by a total of $43$ people on $32\times32$ bitmaps.
    Then the bitmaps are divided into $4\times4$ non-overlapping blocks,
    and an $8\times8$ descriptor is generated by calculating the sum of $0$-$1$ pixels in each block.
    For ease of exposition, only $3$ distinct digits of ``1'', ``6'' and ``9'' are
    chosen for classification.
    Figure \ref{figCVPR11:opt}(a) illustrates the mean images of their training data examples
    of the three digits.
    The index of each block (feature) is also printed on Figure \ref{figCVPR11:opt}(a)
    for the convenience of exposition.

    We train multi-class boosting on this data set.
    The number of maximum training iterations is set to $500$.
    $75\%$ data are used for training, and the rest for test.
    Again $ 5 $-fold cross validation is used.
    We still use decision stumps as the weak classifiers.
    Boosting learning with decision stumps implies that we select features at the same time.
    In other words, decision stumps select most discriminative blocks
    for classifying these digits.
    The four compared algorithms have similar performances on this test
    with nearly $98\%$ test accuracy.
    We plot the models of AdaBoost.ECC, \mbh and \mbe in Figures \ref{figCVPR11:opt}(b)-(d).
    AdaBoost.MO can be hardly illustrated as it involves a multi-dimensional
    coding scheme.
    Notice that a decision stump divides the value range of the feature into two parts,
    on which there are necessarily two different attributions,
    we use red circles and green crosses to represent the positive and negative parts.
    For example, if a decision stump on the $10$-th feature is $x_{10} > \tau$
    and assigns a set of weights $\{0.5,0.2,0.8\}$ to three labels,
    we mark $10$-th block in the third digit image with a red circle,
    and $10$-th block in the second digit with a green cross;
    if the stump is $x_{10} < \tau$ with the same weights,
    we do the opposite marks. In other words,
    red circles indicate the decision stumps should take bigger values on these blocks,
    while green crosses indicate these classifiers should take some values as small as possible.
    The width of a mark stands for the minimal margin defined in Equation \eqref{EQCVPR11:margin},
    that is, in the $i$-th digit, the width is proportional to
    $h(\x) \w_{y_i} - \max\{ h(\x)\w_r \}$, $\forall r \neq y_i$.
    Some features may be selected multiple times,
    which divide the value range into several segments.
    In this case, we neglect all the middle parts.


\begin{figure*}[tp!]
\setlength{\fboxsep}{1pt}
\setlength{\fboxrule}{1pt}
 \colorbox{green}{\includegraphics[width = 0.52cm,height = 0.85cm]{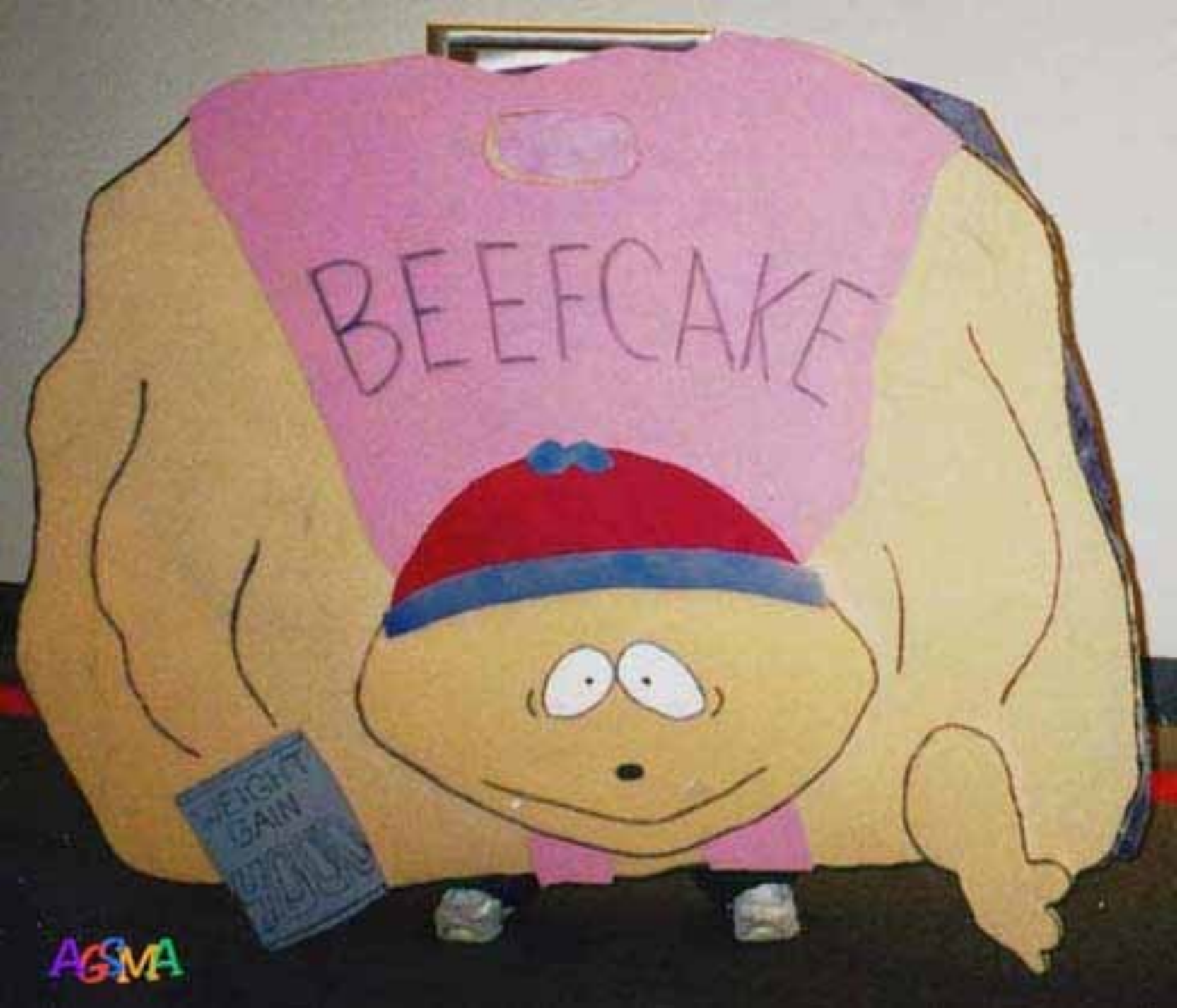}}
 \colorbox{green}{\includegraphics[width = 0.52cm,height = 0.85cm]{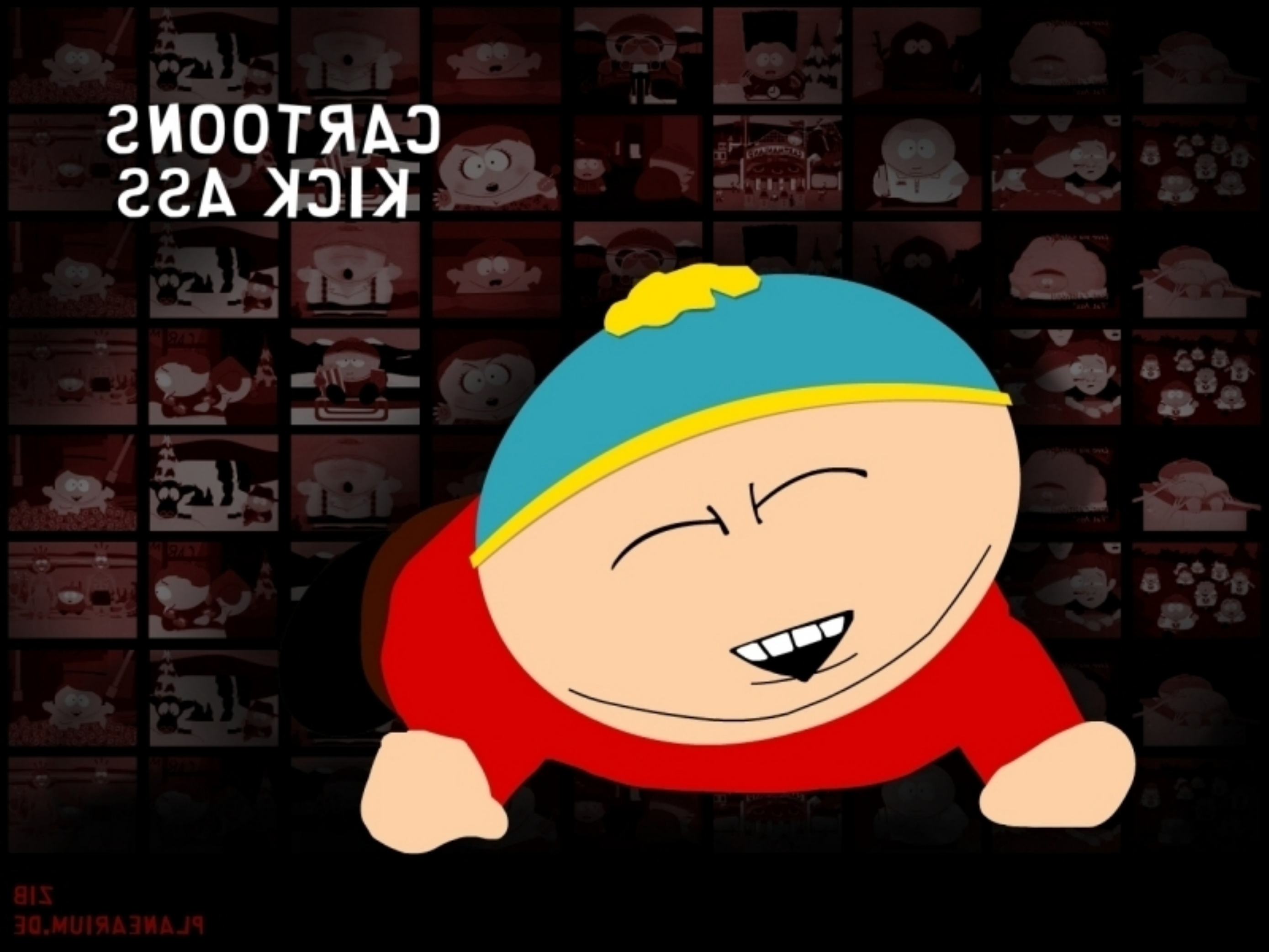}}
 \colorbox{green}{\includegraphics[width = 0.52cm,height = 0.85cm]{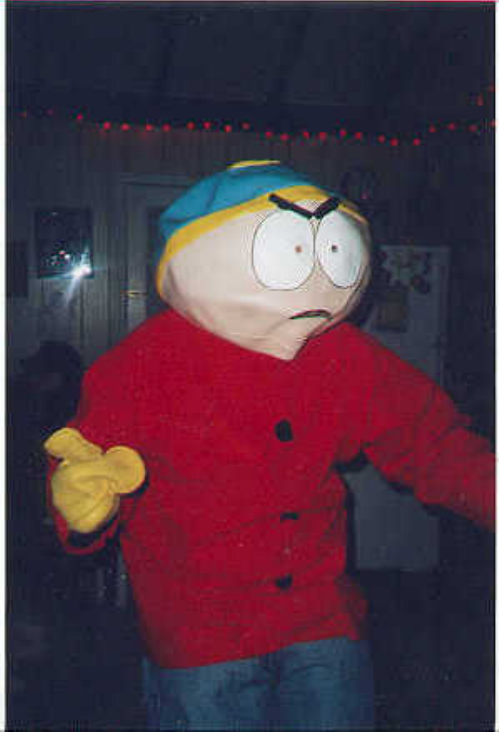}}
 \colorbox{green}{\includegraphics[width = 0.52cm,height = 0.85cm]{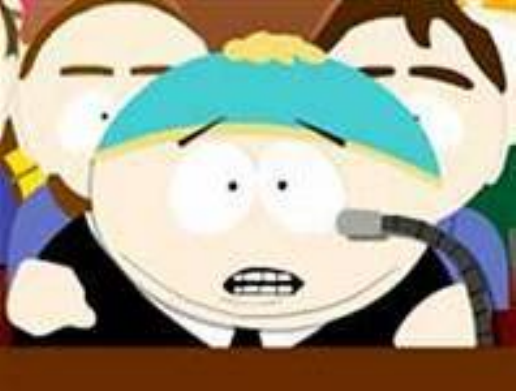}}
 \includegraphics[width = 0.25cm,height = 0.85cm]{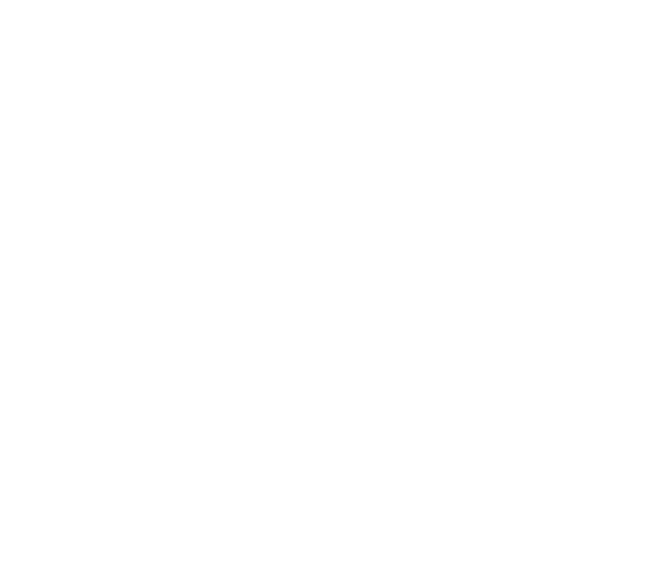}
 \colorbox{green}{\includegraphics[width = 0.52cm,height = 0.85cm]{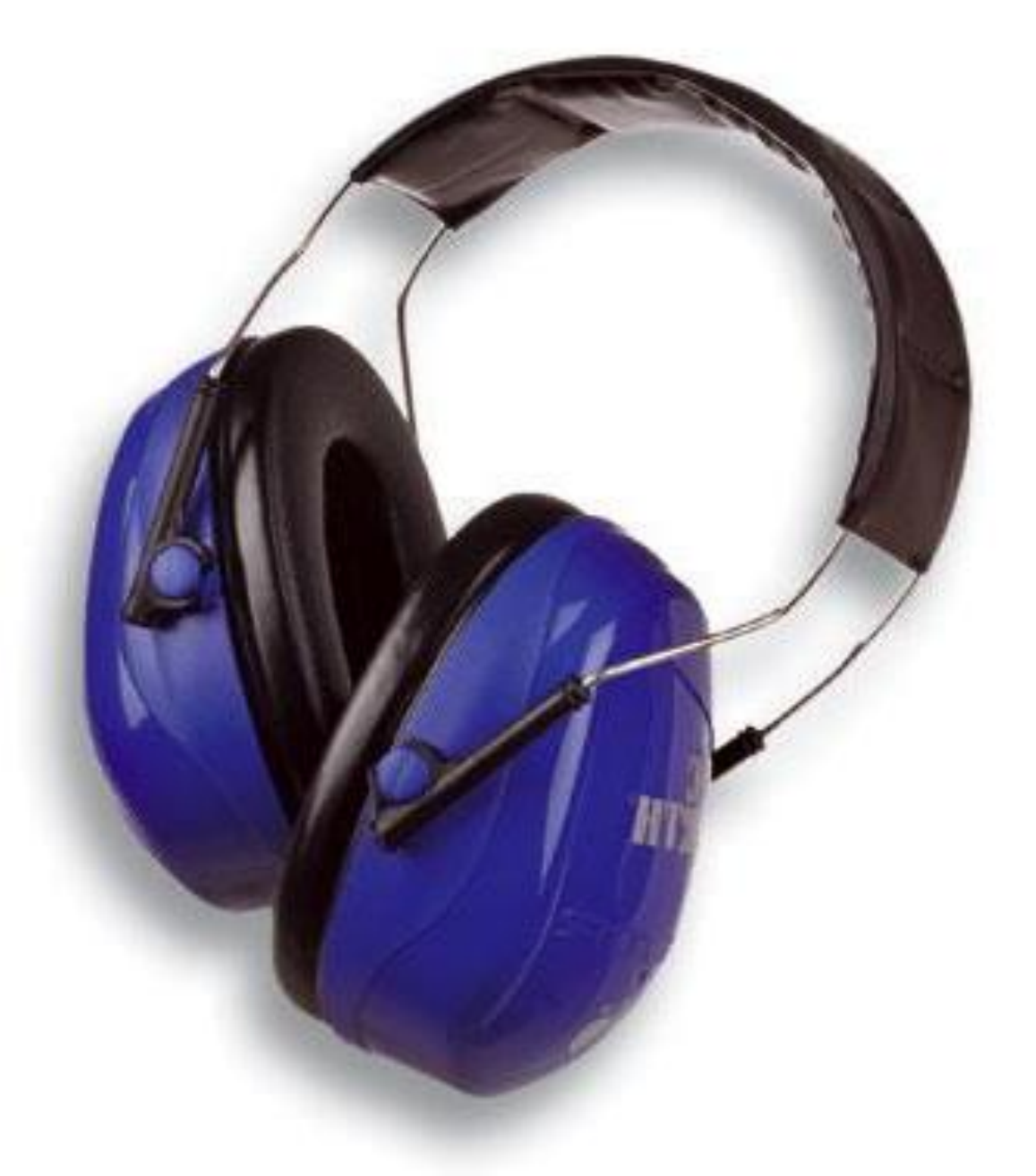}}
 \colorbox{green}{\includegraphics[width = 0.52cm,height = 0.85cm]{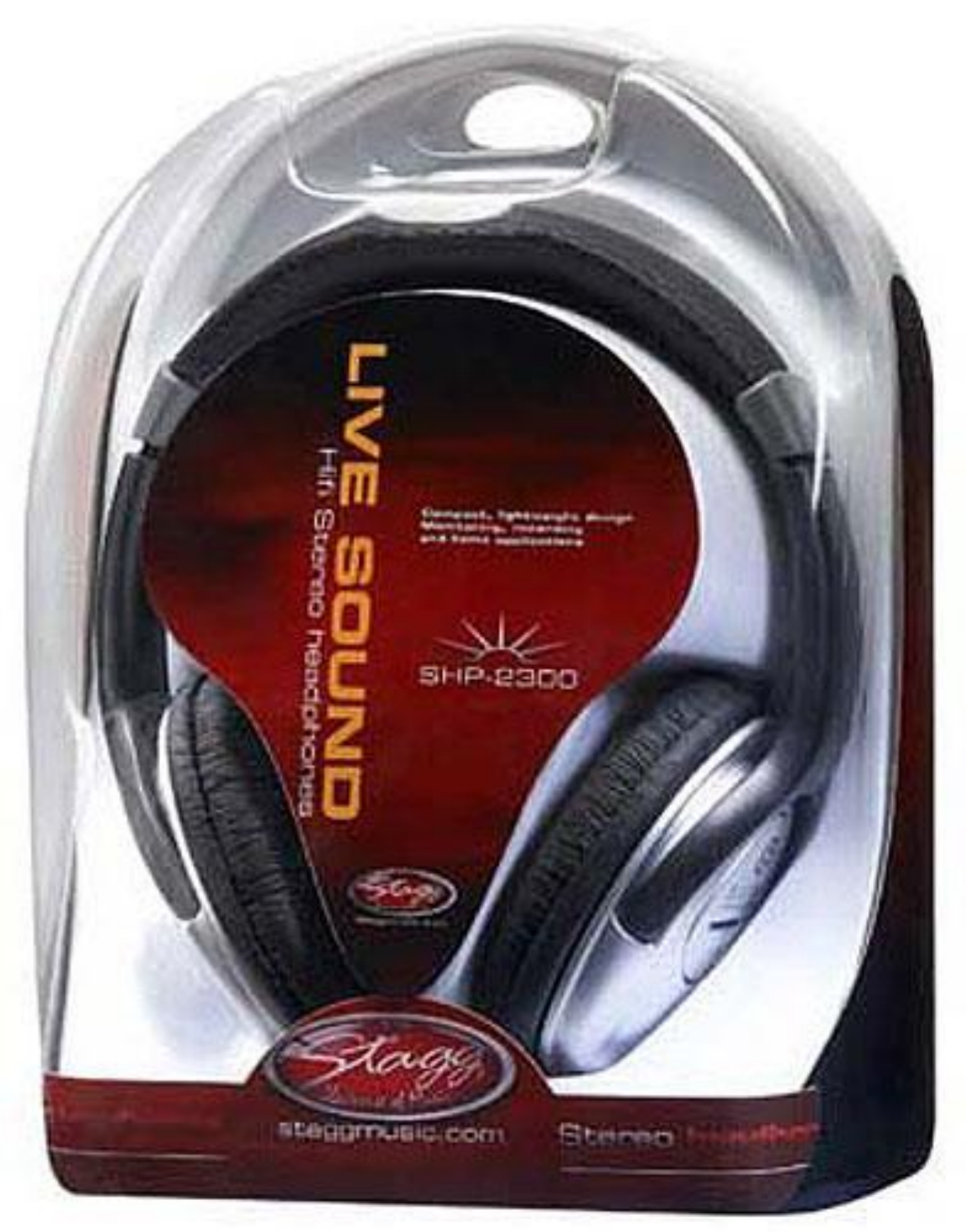}}
 \colorbox{green}{\includegraphics[width = 0.52cm,height = 0.85cm]{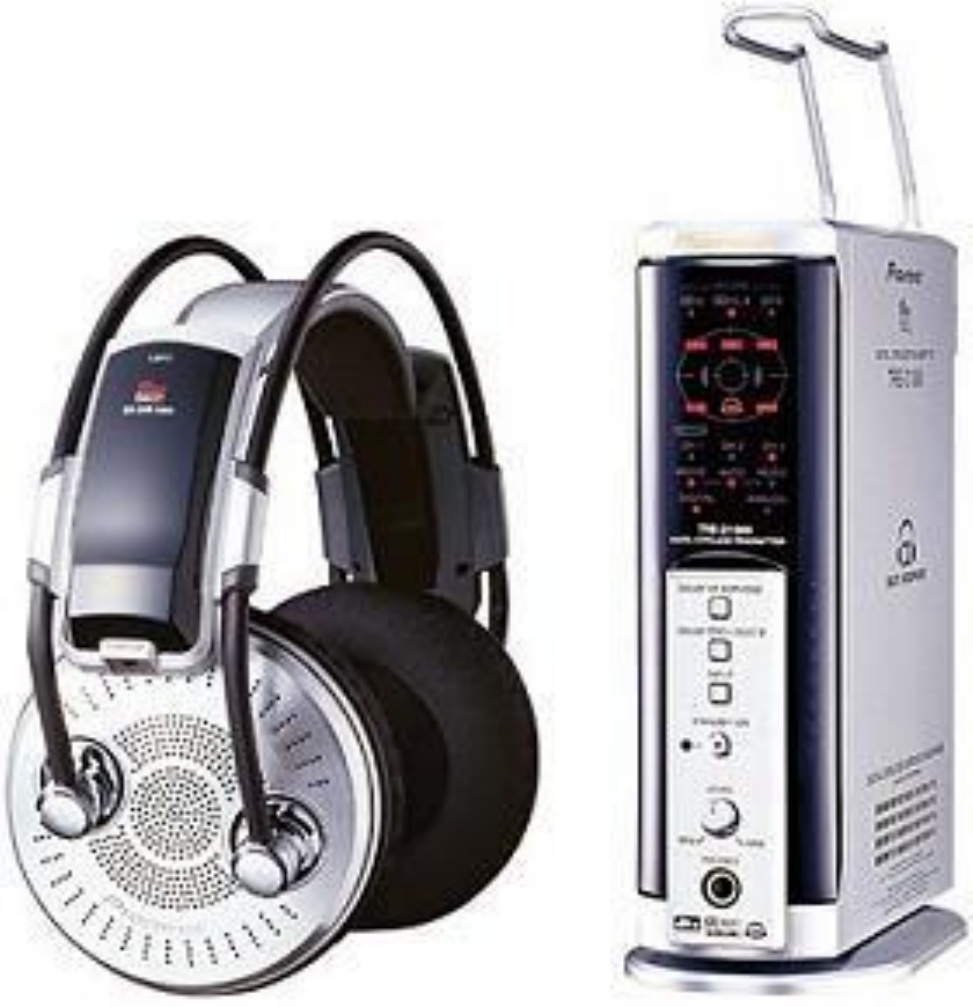}}
 \colorbox{green}{\includegraphics[width = 0.52cm,height = 0.85cm]{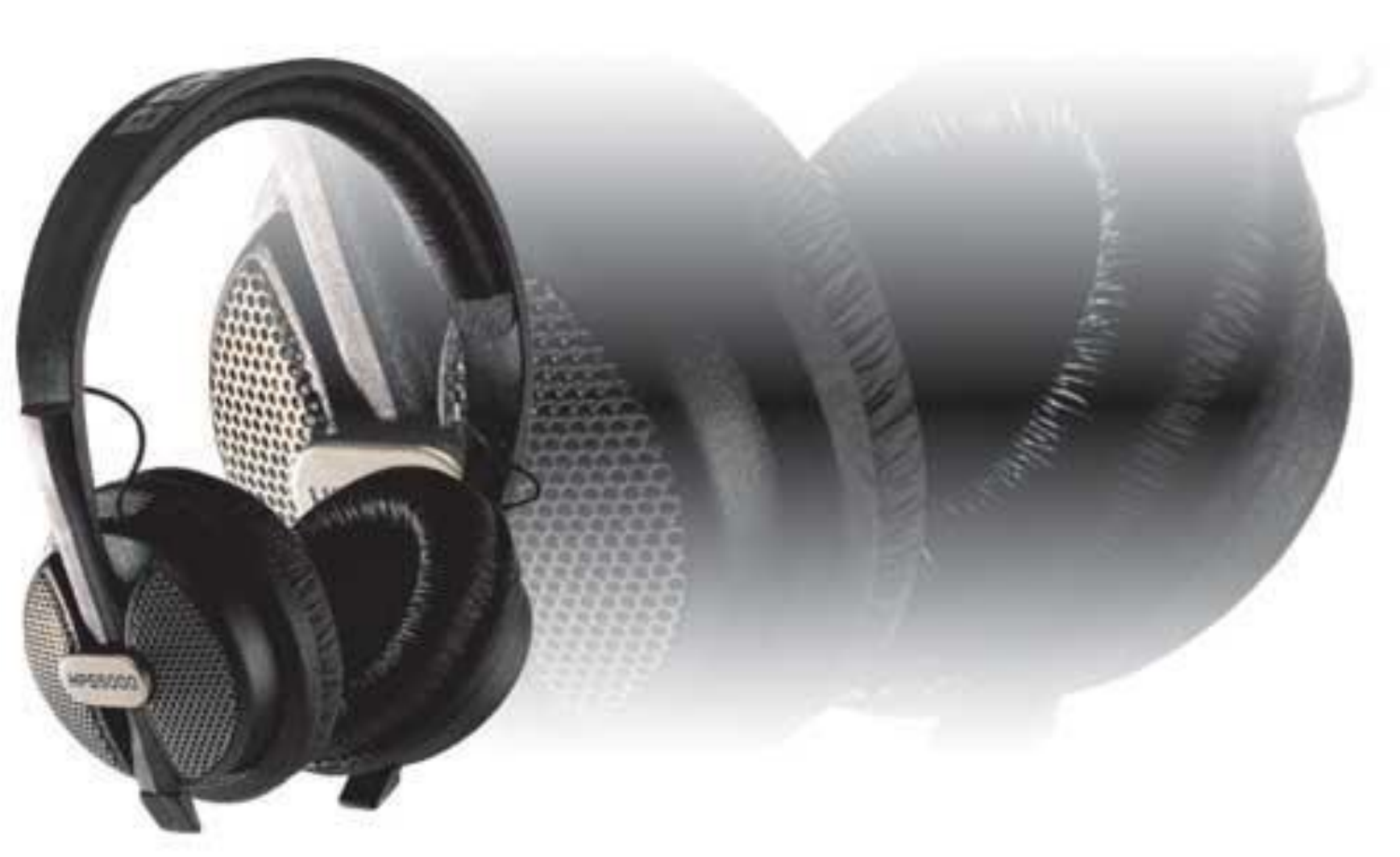}}
 \colorbox{green}{\includegraphics[width = 0.52cm,height = 0.85cm]{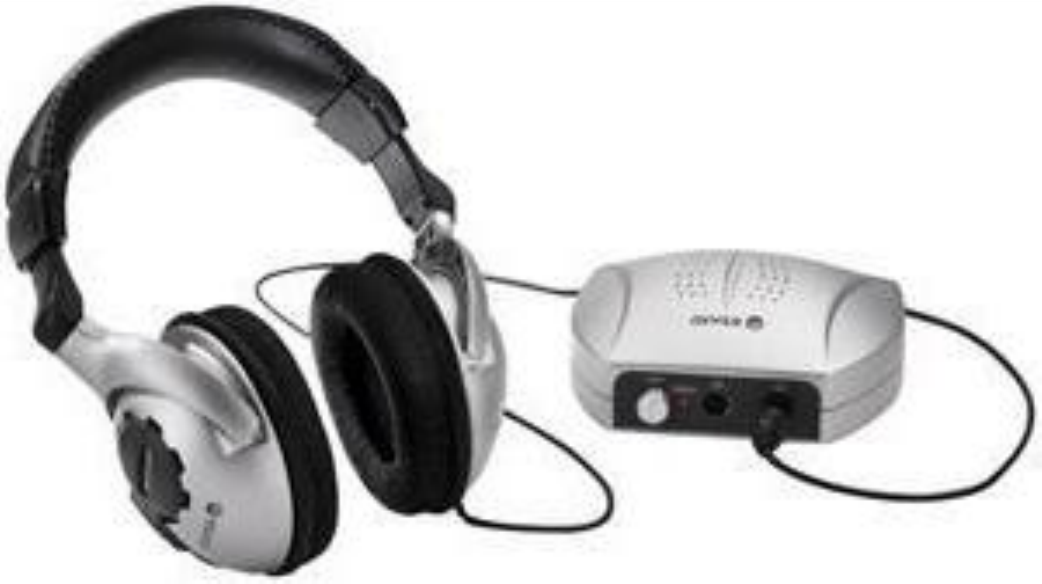}}
 \includegraphics[width = 0.25cm,height = 0.85cm]{white}
 \colorbox{green}{\includegraphics[width = 0.52cm,height = 0.85cm]{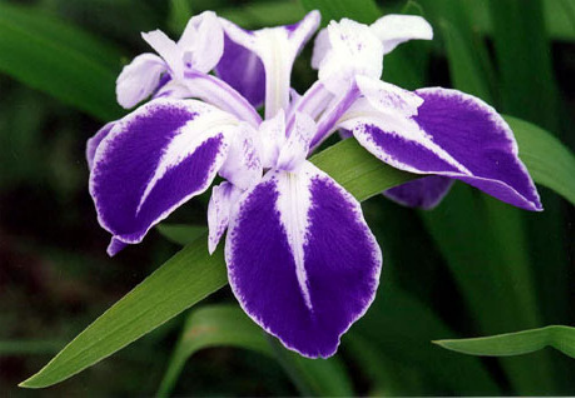}}
 \colorbox{green}{\includegraphics[width = 0.52cm,height = 0.85cm]{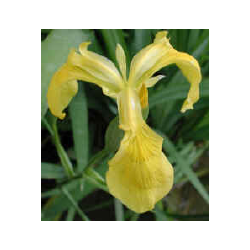}}
 \colorbox{green}{\includegraphics[width = 0.52cm,height = 0.85cm]{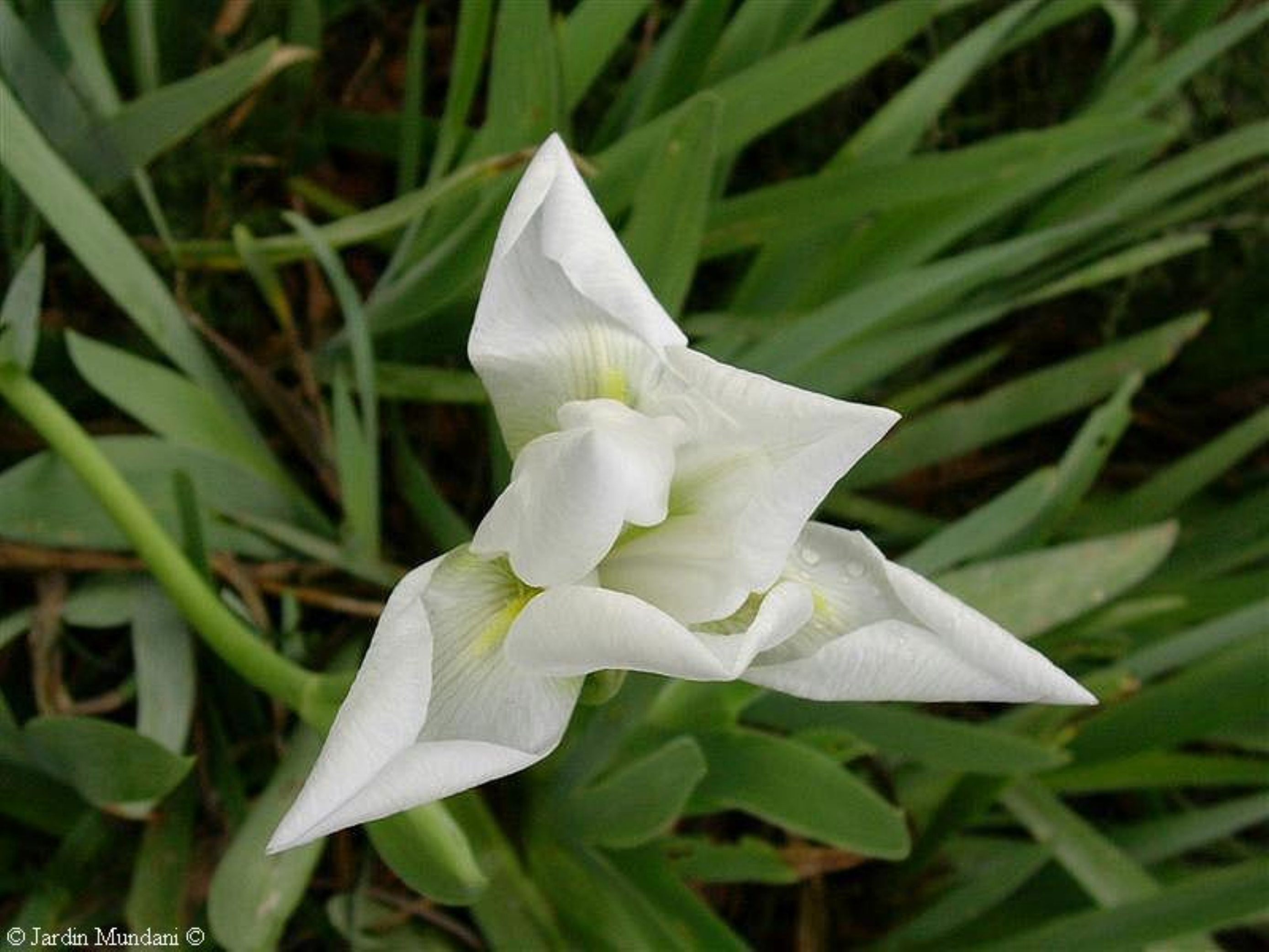}}
 \colorbox{green}{\includegraphics[width = 0.52cm,height = 0.85cm]{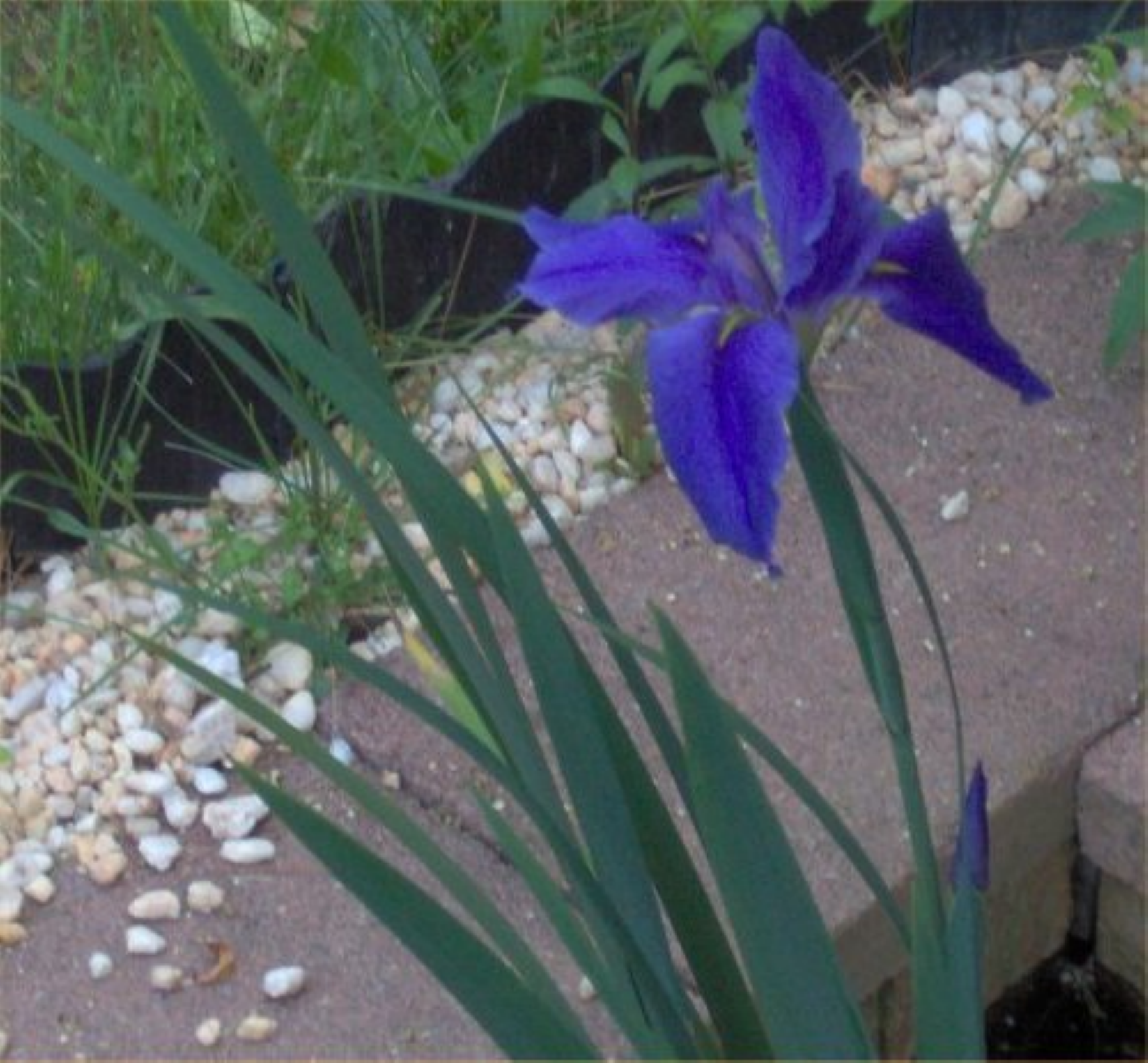}}
 \includegraphics[width = 0.25cm,height = 0.85cm]{white}
 \colorbox{green}{\includegraphics[width = 0.52cm,height = 0.85cm]{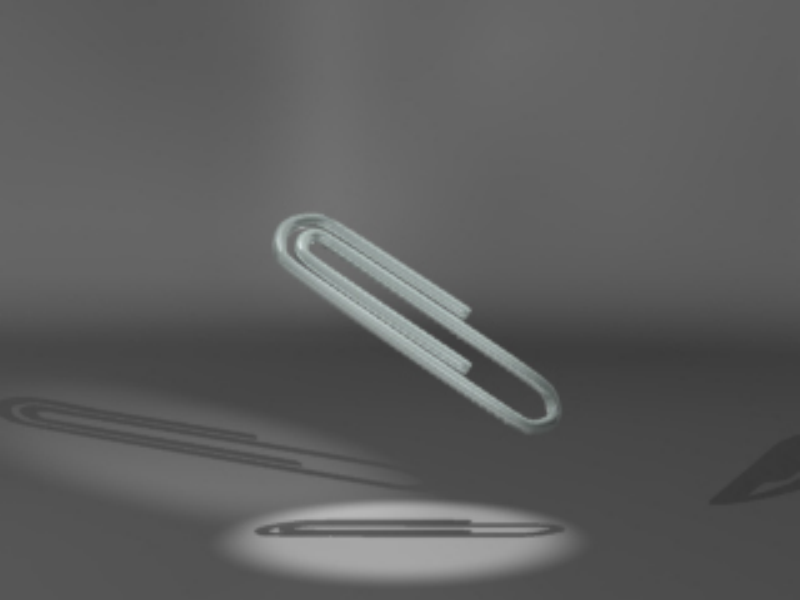}}
 \colorbox{green}{\includegraphics[width = 0.52cm,height = 0.85cm]{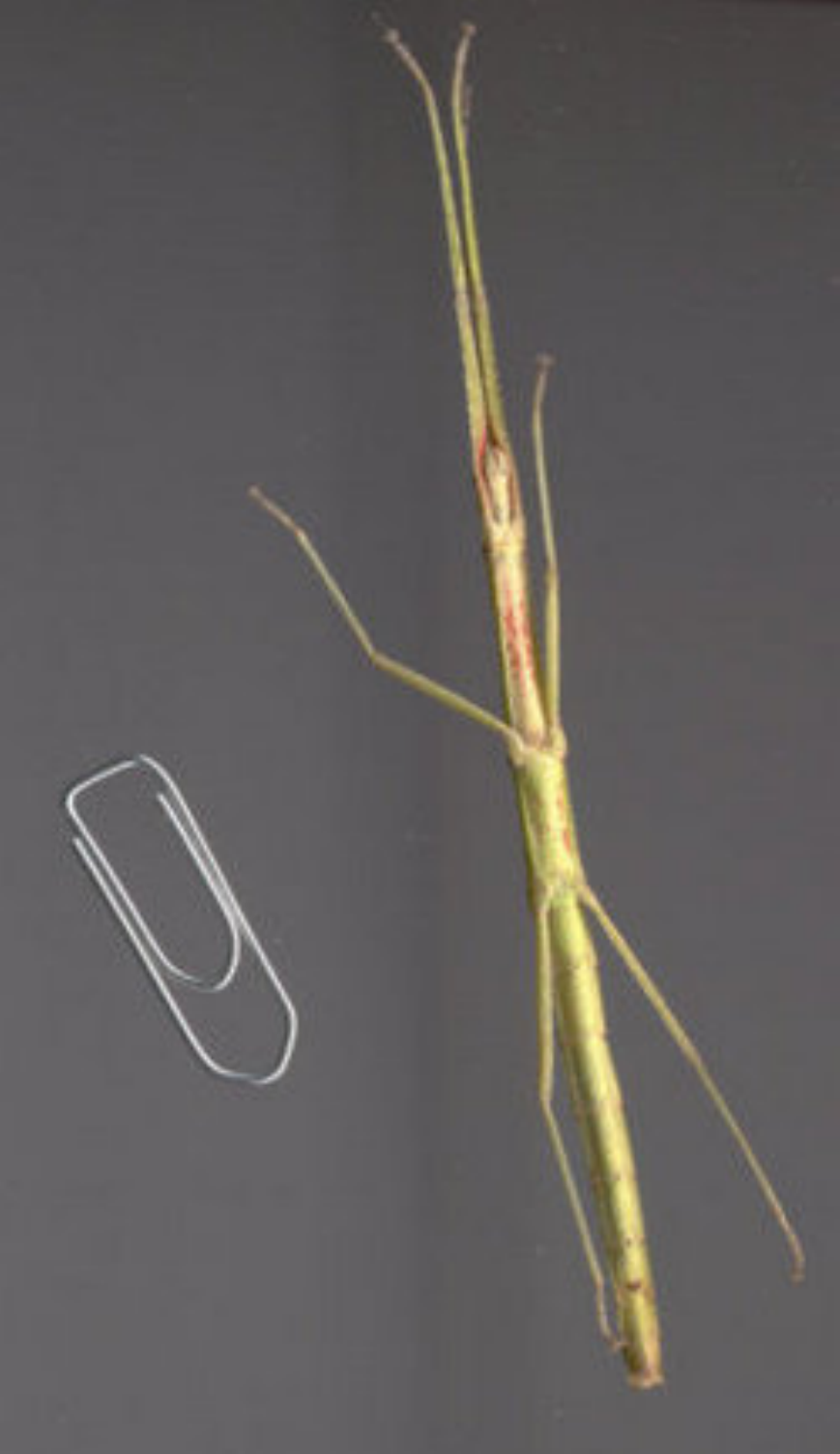}}
 \includegraphics[width = 0.25cm,height = 0.85cm]{white}
 \colorbox{green}{\includegraphics[width = 0.52cm,height = 0.85cm]{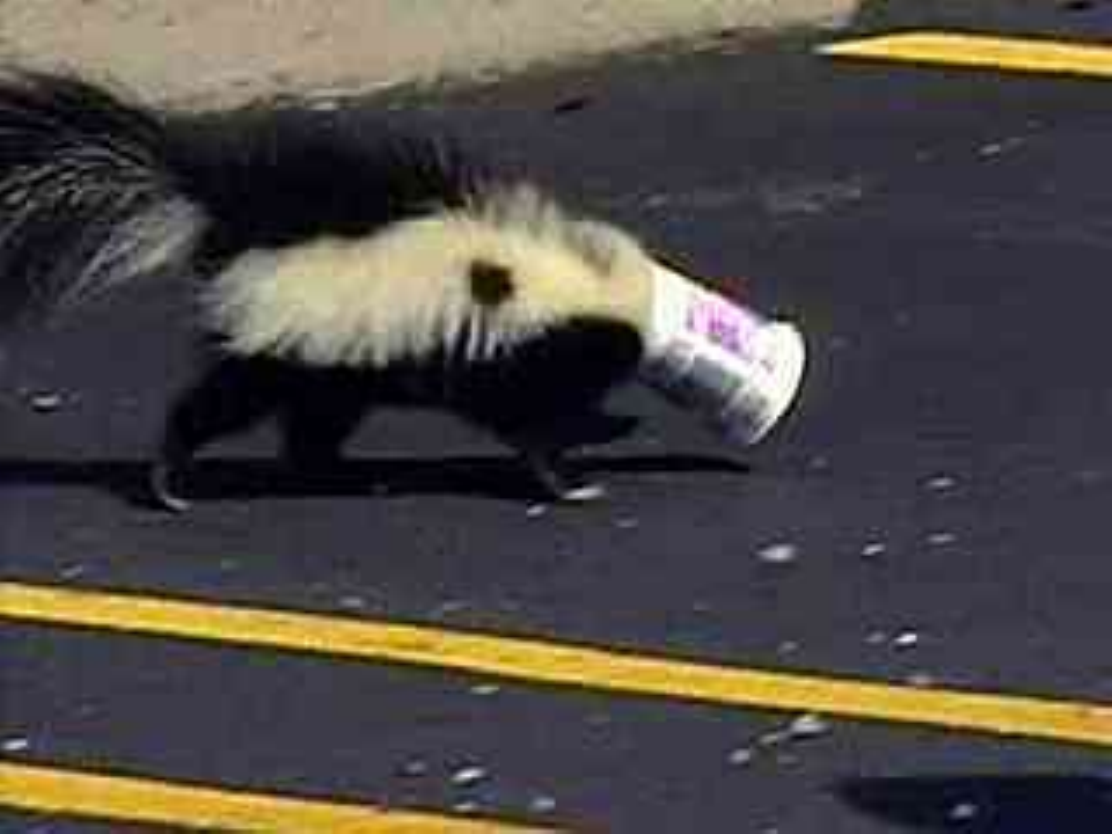}}
 \colorbox{green}{\includegraphics[width = 0.52cm,height = 0.85cm]{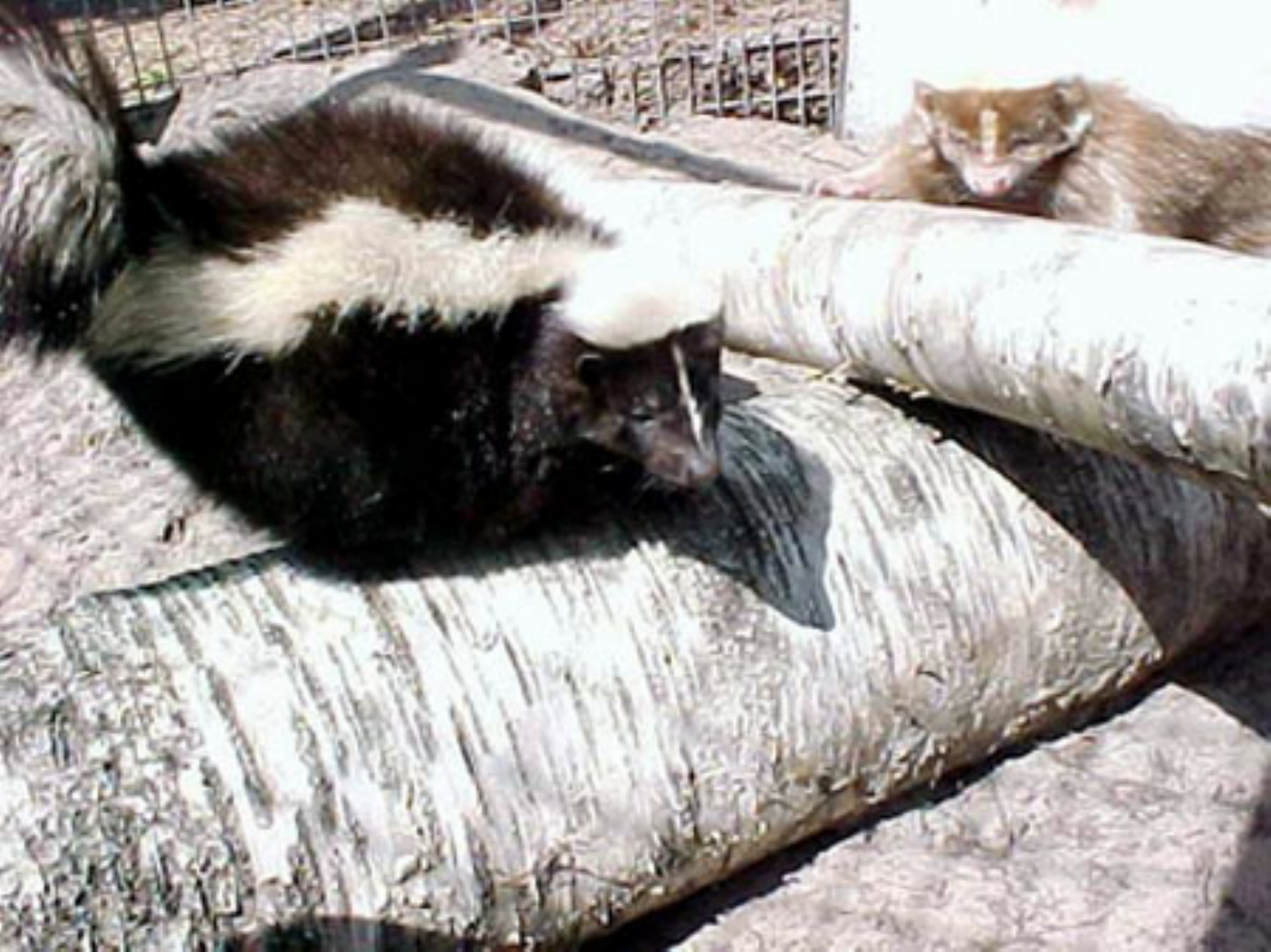}}
 \colorbox{green}{\includegraphics[width = 0.52cm,height = 0.85cm]{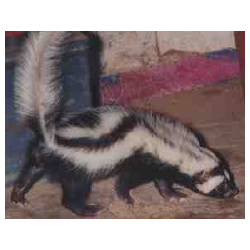}}
 \colorbox{green}{\includegraphics[width = 0.52cm,height = 0.85cm]{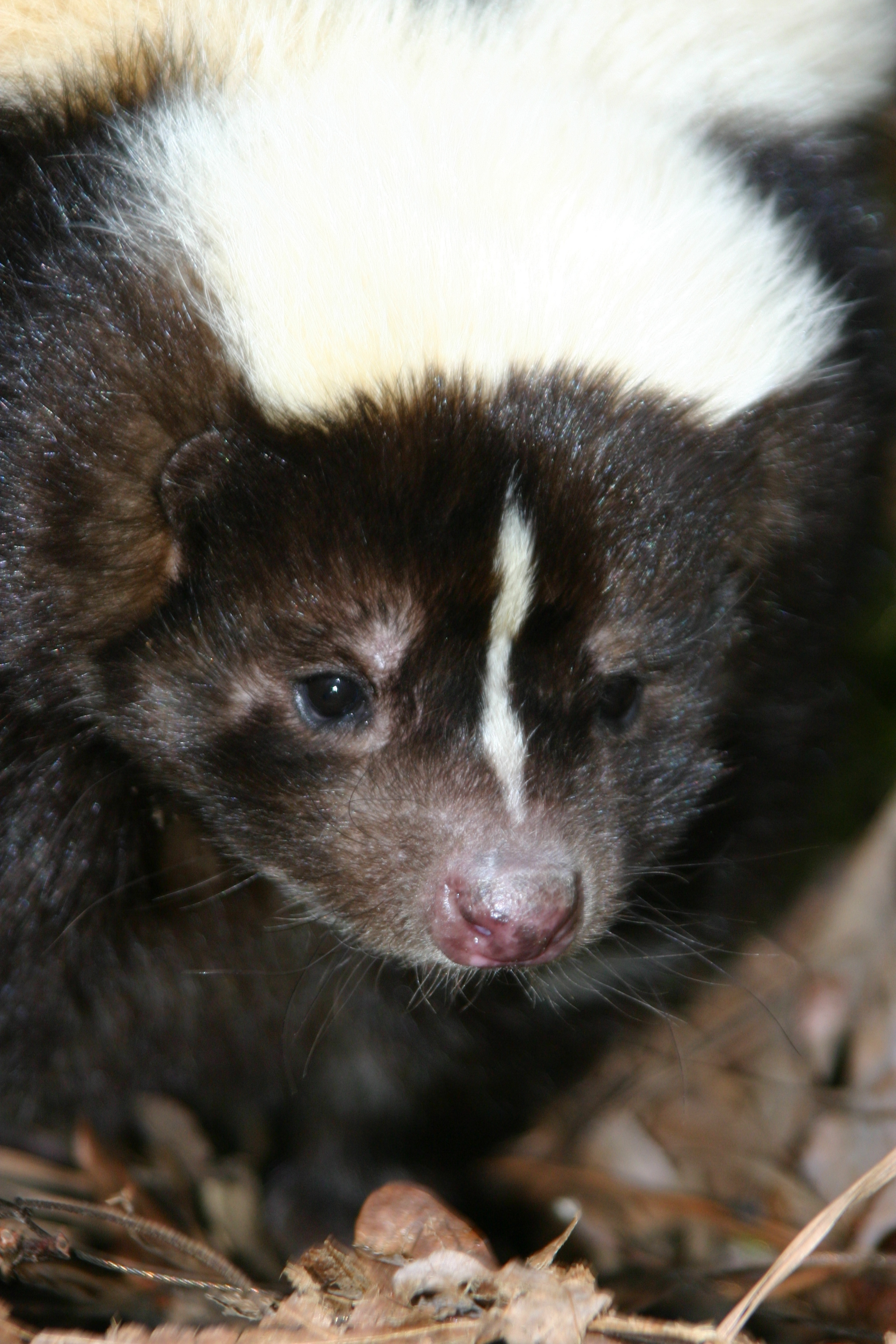}}
\colorbox{green}{\includegraphics[width = 0.52cm,height = 0.85cm]{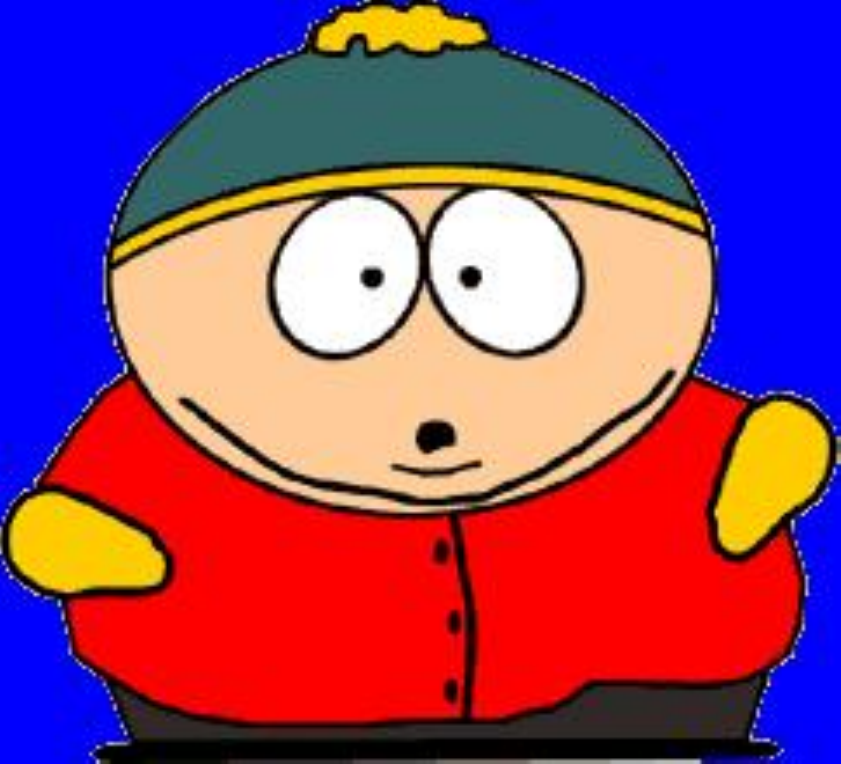}}
\colorbox{green}{\includegraphics[width = 0.52cm,height = 0.85cm]{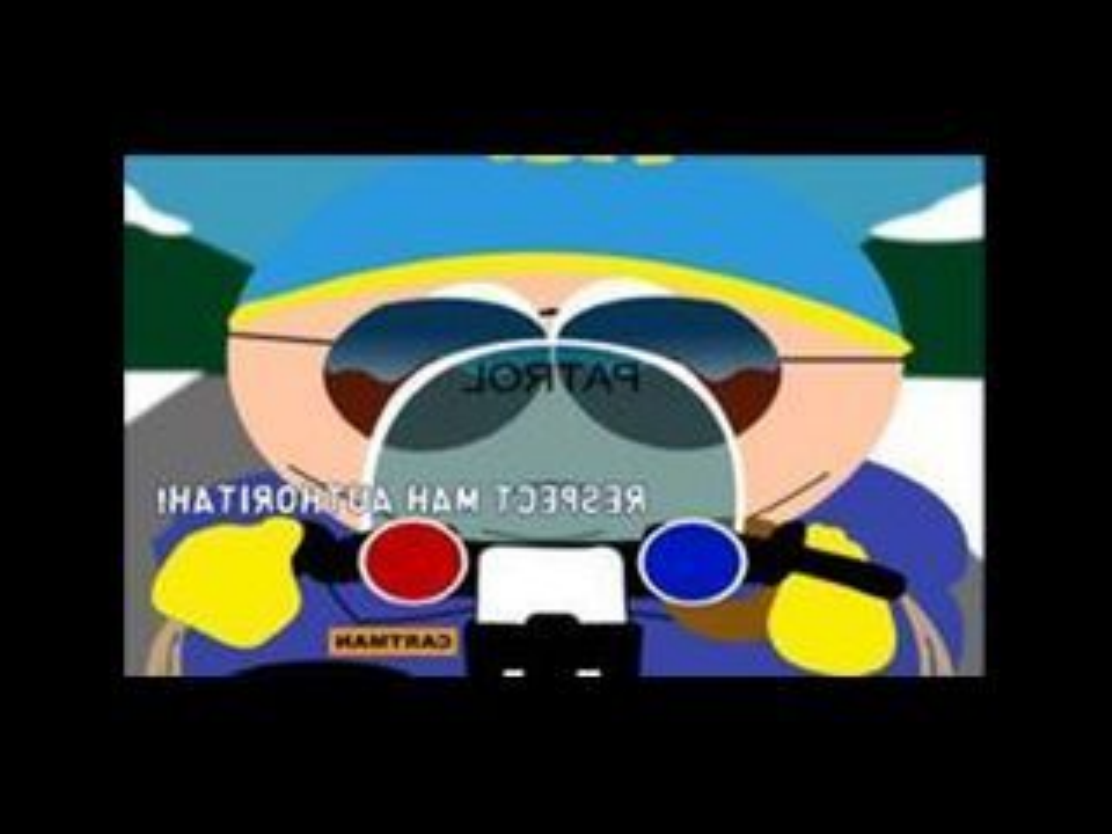}}
\colorbox{green}{\includegraphics[width = 0.52cm,height = 0.85cm]{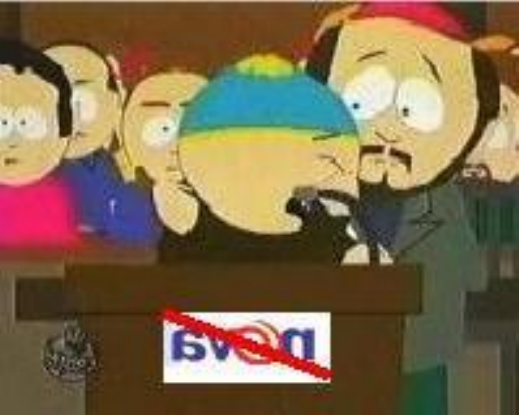}}
\colorbox{green}{\includegraphics[width = 0.52cm,height = 0.85cm]{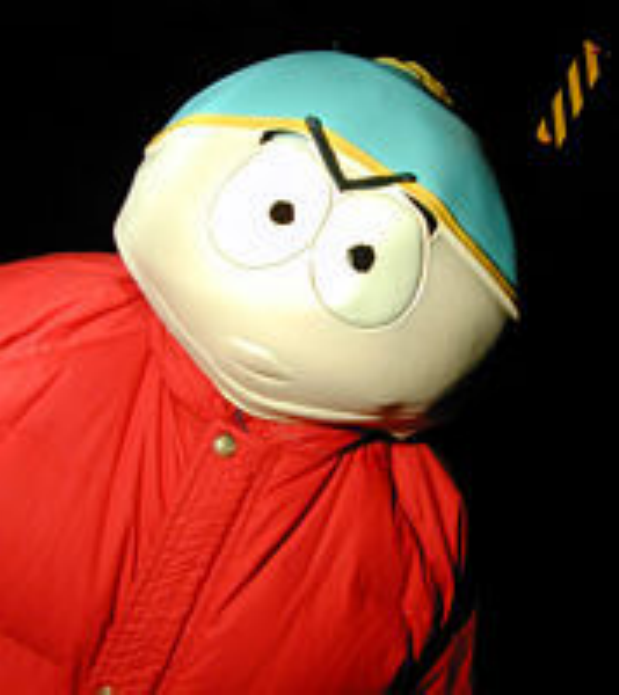}}
\includegraphics[width = 0.3cm,height = 0.85cm]{white}
\colorbox{green}{\includegraphics[width = 0.52cm,height = 0.85cm]{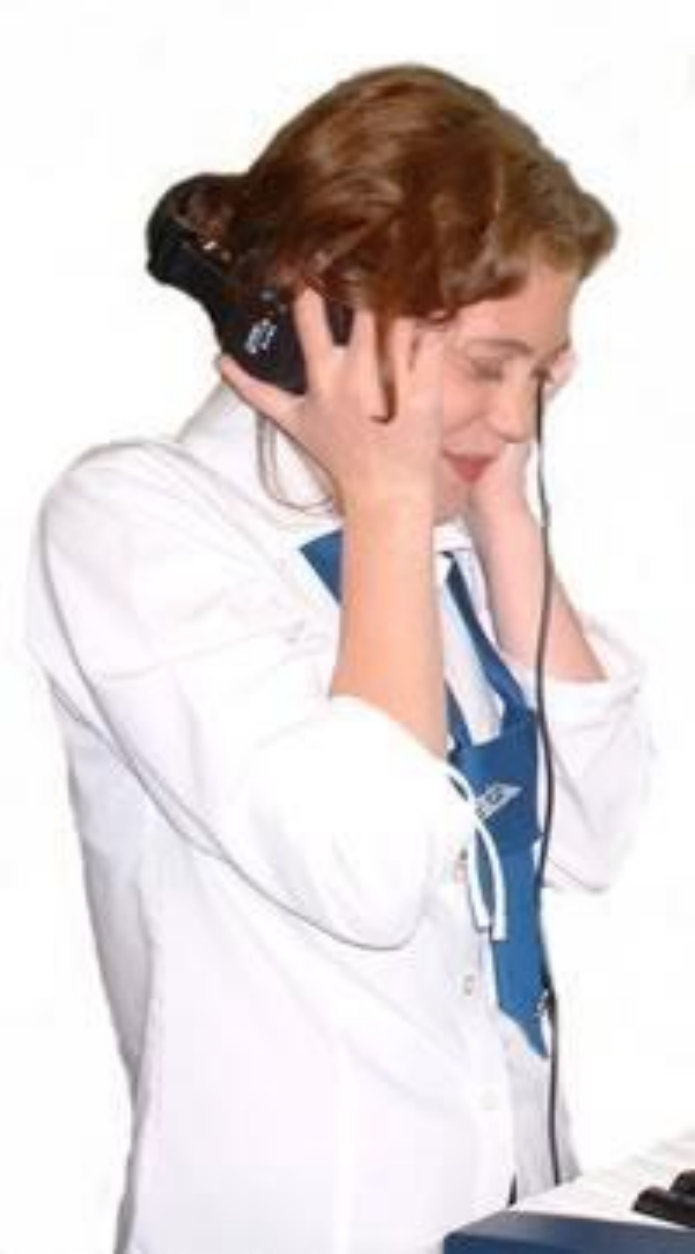}}
\colorbox{green}{\includegraphics[width = 0.52cm,height = 0.85cm]{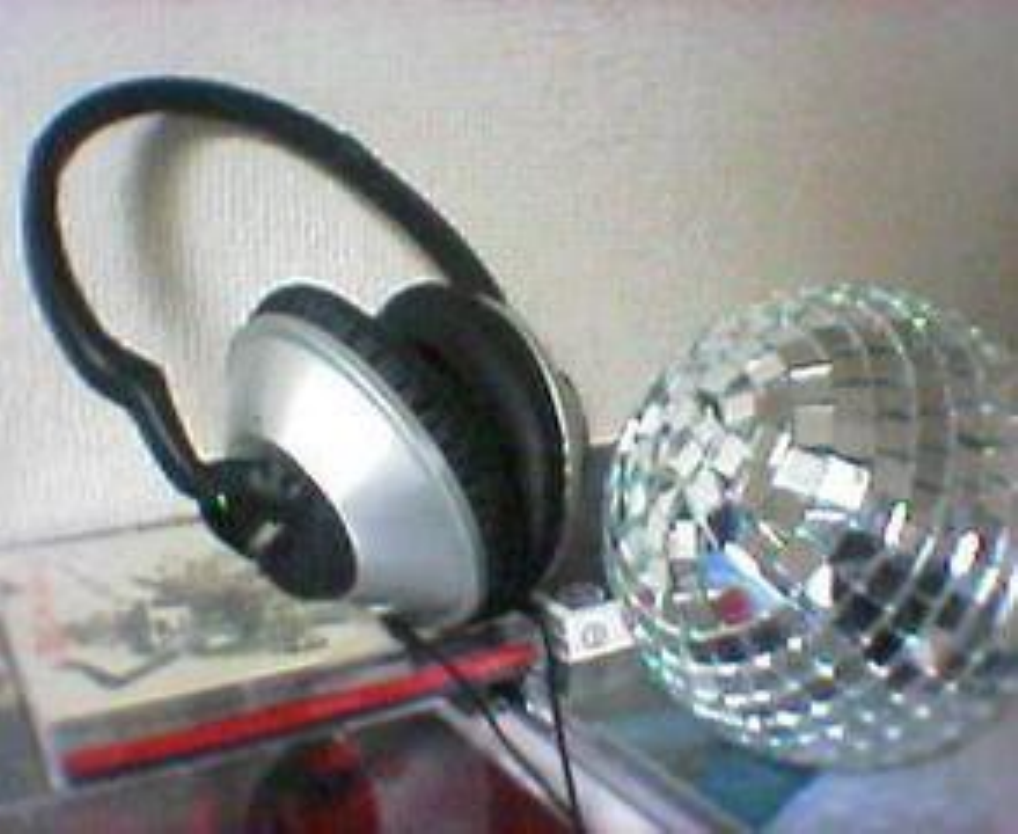}}
\colorbox{green}{\includegraphics[width = 0.52cm,height = 0.85cm]{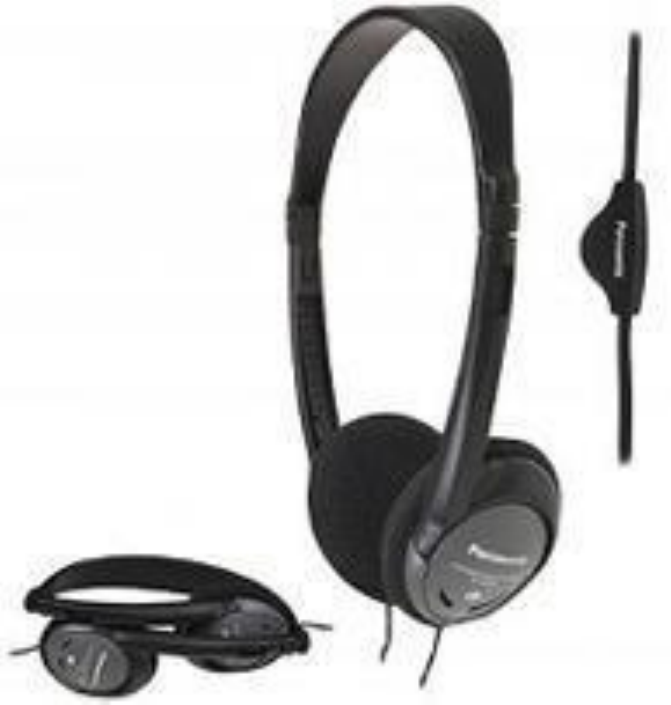}}
\colorbox{green}{\includegraphics[width = 0.52cm,height = 0.85cm]{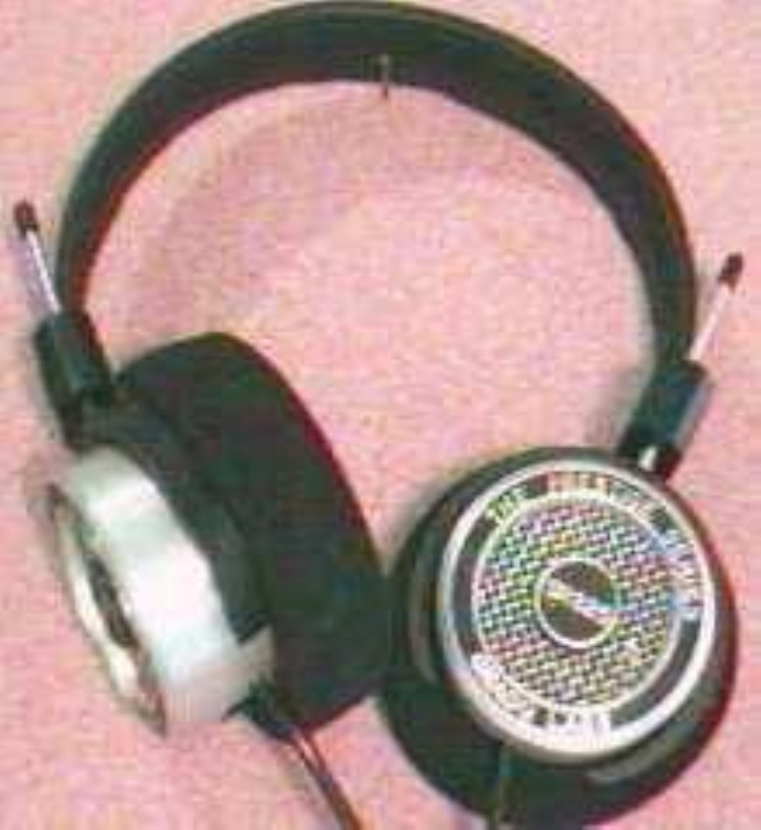}}
\colorbox{green}{\includegraphics[width = 0.52cm,height = 0.85cm]{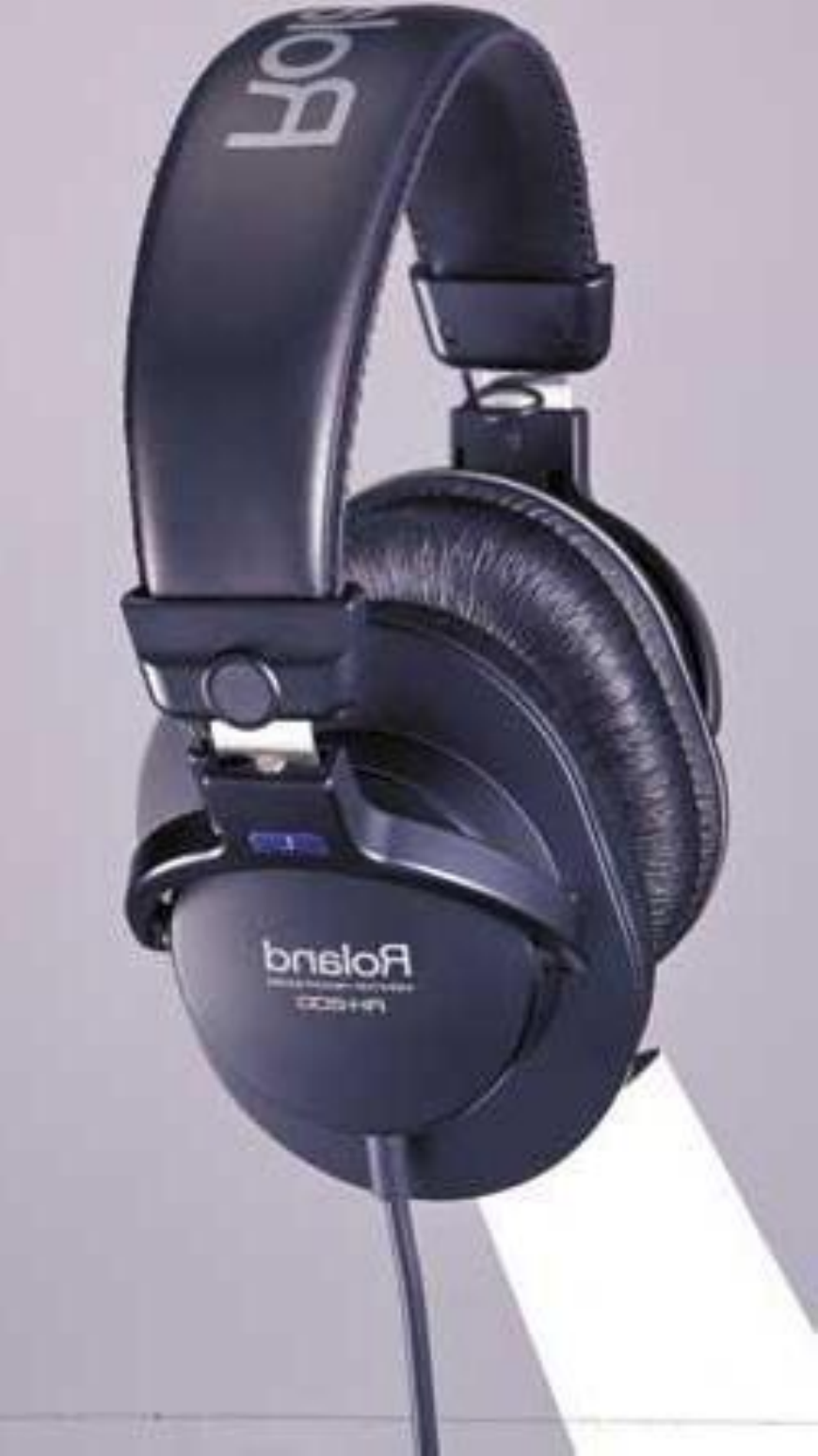}}
\includegraphics[width = 0.3cm,height = 0.85cm]{white}
\colorbox{green}{\includegraphics[width = 0.52cm,height = 0.85cm]{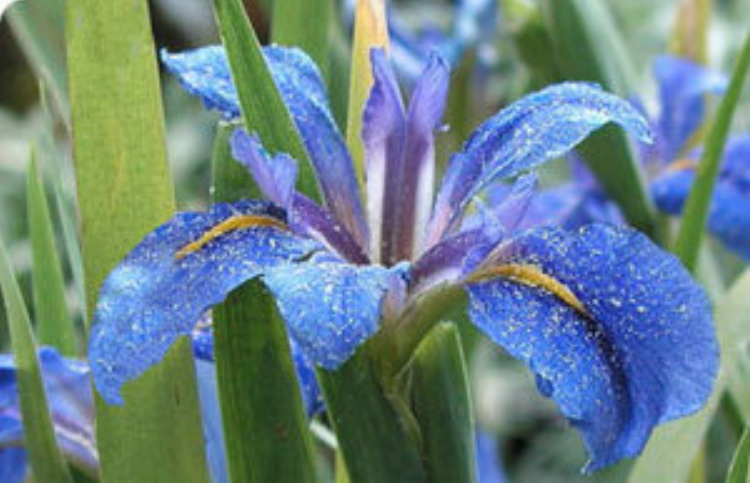}}
\colorbox{green}{\includegraphics[width = 0.52cm,height = 0.85cm]{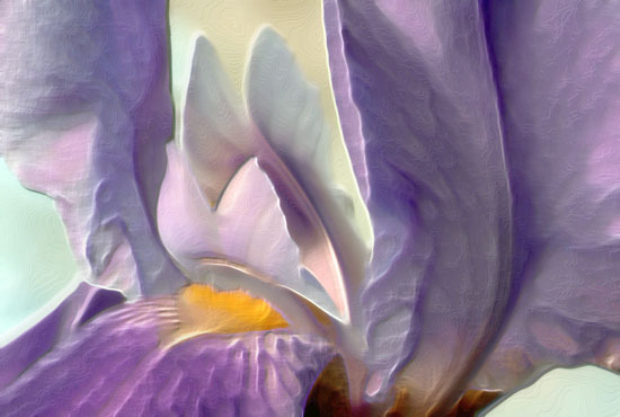}}
\colorbox{green}{\includegraphics[width = 0.52cm,height = 0.85cm]{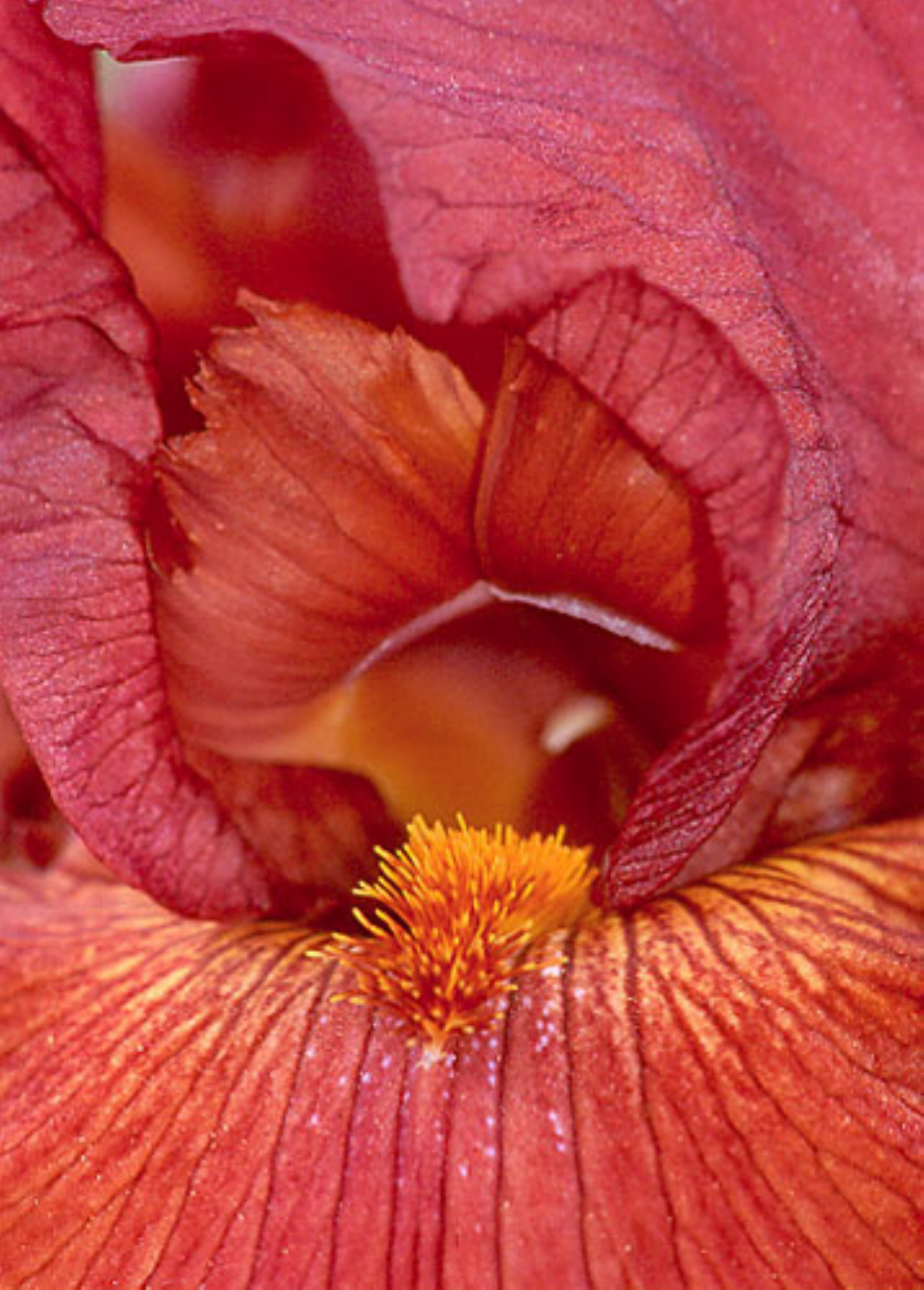}}
\colorbox{green}{\includegraphics[width = 0.52cm,height = 0.85cm]{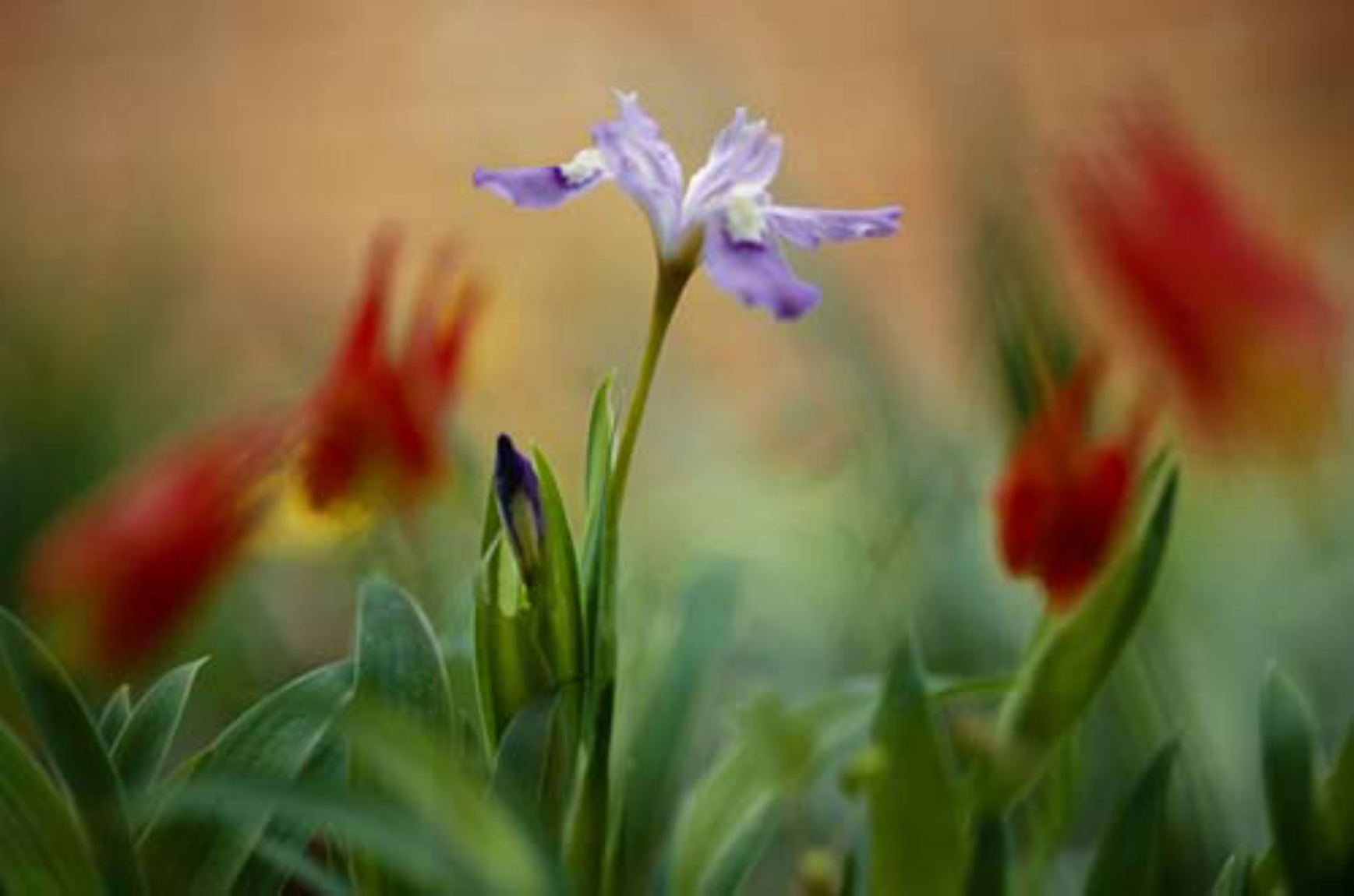}}
\includegraphics[width = 0.3cm,height = 0.85cm]{white}
\colorbox{green}{\includegraphics[width = 0.52cm,height = 0.85cm]{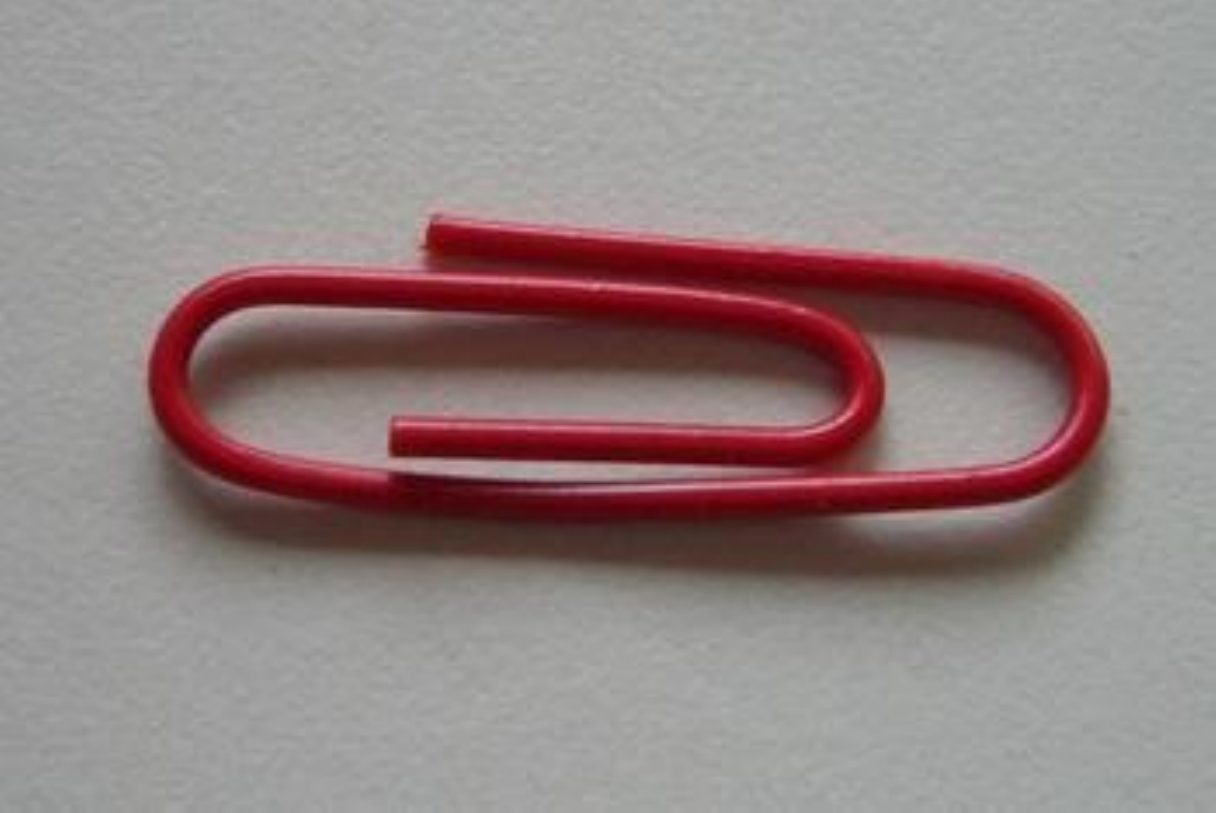}}
\colorbox{green}{\includegraphics[width = 0.52cm,height = 0.85cm]{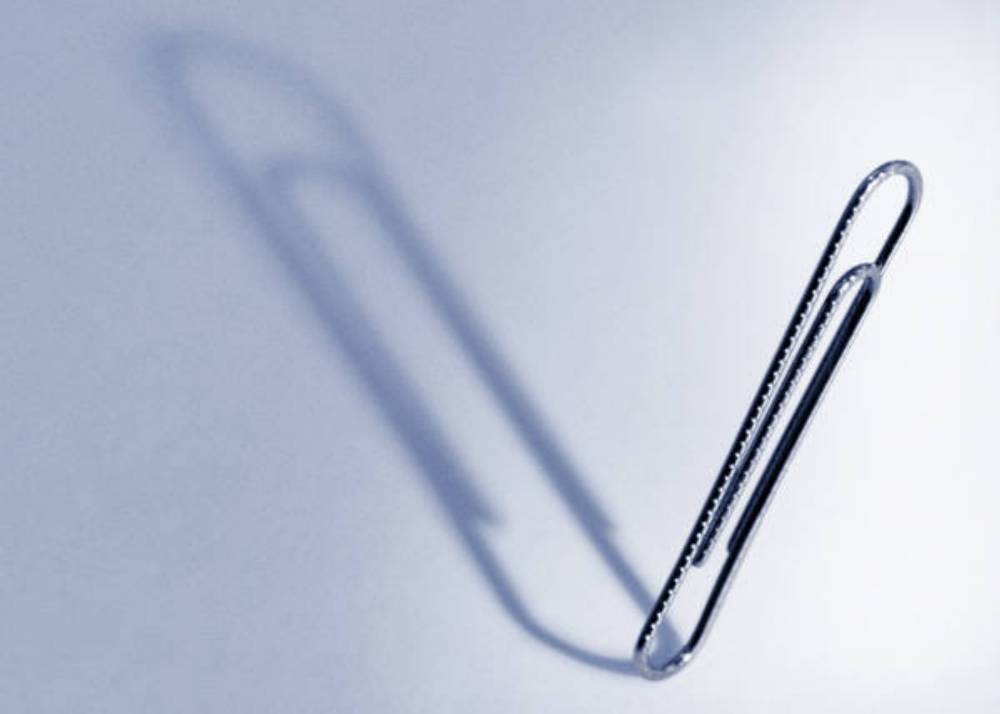}}
\includegraphics[width = 0.3cm,height = 0.85cm]{white}
\colorbox{green}{\includegraphics[width = 0.52cm,height = 0.85cm]{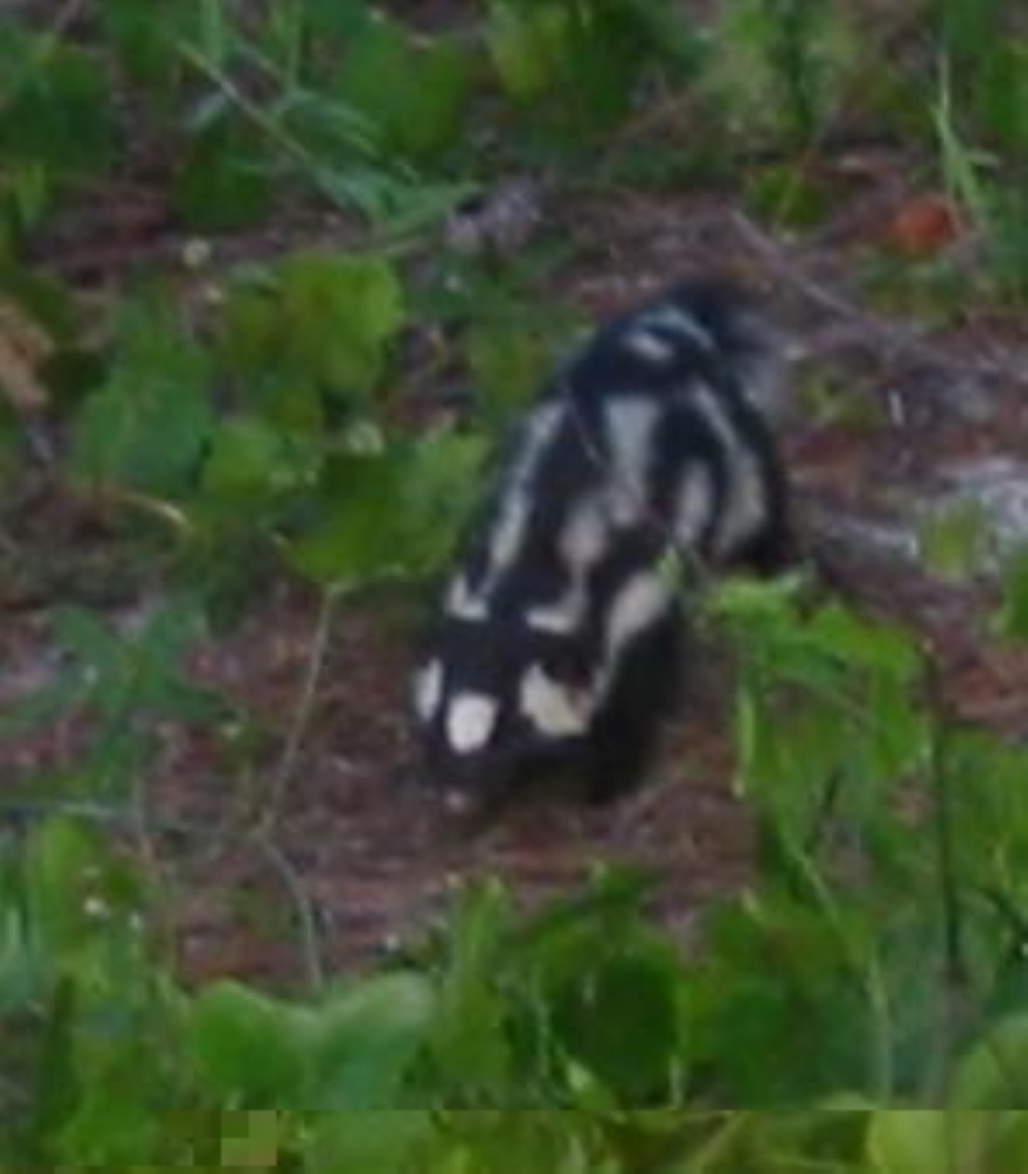}}
\colorbox{green}{\includegraphics[width = 0.52cm,height = 0.85cm]{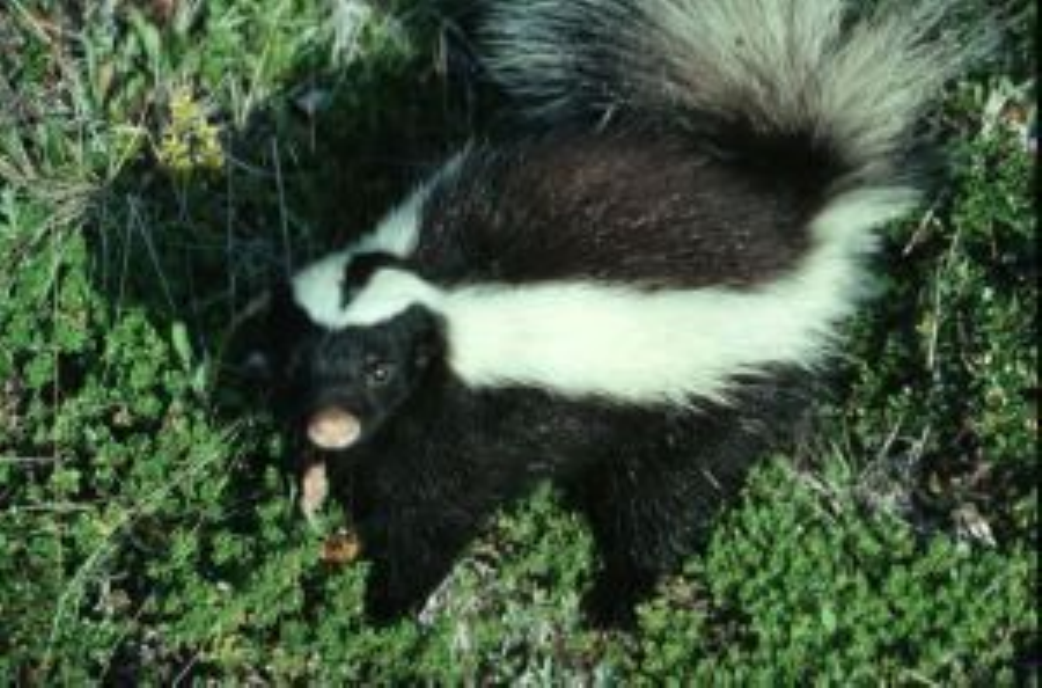}}
\colorbox{green}{\includegraphics[width = 0.52cm,height = 0.85cm]{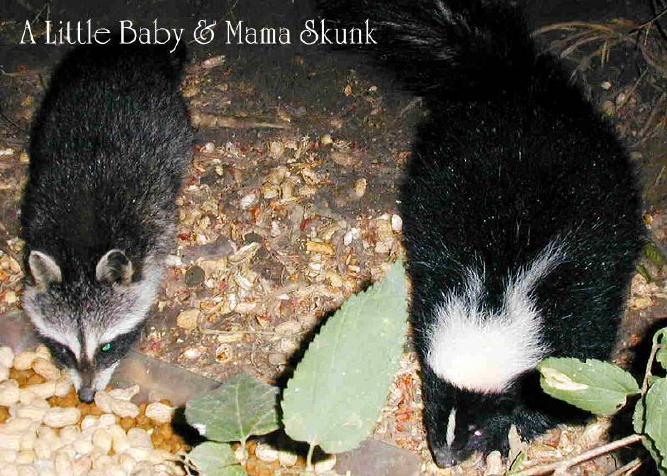}}
\colorbox{green}{\includegraphics[width = 0.52cm,height = 0.85cm]{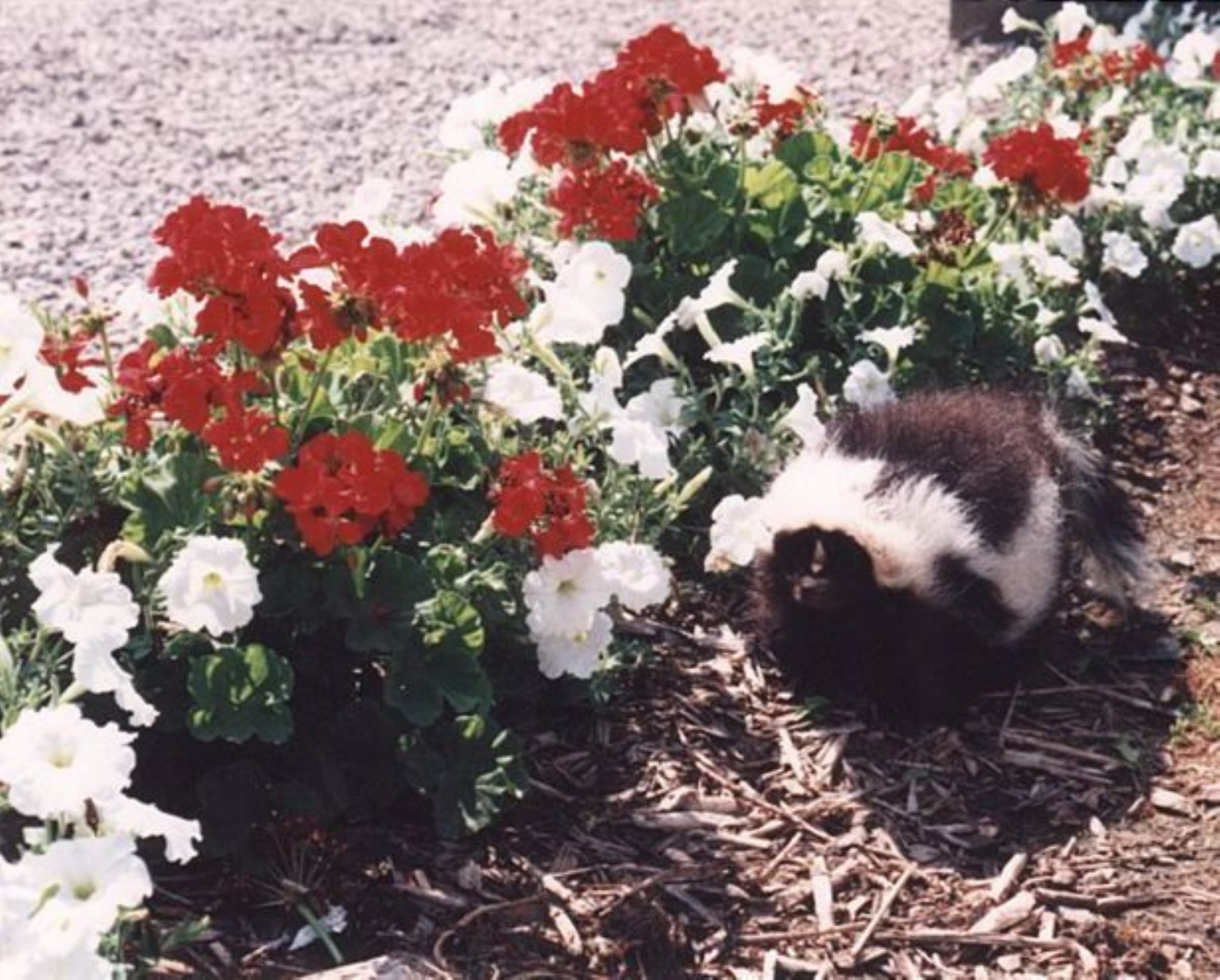}}
\\
\colorbox{red}{\includegraphics[width = 0.52cm,height = 0.85cm]{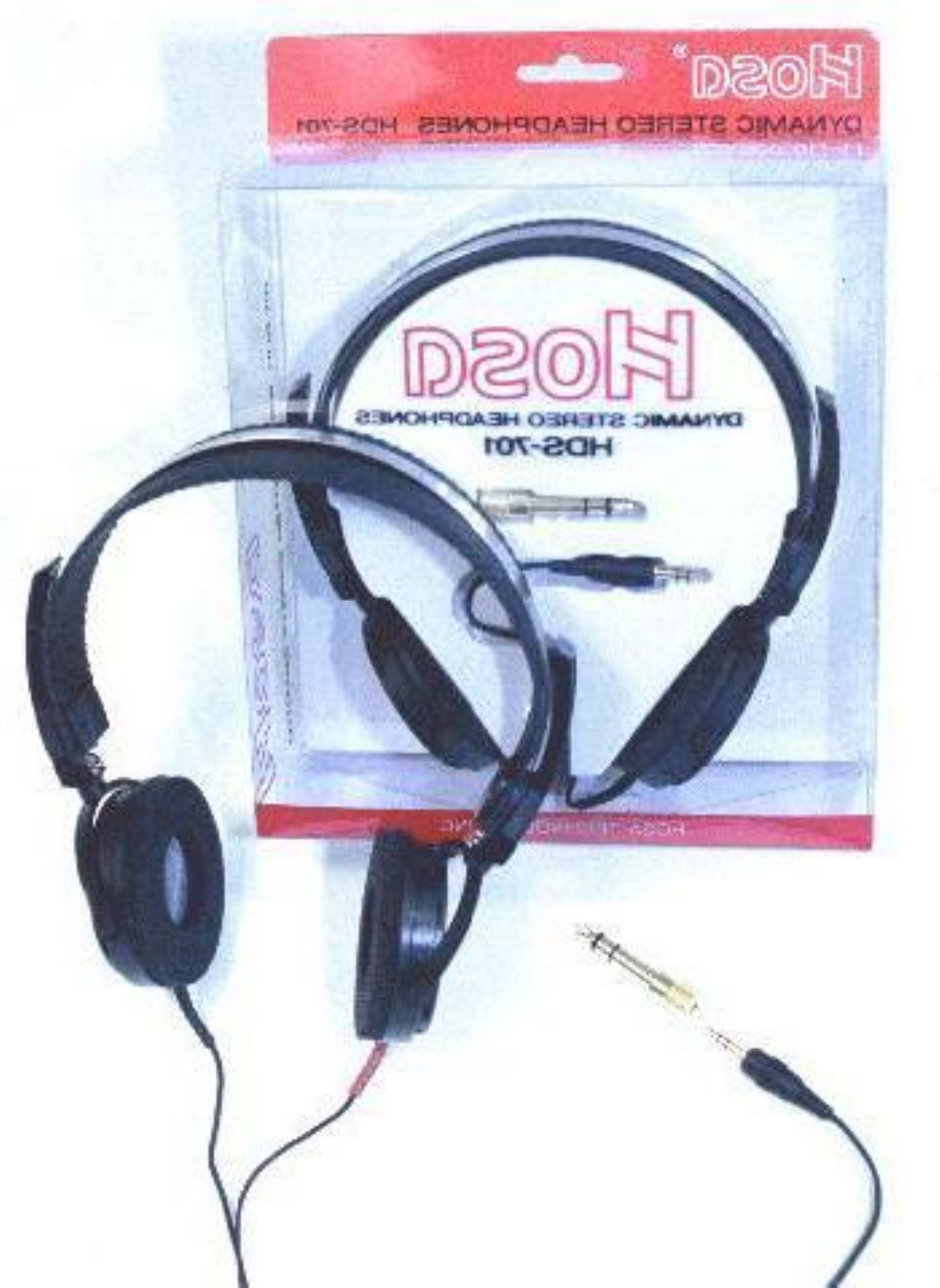}}
\colorbox{red}{\includegraphics[width = 0.52cm,height = 0.85cm]{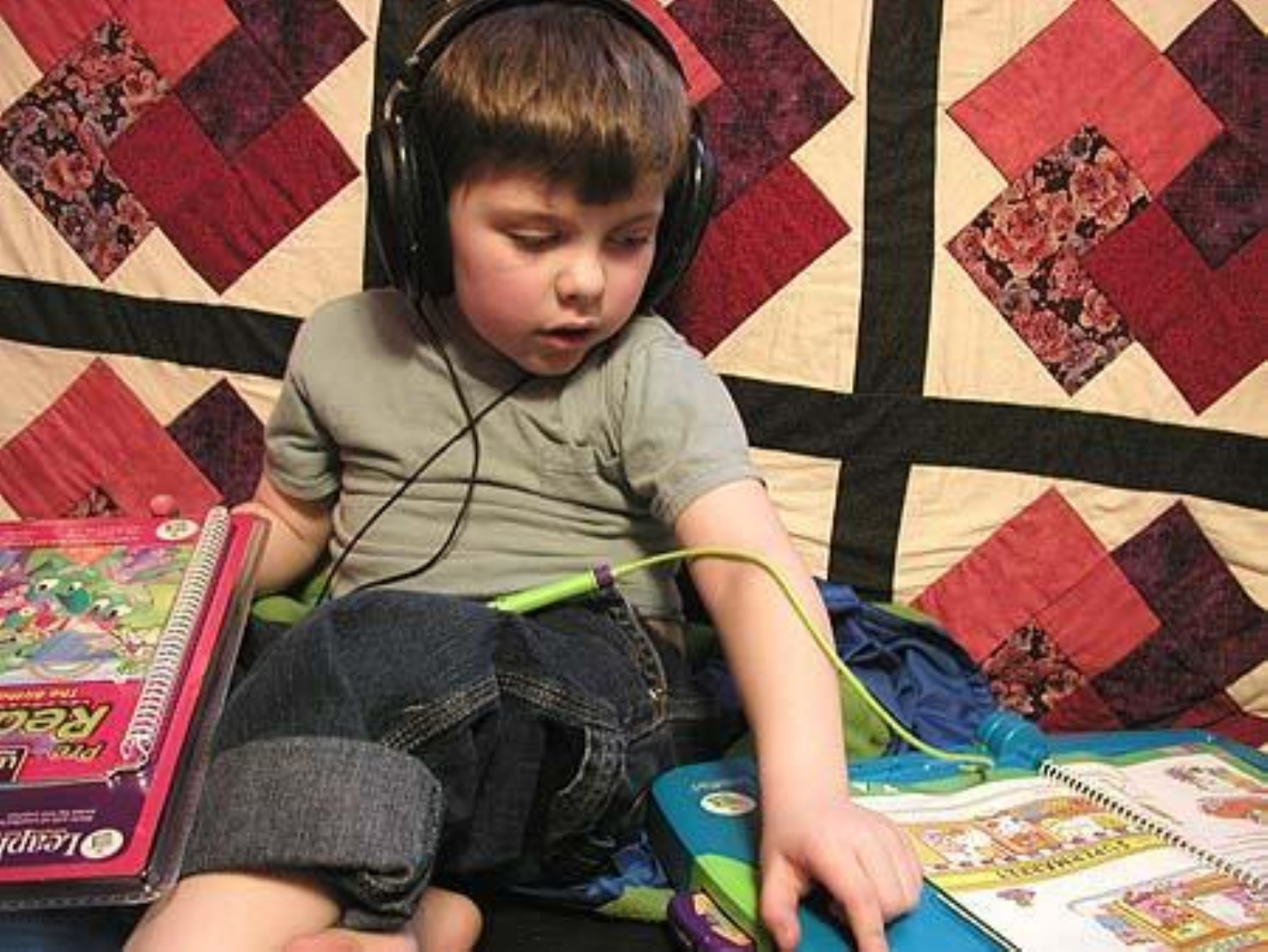}}
\colorbox{red}{\includegraphics[width = 0.52cm,height = 0.85cm]{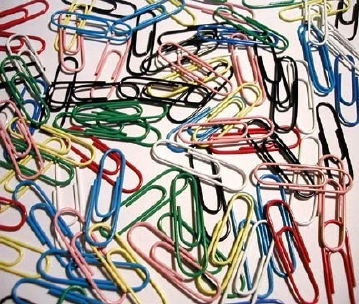}}
\colorbox{red}{\includegraphics[width = 0.52cm,height = 0.85cm]{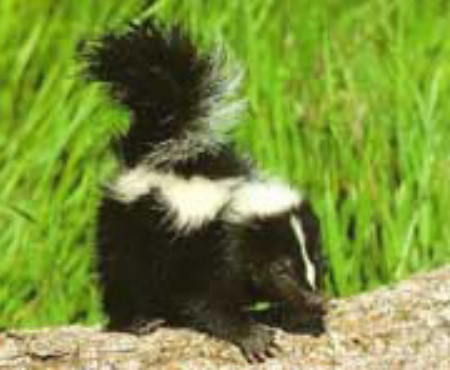}}
\includegraphics[width = 0.3cm,height = 0.85cm]{white}
\colorbox{red}{\includegraphics[width = 0.52cm,height = 0.85cm]{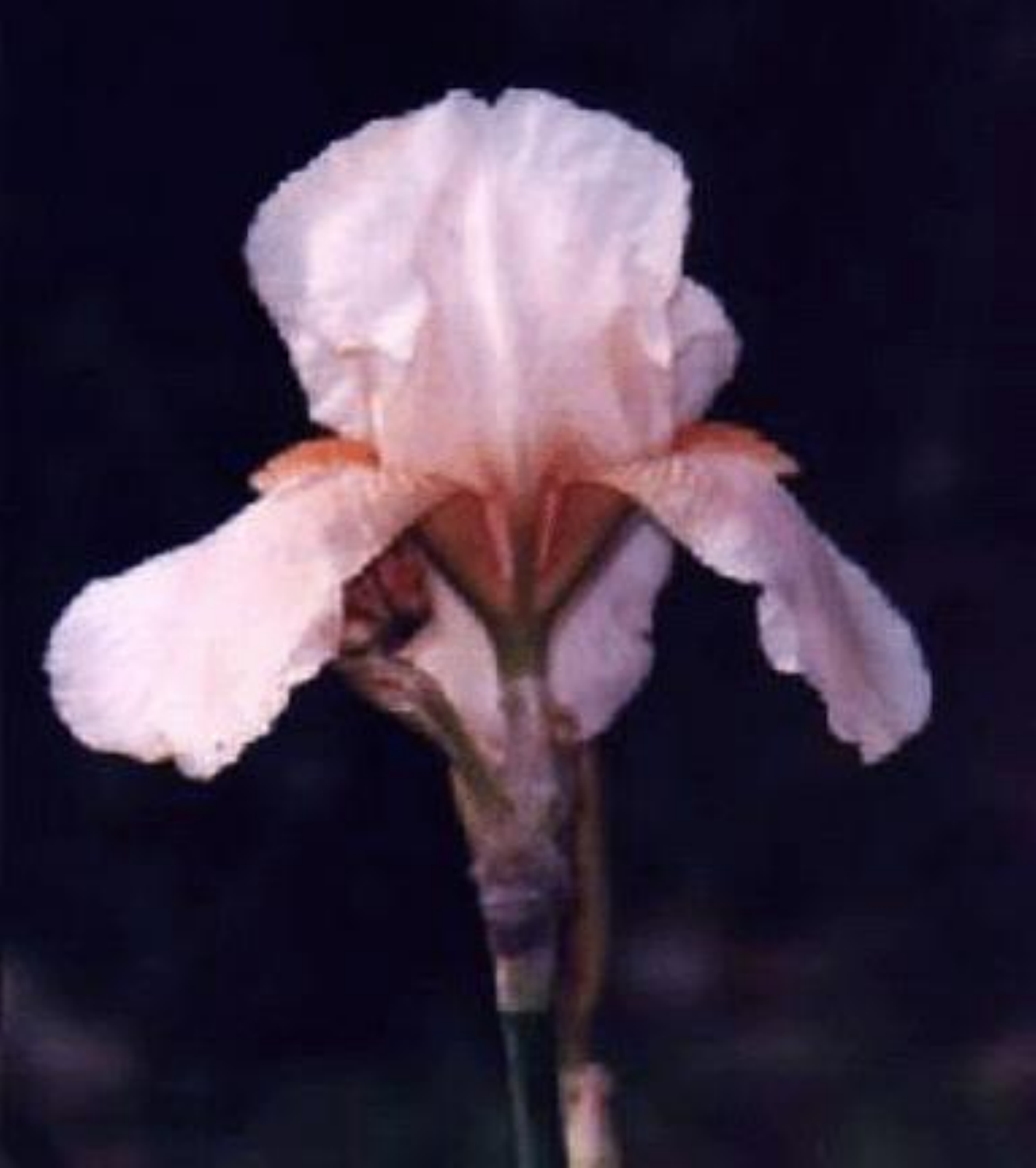}}
\colorbox{red}{\includegraphics[width = 0.52cm,height = 0.85cm]{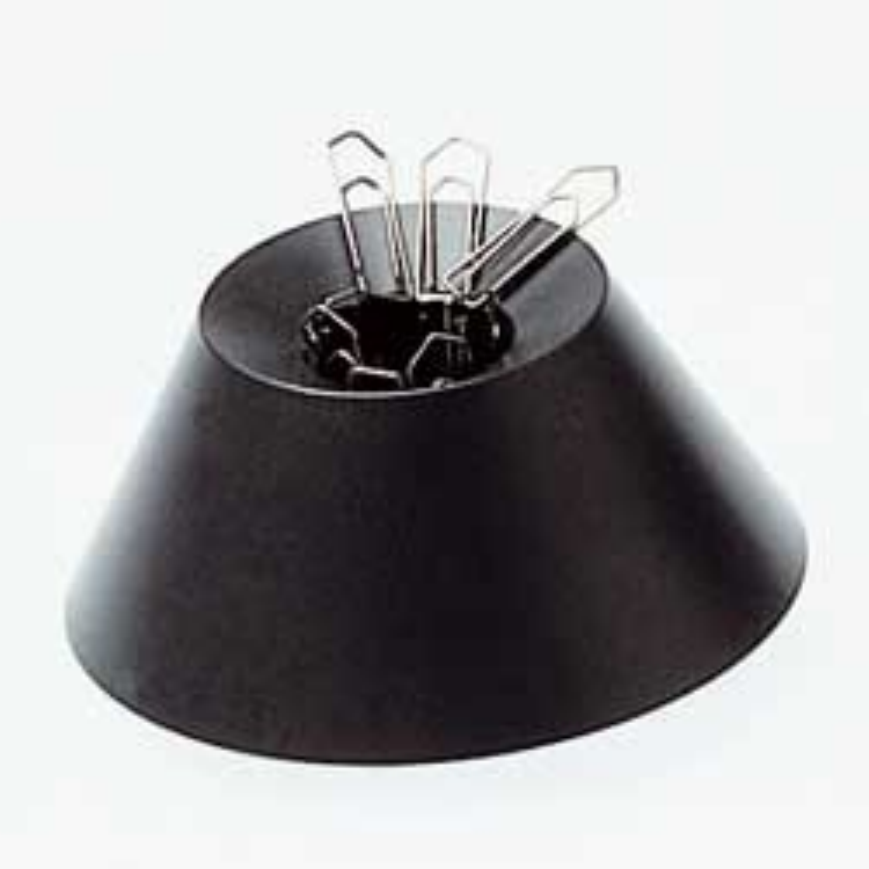}}
\colorbox{red}{\includegraphics[width = 0.52cm,height = 0.85cm]{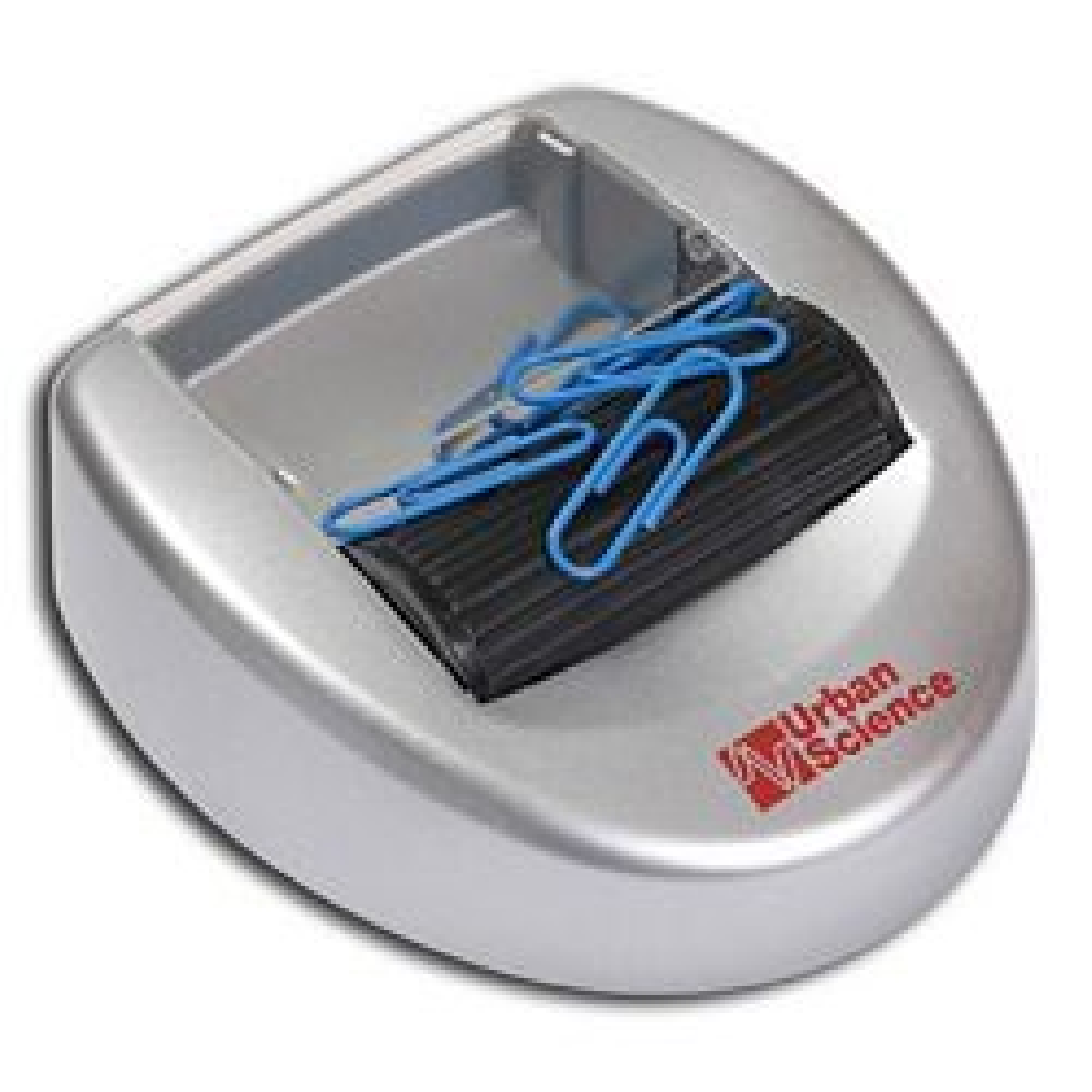}}
\colorbox{red}{\includegraphics[width = 0.52cm,height = 0.85cm]{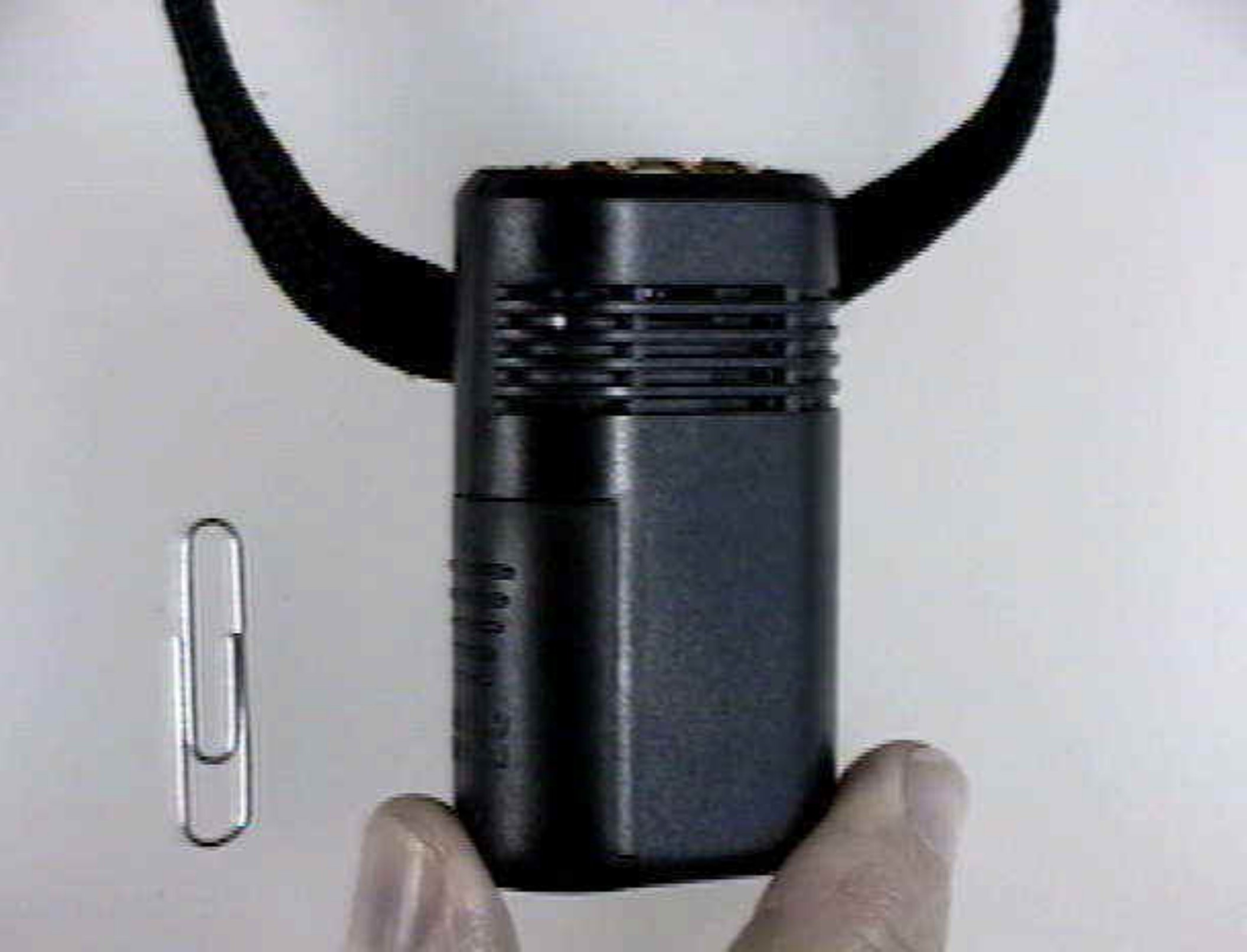}}
\colorbox{red}{\includegraphics[width = 0.52cm,height = 0.85cm]{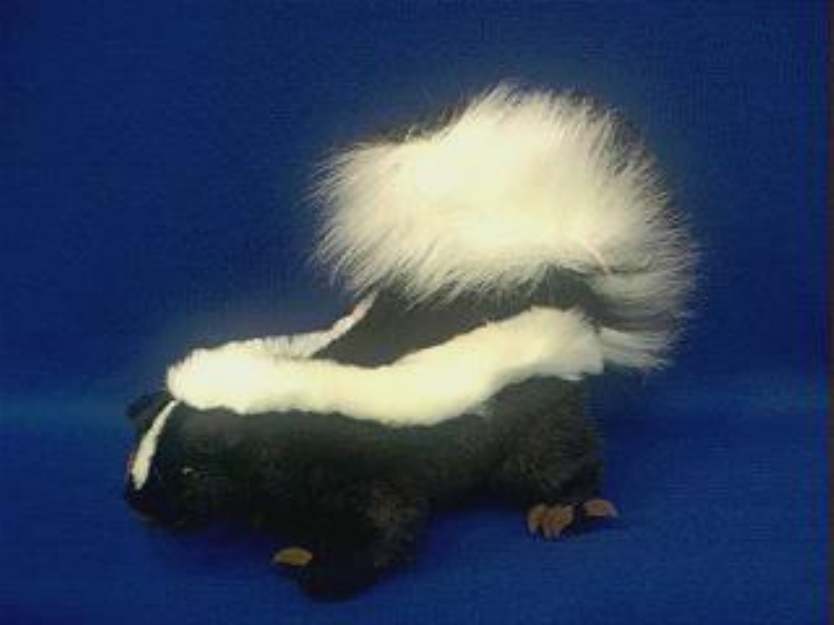}}
\includegraphics[width = 0.3cm,height = 0.85cm]{white}
\colorbox{red}{\includegraphics[width = 0.52cm,height = 0.85cm]{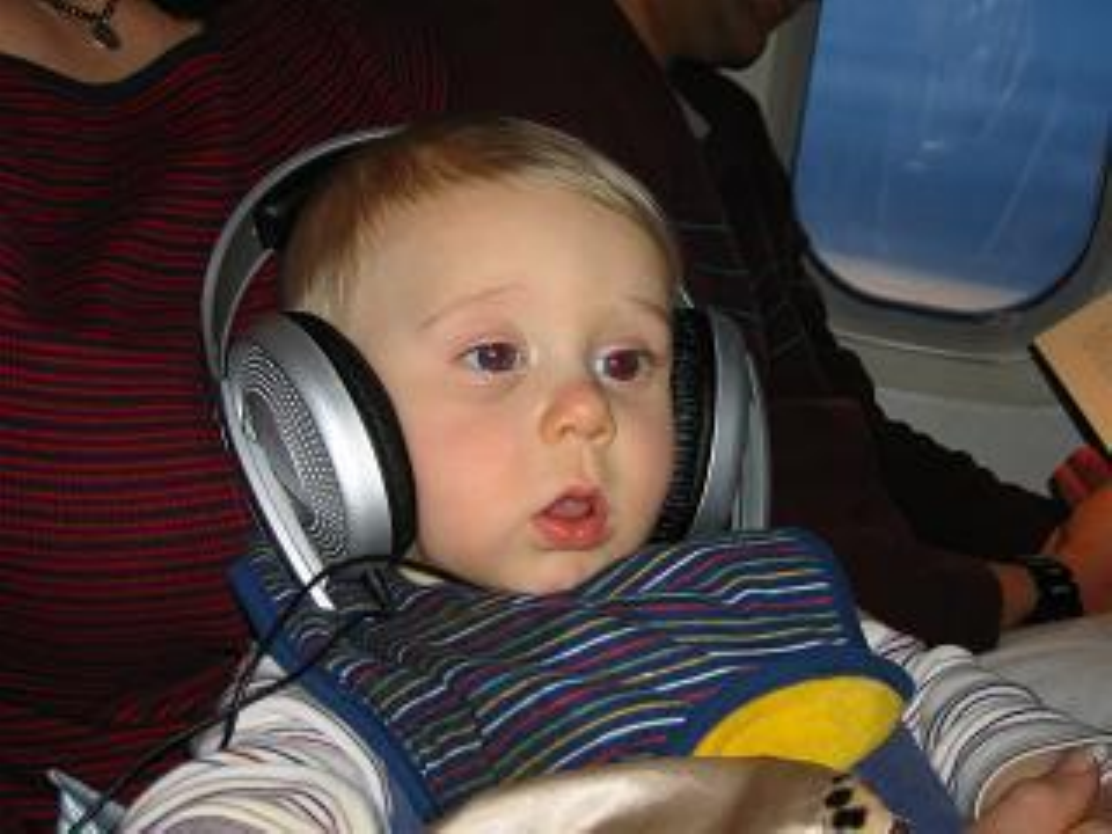}}
\colorbox{red}{\includegraphics[width = 0.52cm,height = 0.85cm]{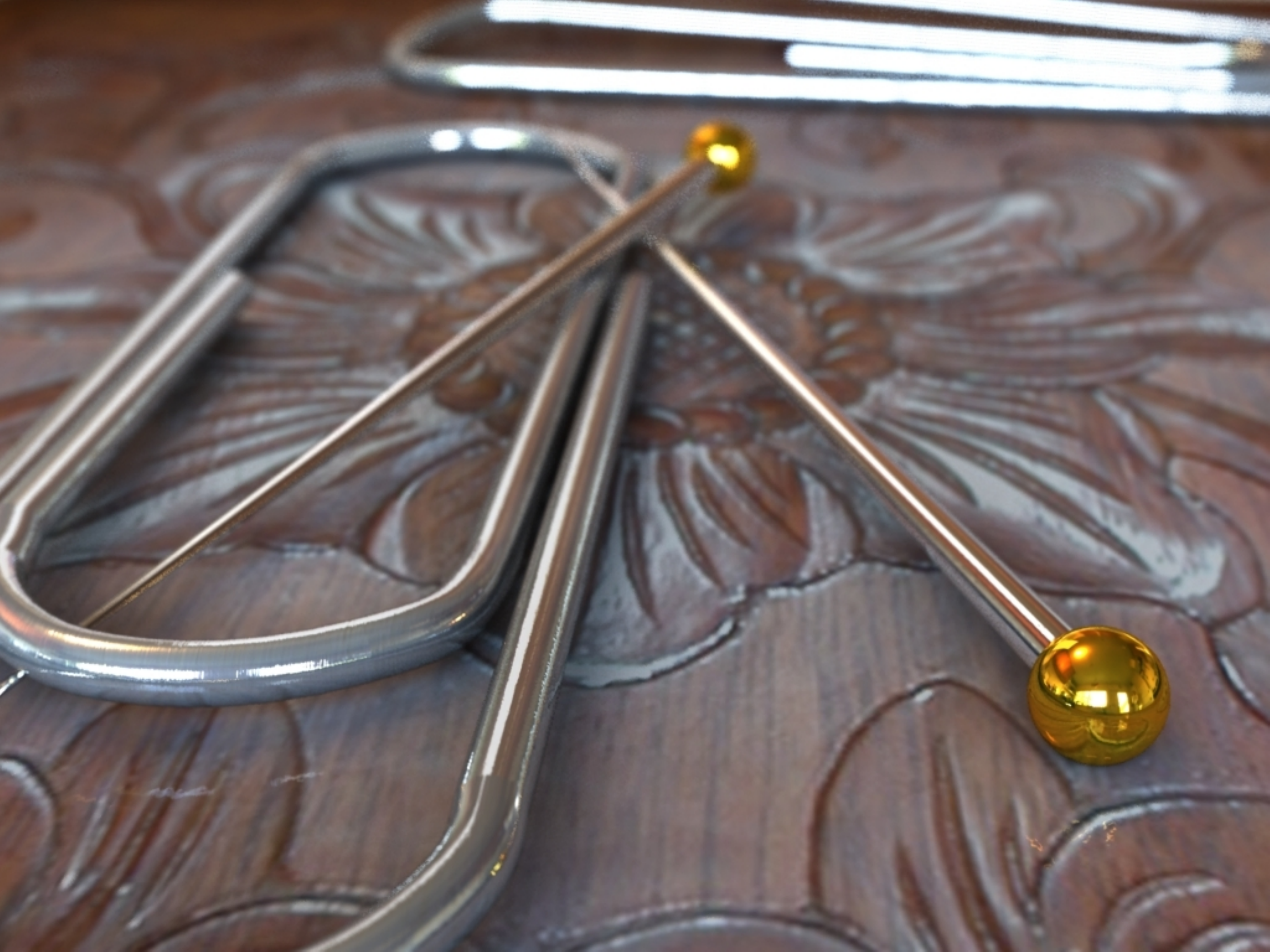}}
\colorbox{red}{\includegraphics[width = 0.52cm,height = 0.85cm]{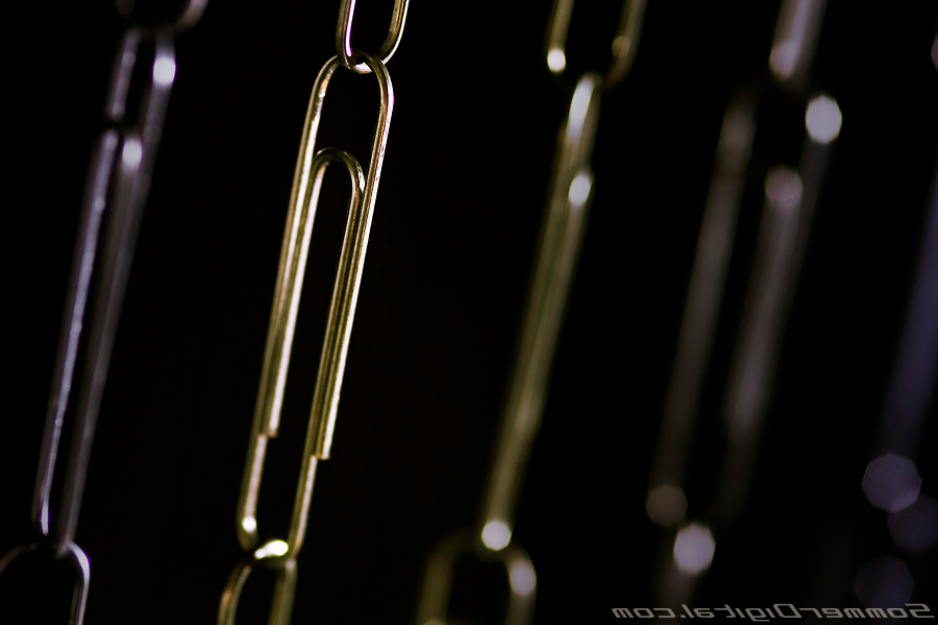}}
\colorbox{red}{\includegraphics[width = 0.52cm,height = 0.85cm]{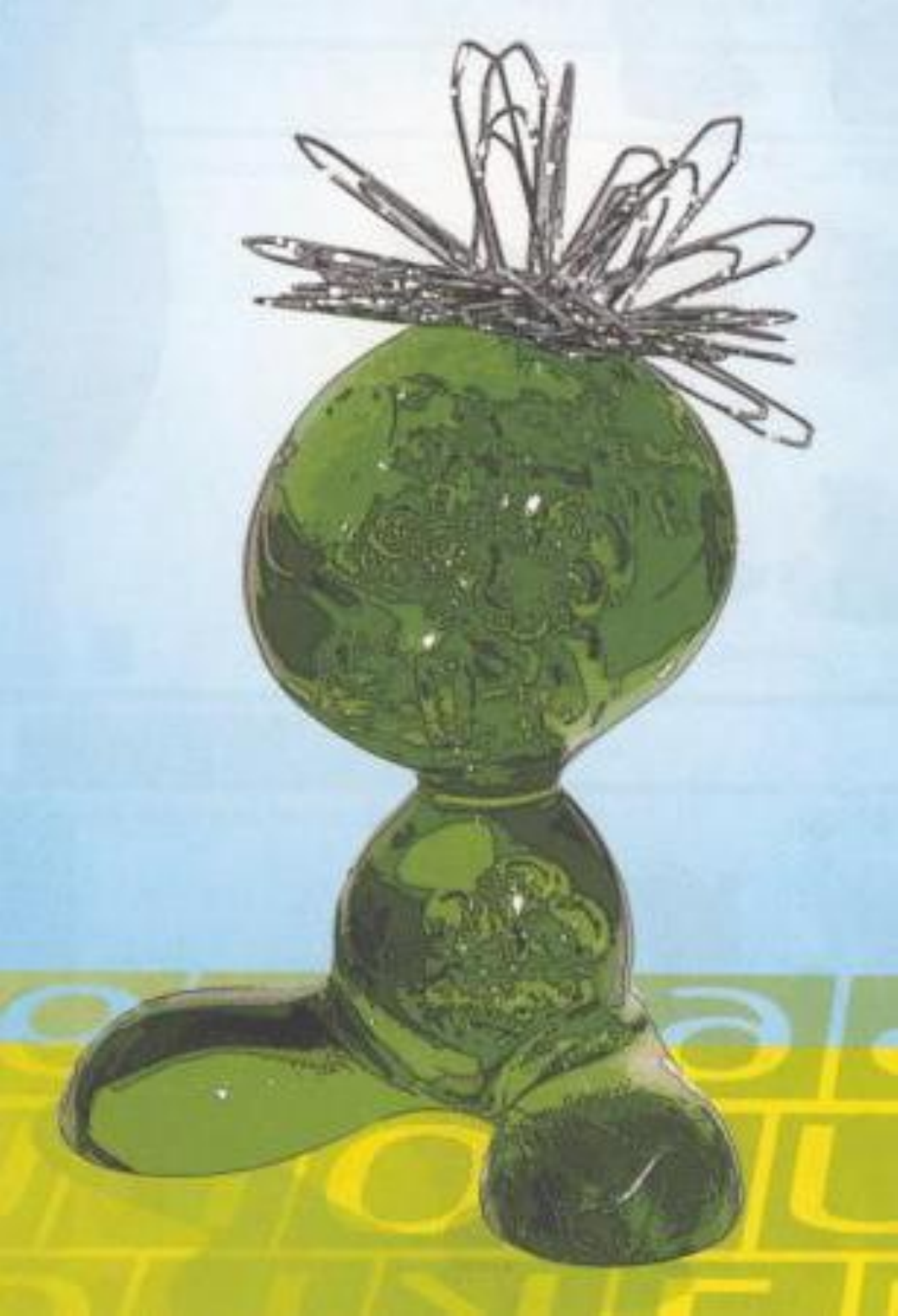}}
\includegraphics[width = 0.3cm,height = 0.85cm]{white}
\includegraphics[width = 0.52cm,height = 0.85cm]{white}
\includegraphics[width = 0.52cm,height = 0.85cm]{white}
\includegraphics[width = 0.5cm,height = 0.85cm]{white}
\colorbox{red}{\includegraphics[width = 0.52cm,height = 0.85cm]{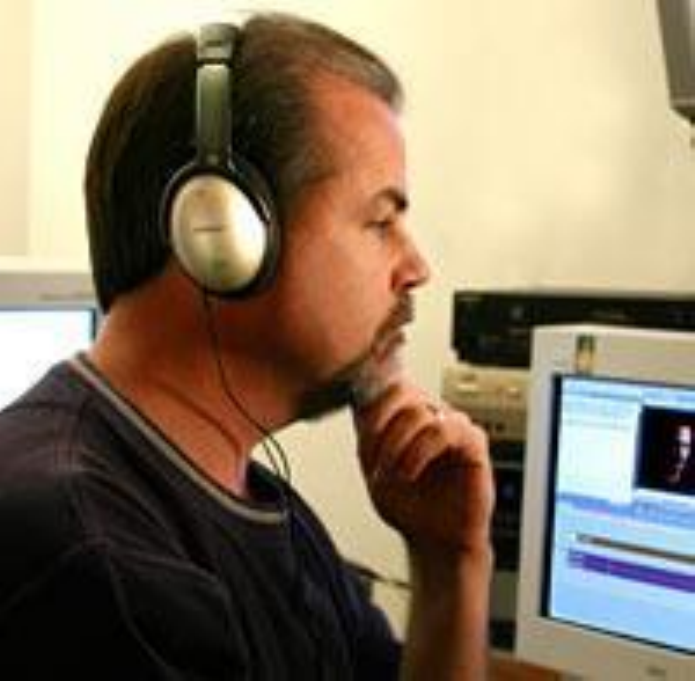}}
\colorbox{red}{\includegraphics[width = 0.52cm,height = 0.85cm]{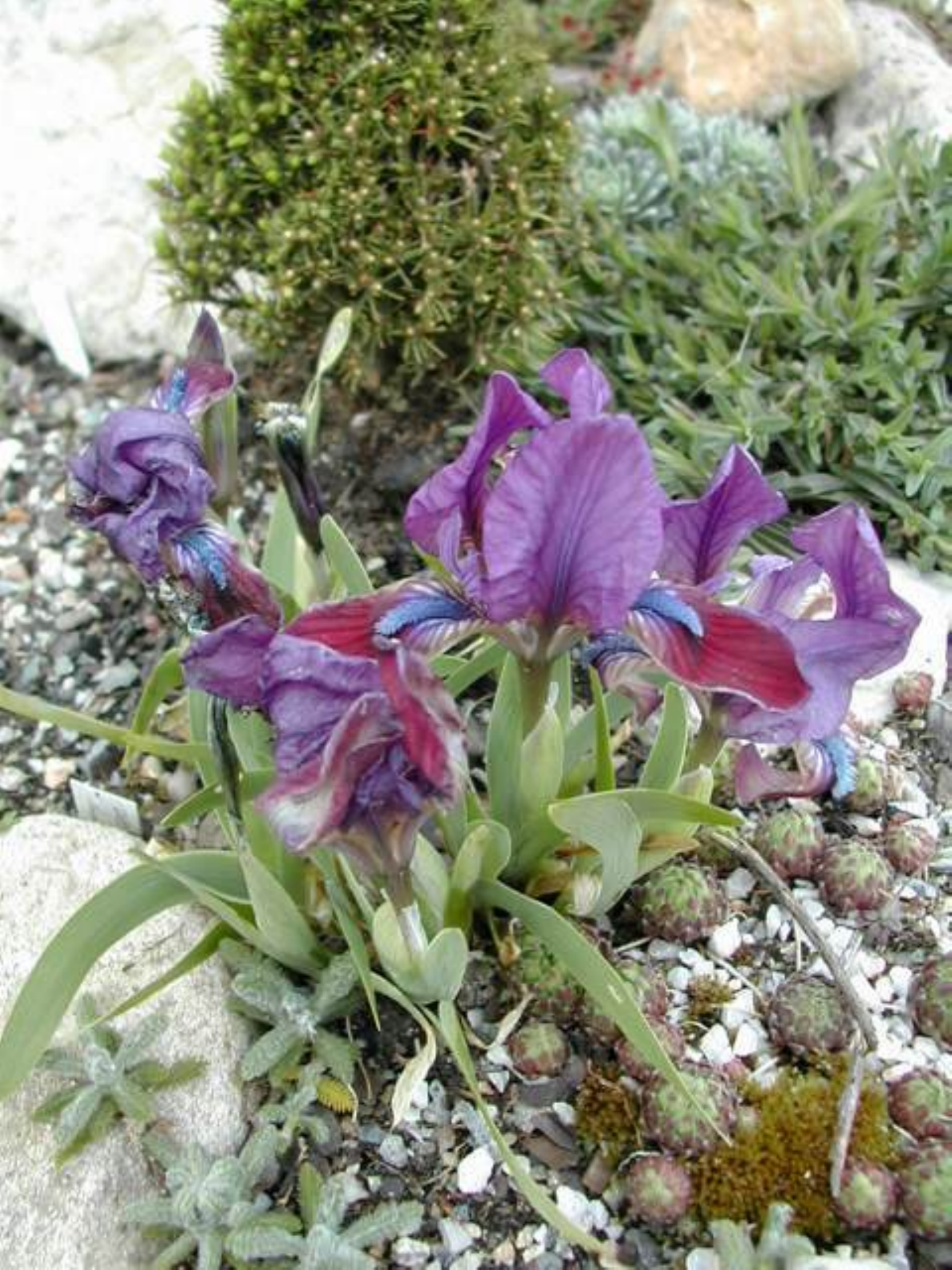}}
\colorbox{red}{\includegraphics[width = 0.52cm,height = 0.85cm]{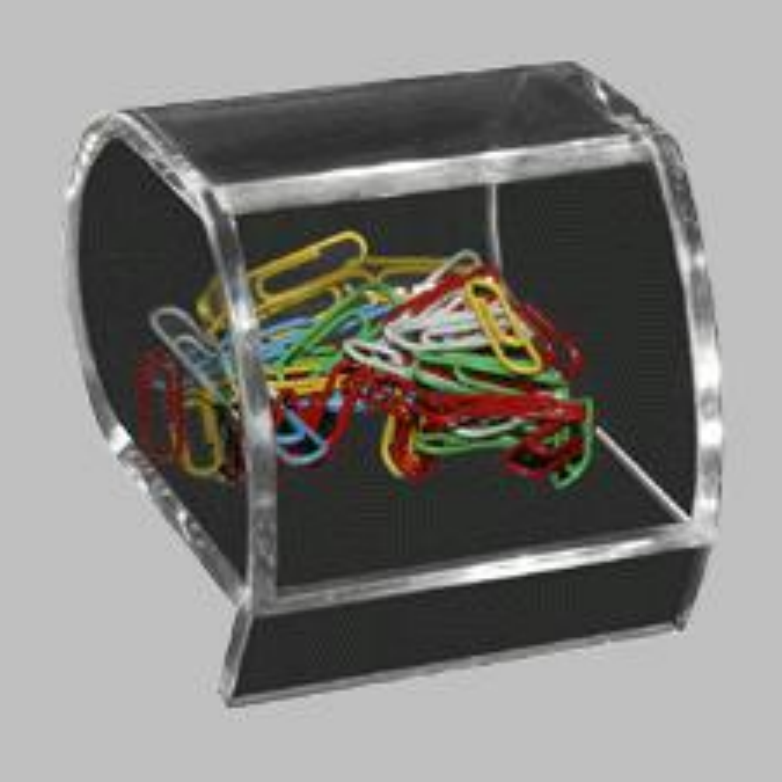}}
\colorbox{red}{\includegraphics[width = 0.52cm,height = 0.85cm]{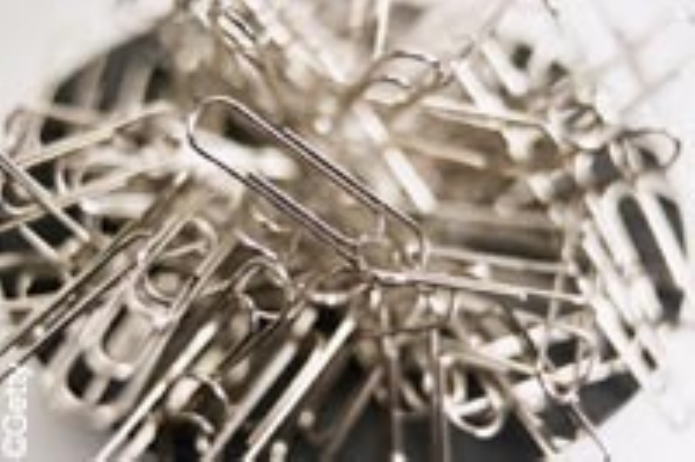}}
   \caption{
   Some examples of correctly classified (top two rows) and misclassified (bottom row)
   images by \mbh.
   The categories are ``{cartman}'', ``{headphones}'', ``{iris}'',
   ``{paperclip}'' and ``{skunk}''.
   The accuracy of this test is $71.2\%$.
   No image is falsely classified into the category of ``{paperclip}''.
   }
\label{figCVPR11:cal-samples}
\end{figure*}

    Clearly, all the results of three algorithms on feature selection make sense.
    Most discriminative features are tagged with circles or crosses.
    Some blocks that contain significant information on luminance are tagged with thick marks,
    such as the $22$-th and $43$-th features in digit ``6'',
    and the $22$-th and $11$-th in ``9''.
    If taking a close look at the figure, we can find \mbh is slightly better than AdaBoost.ECC.
    For example, on the $43$-th feature the green cross should be marked on digit ``9''
    instead of ``1''.
    Also in ``1'', the $21$-th feature should be tagged with a relatively thicker circle.
    However, \mbe's results are not as meaningful as \mbh.

\paragraph{Object recognition on a subset of Caltech-$256$}
    Finally, we test our algorithms on the data set of Caltech-$256$,
    which is one of the most popular multi-class benchmarks.
    We randomly select $5$ categories of images.
    $75\%$ of them are randomly selected for training and the other $25\%$ for test.
    A descriptor of $1000$ dimensions is used, which combines quantized color and texture
    local invariant features (also called \emph{visterms} \citep{quelhas2006natural}).
    The maximum number of  iterations is still set to $500$.
    The averaged test accuracies of $10$ runs are reported in Figure \ref{figCVPR11:caltech}.
    Again, we use the simplest decision stumps as weak classifiers.
    We can see that all the four boosting algorithms perform similarly,
    except that \mbe performs worse than the other three. It may be
    due to the fact that we have not fine tuned the cross validation parameter.
     We show some images that are correctly classified and falsely classified by \mbh
     in Figure~\ref{figCVPR11:cal-samples}.

\subsection{\multistruct}

    Next we evaluate our mixed-norm regularized boosting algorithms.
    We mainly use the $ \ell_{1,2} $
    regularization since $ \ell_{1,\infinity} $ delivers similar
    performance.
    In order to ensure a fair comparison we evaluate the performance
    of the proposed algorithms against other
    multi-class boosting algorithms evaluated previously, along with
    AdaBoost-SIP
    \citep{Zhang2009Finding}, JointBoost \citep{Torralba2007Sharing},
    GradBoost ($\ell_1/\ell_2$-regularized) \citep{Duchi2009Boosting}.
    Note that the last three also try to share features across classes.

\paragraph{Artificial data}
    We consider the problem of discriminating $6$ object
    classes on a $2$D plane.  Each sample consists of $2$ measurements:
    orientation and radius.
    For all classes, the orientation is drawn uniformly between $0$ and $2\pi$.
    The radius of the first group is drawn uniformly between $0$ and
    $1$, the radius of the second group between
    $1$ and $2$, and so on.
    We generate $50$ samples in the first group, $100$ samples in the second group,
    $150$ samples in the third group, and so on.
    The number of training sets is the same as the number of test sets.
    In this example feature vectors are the
    vertical and horizontal coordinates of the samples.  We train $5$ different
    classifiers based on the proposed \multistruct\ (hinge loss),
    AdaBoost.MH \citep{Schapire1999Improved}, AdaBoost.ECC
    \citep{Guruswami1999Multiclass} and JointBoost
    \citep{Torralba2007Sharing}.
    The multi-class classifier is composed of
    a set of binary decision stumps.
    For our algorithm, we choose the regularization parameter $ \nu $  from $\{ 10^{-5},
    10^{-4}, 10^{-3}, 10^{-2}, 10^{-1}\}$.  For JointBoost, we set the
    outermost class (maximal radius) as background.  We evaluate $5$
    boosting algorithms on this toy data and plot the decision
    boundary in Figure~\ref{figCVPR12:toy}.
    Table \ref{tab:Toy} reports some training and test error rates.
    Our
    algorithm performs best amongst five evaluated classifiers.
    We conjecture that the poor
    performance of JointBoost is due to the small number of background
    samples in the
    training data.  JointBoost was designed for the task of
    multi-class object detection where the objective is to detect
    several classes of objects from background samples.  The algorithm
    might not work well on general multi-class problems.
    We then repeat our
    experiment by increasing the number of iterations to $500$, and
    JointBoost, Adaboost.MH and AdaBoost.ECC
    still perform poorly on this toy data set compared to our
    approach.

    \begin{table}
      \centering
      \scalebox{1}{
      \begin{tabular}{r|ccccc}
      \hline
       $\#$ feat. &  Ada.ECC   &  Ada.MH  & JointBoost & \multiboost & \multistruct \\
      \hline
      \hline
      $20$  & $0.62/0.68$ & $0.48/0.53$ & $0.71/0.71$ & $0.10/$\textbf{0.14} & $0.10/$\textbf{0.14} \\
      $100$ & $0.23/0.33$ & $0.17/0.24$ & $0.44/0.50$ & $0.05/0.13$ & $0.03/$\textbf{0.10} \\
      $500$ & $0.08/0.20$ & $0.09/0.18$ & $0.24/0.38$ & $0.03/0.10$ & $0.02/$\textbf{0.09} \\
      \hline
      \end{tabular}
      }
      \caption{Training$/$test errors of a few multi-class boosting
      methods on the $2$D toy data set. The proposed \multistruct\ with
      hinge loss performs slightly better than others. See Figure
      \ref{figCVPR12:toy} for an illustration.
      }
      \label{tab:Toy}
    \end{table}

    \begin{figure*}[t]
        \centering
        \begin{minipage}{1\textwidth}
            \begin{tabular}{ p{0.19\textwidth}  p{0.19\textwidth}
                p{0.16\textwidth}  p{0.16\textwidth}  p{0.16\textwidth}}
               \scriptsize ~ ~  AdaBoost.ECC  &
               \scriptsize ~  AdaBoost.MH  &
               \scriptsize JointBoost  &
               \scriptsize \multiboost &
               \scriptsize \multistruct
            \end{tabular}
            \vspace{-.425cm}
        \end{minipage}
           \includegraphics[width=0.19\textwidth,clip]{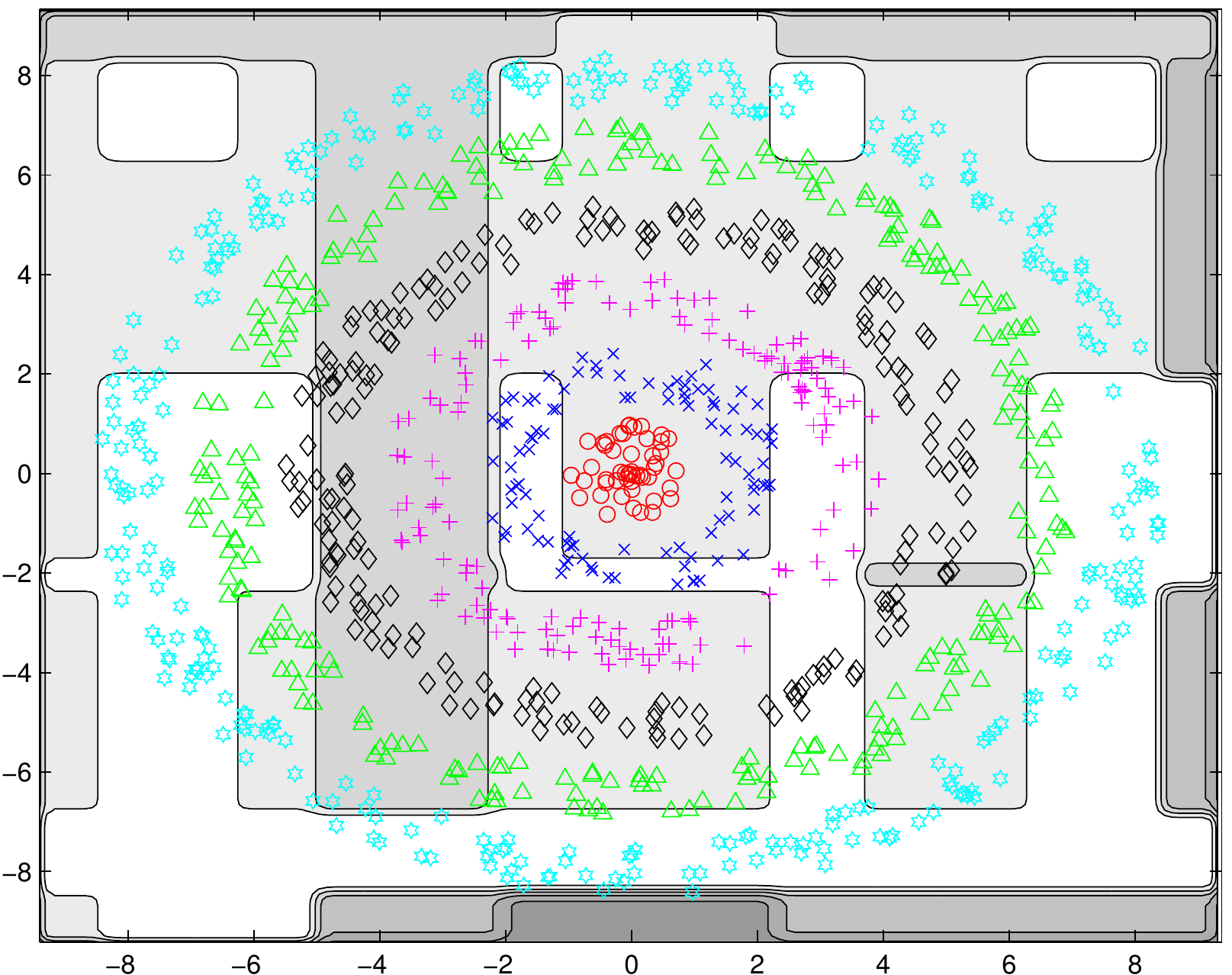}
           \includegraphics[width=0.19\textwidth,clip]{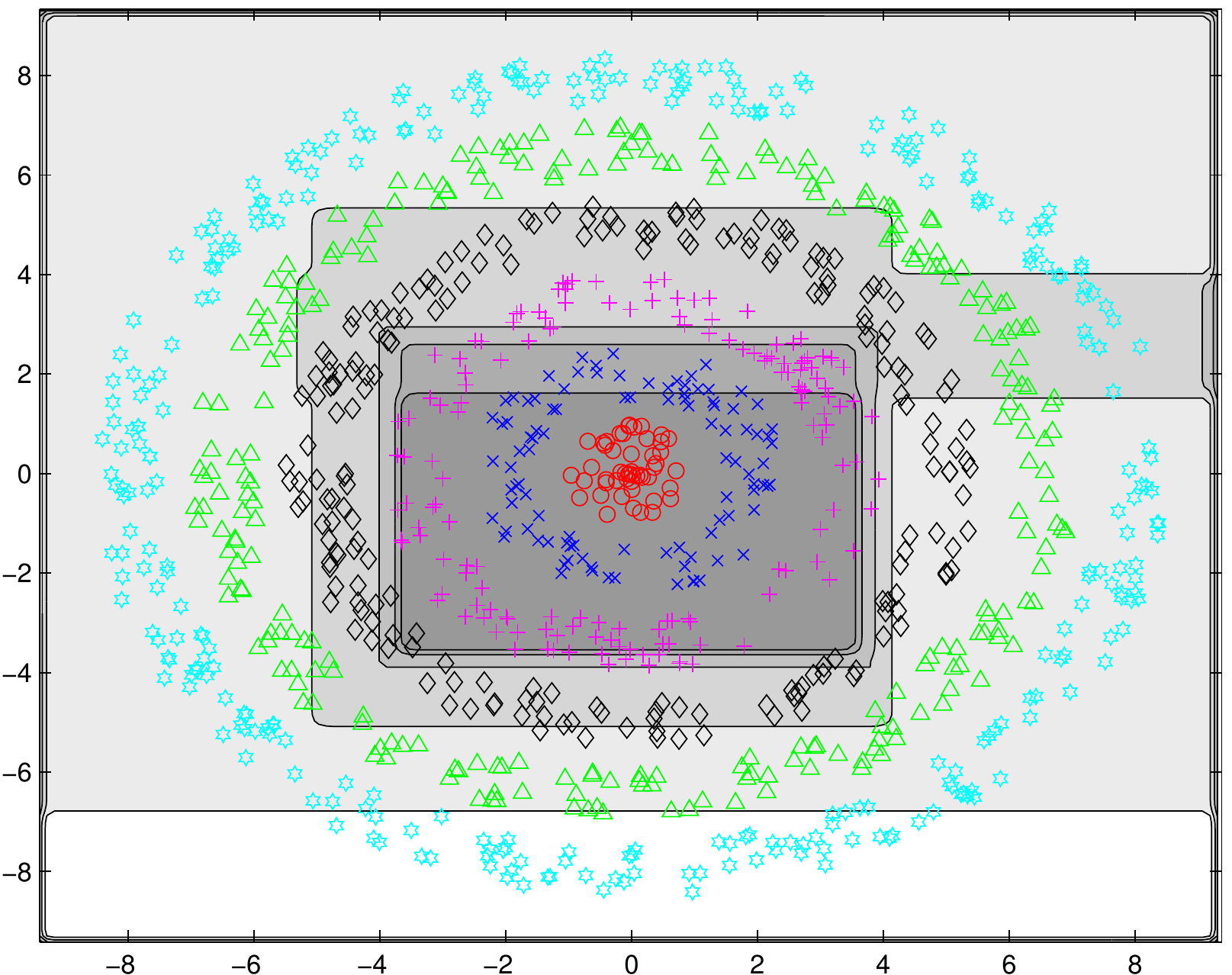}
           \includegraphics[width=0.19\textwidth,clip]{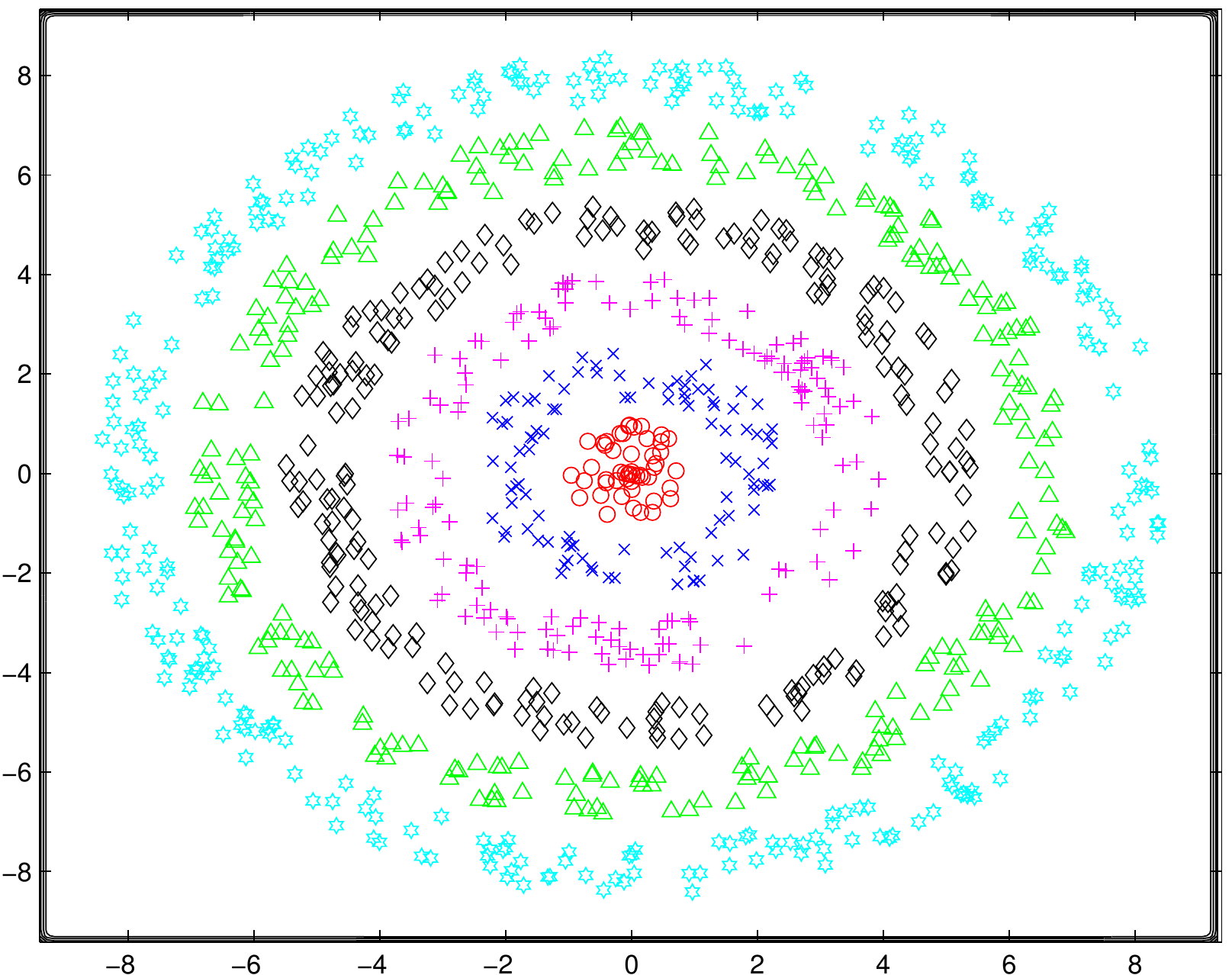}
           \includegraphics[width=0.19\textwidth,clip]{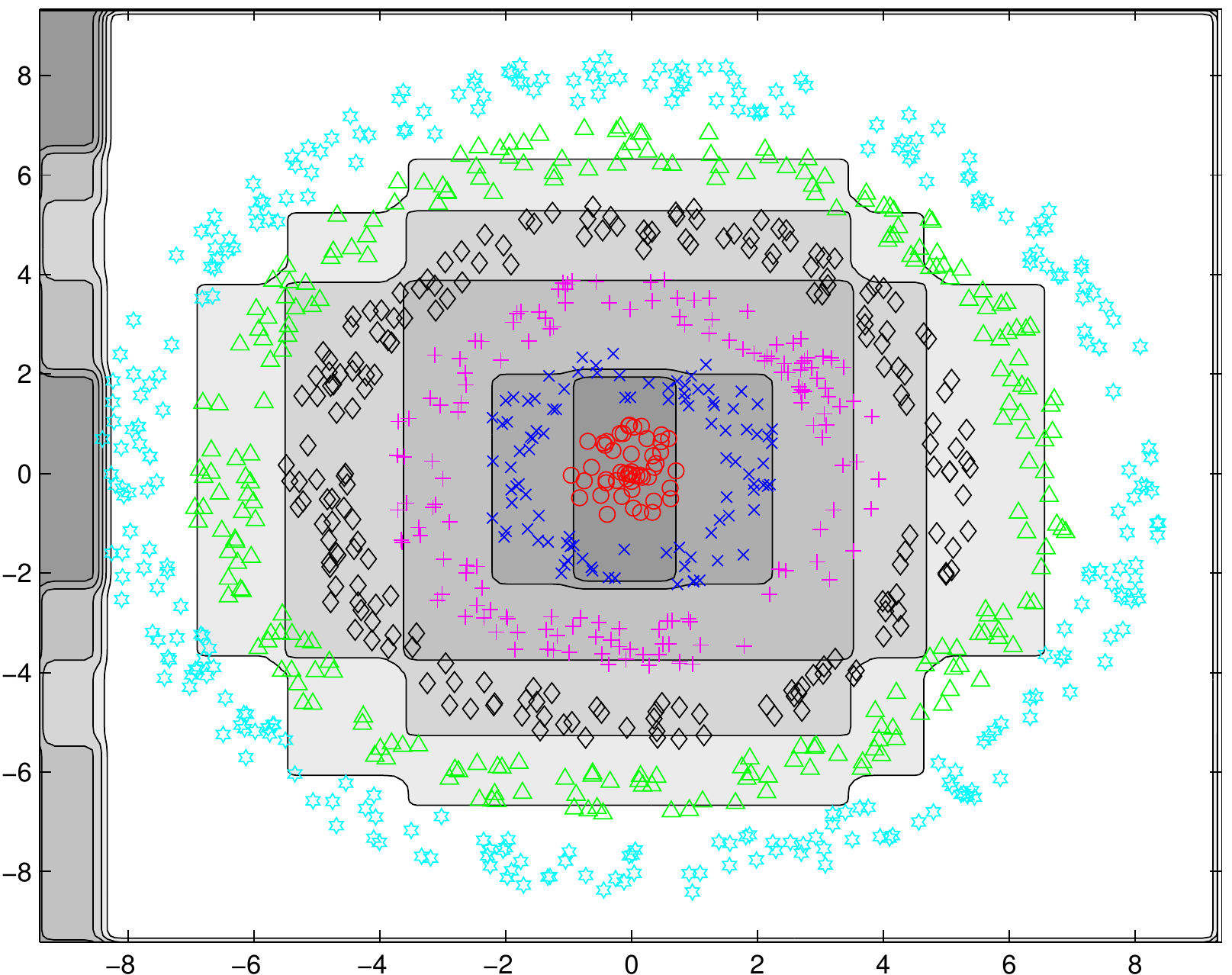}
           \includegraphics[width=0.19\textwidth,clip]{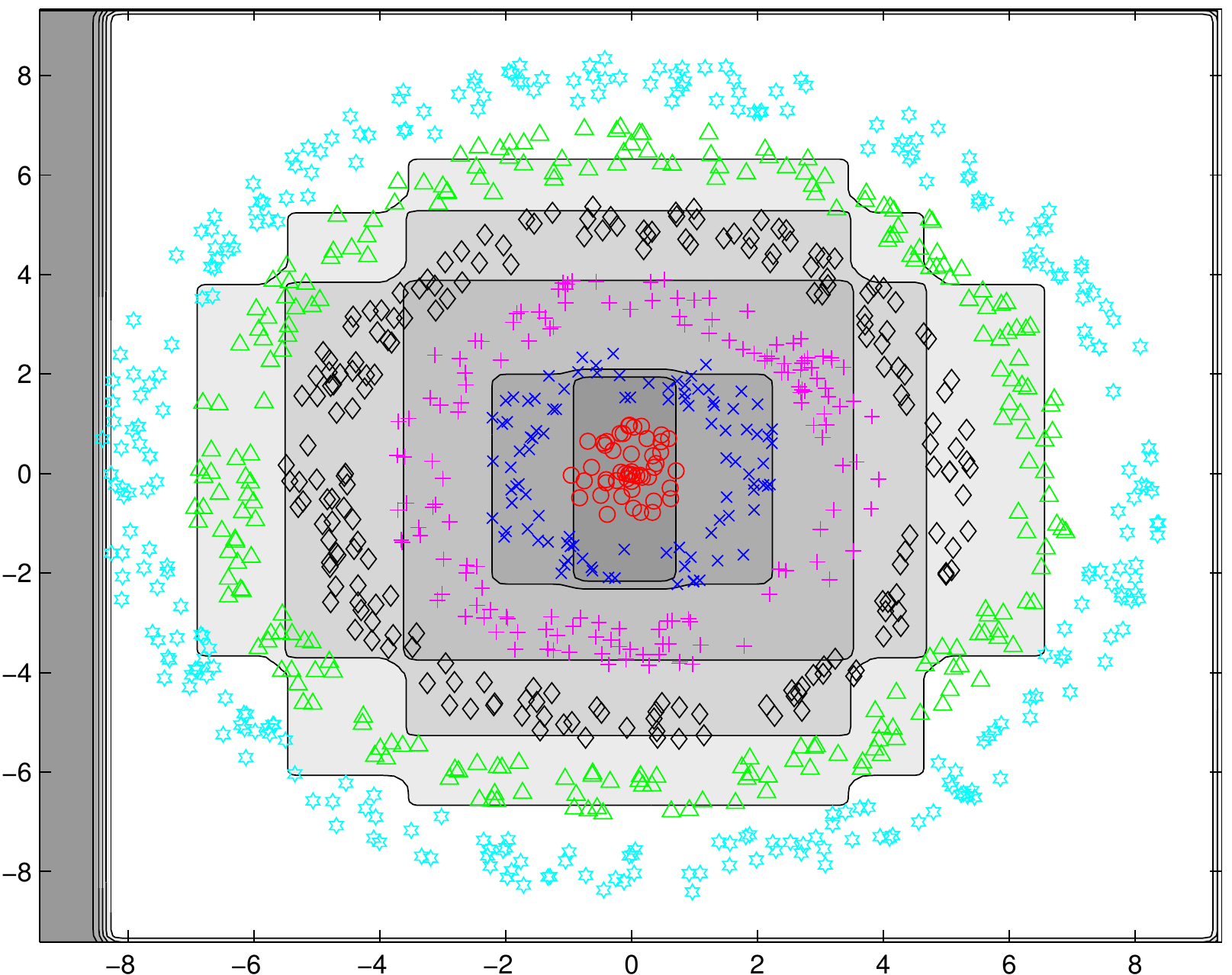}
        \includegraphics[width=0.19\textwidth,clip]{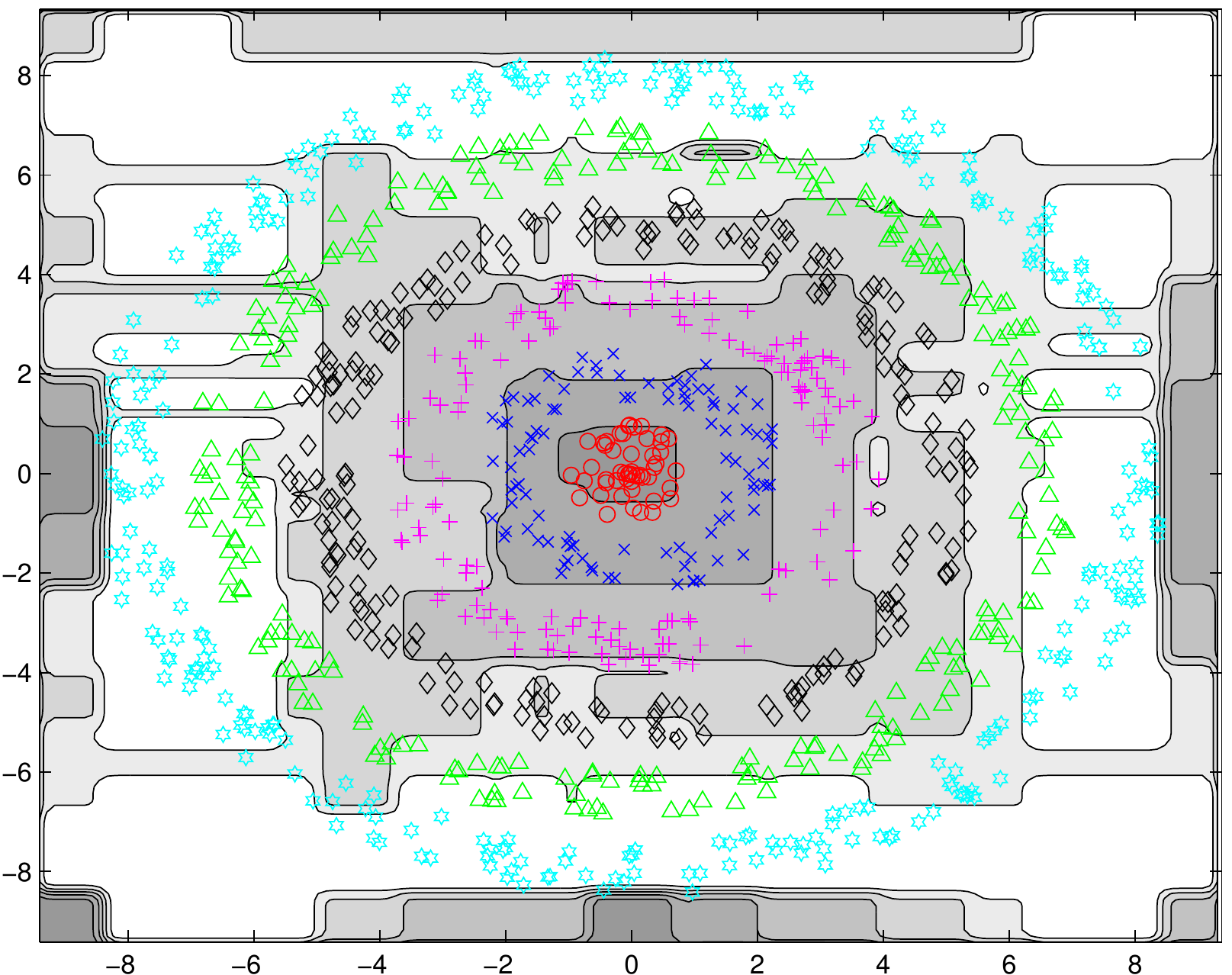}
        \includegraphics[width=0.19\textwidth,clip]{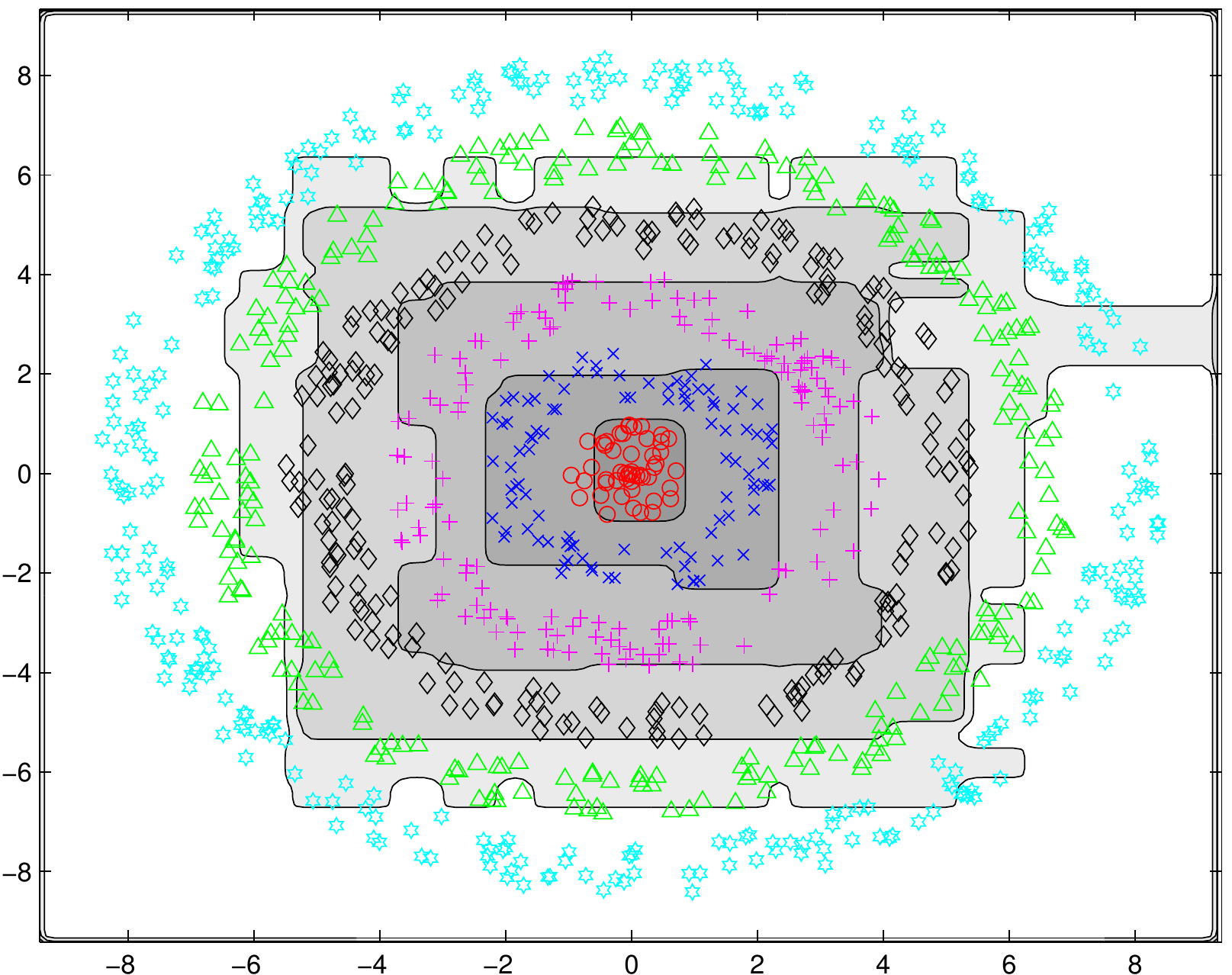}
        \includegraphics[width=0.19\textwidth,clip]{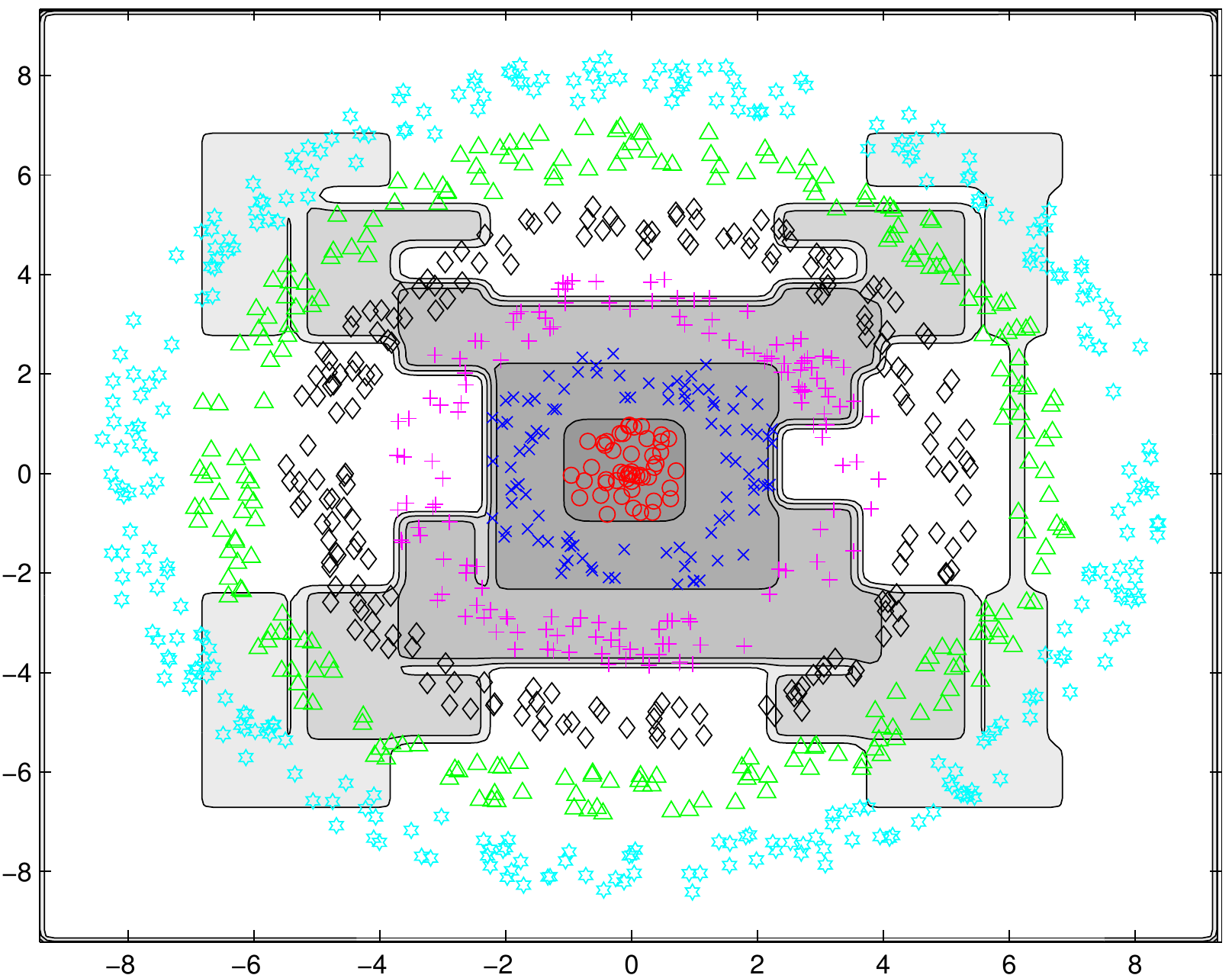}
        \includegraphics[width=0.19\textwidth,clip]{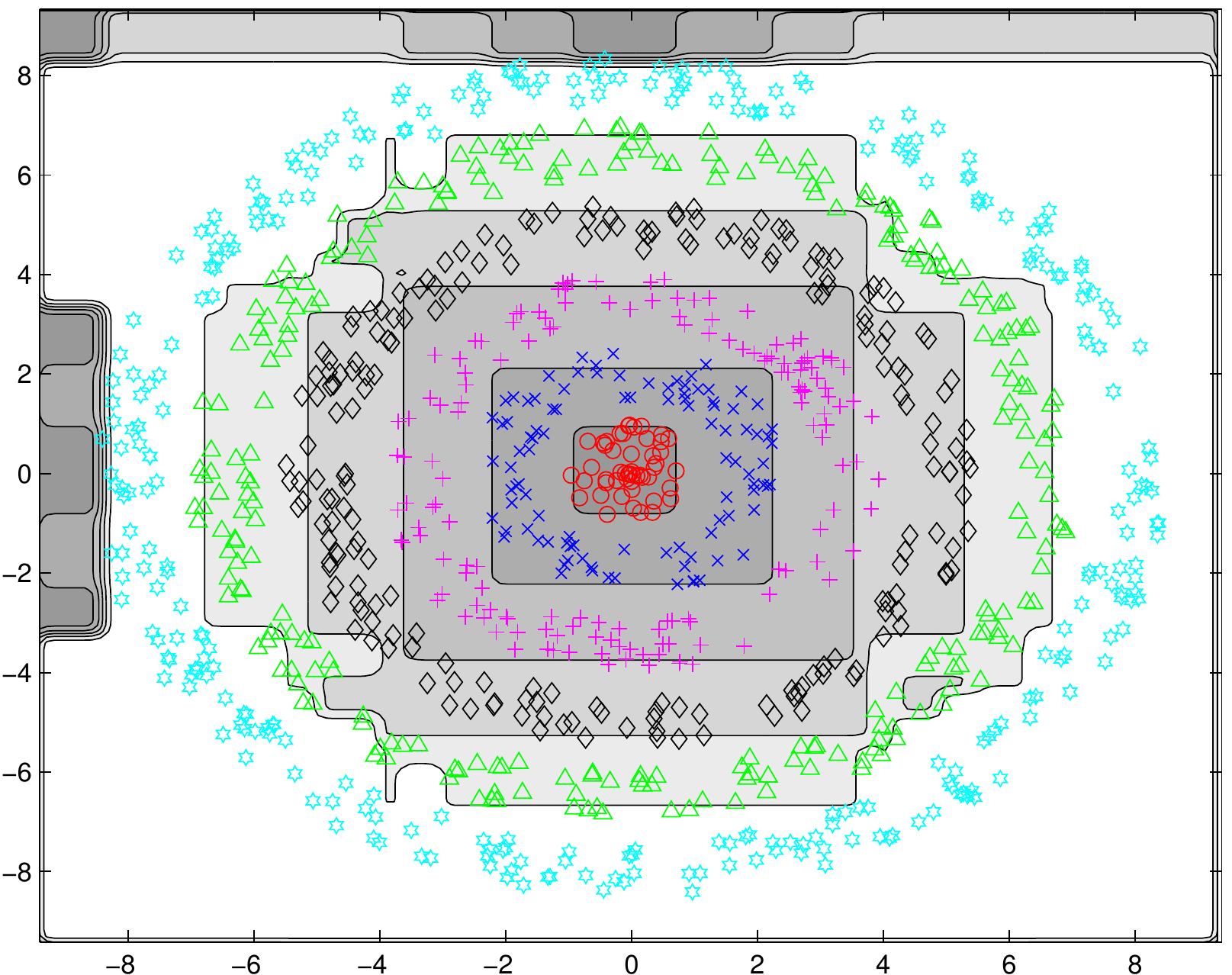}
        \includegraphics[width=0.19\textwidth,clip]{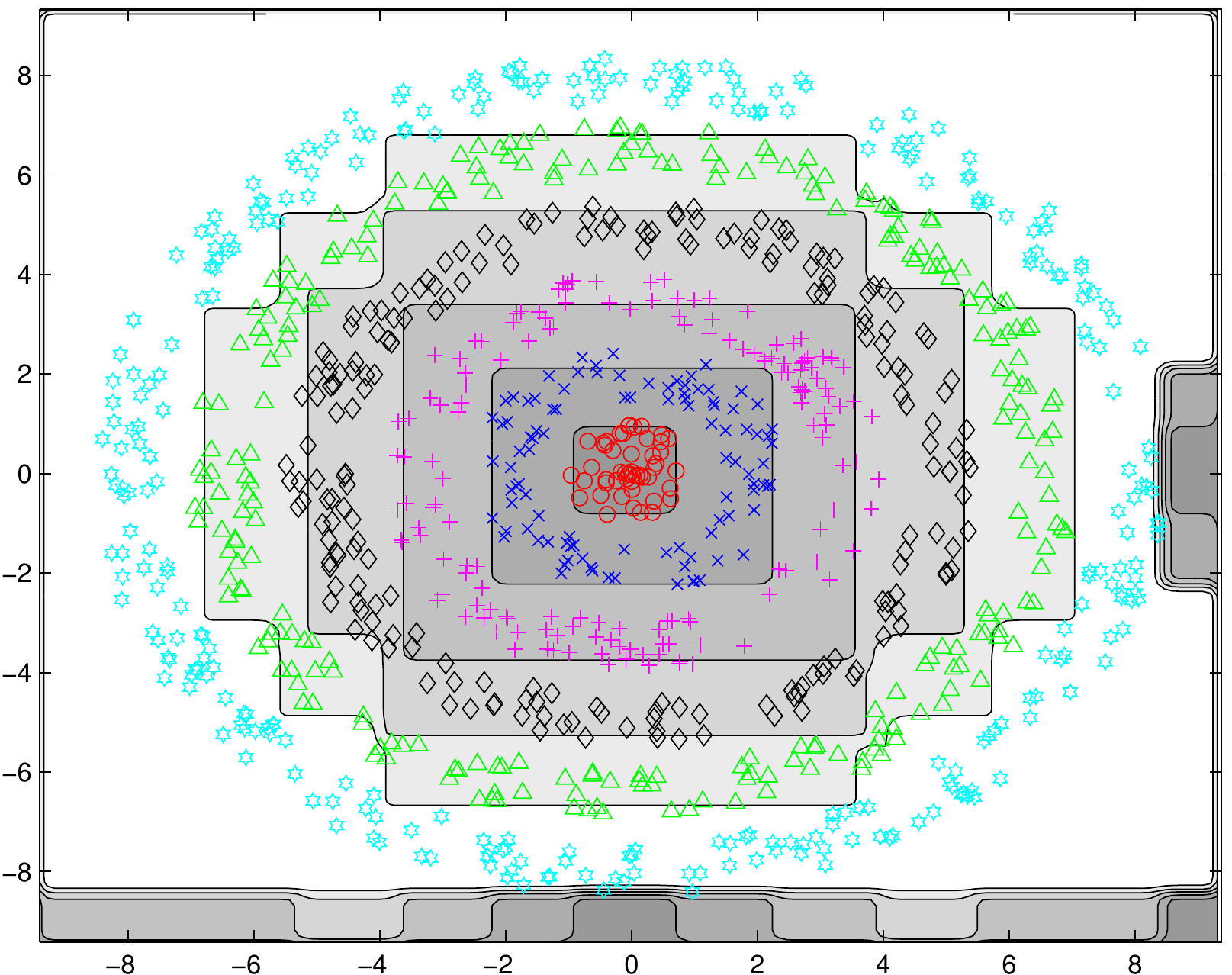}
        \includegraphics[width=0.19\textwidth,clip]{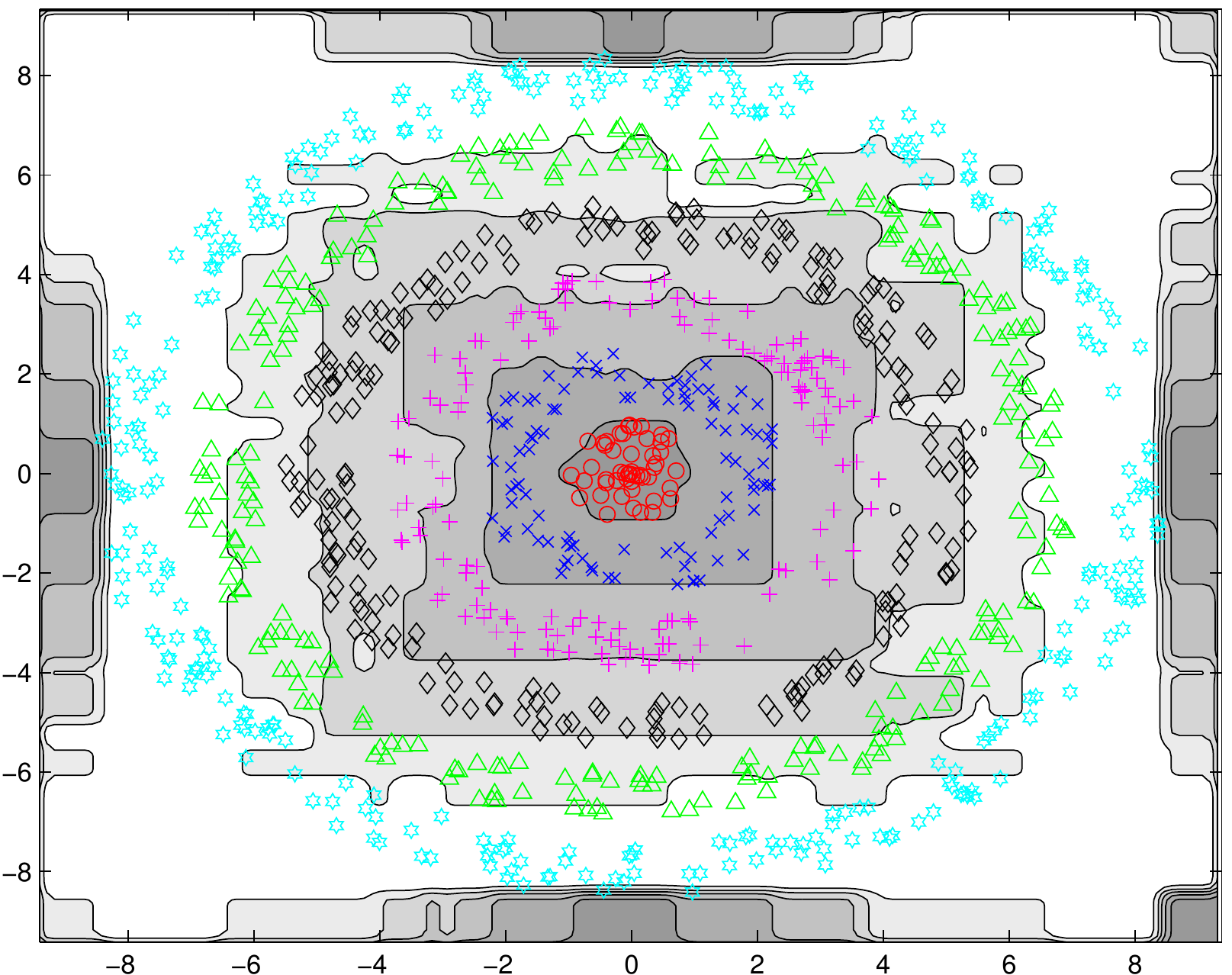}
        \includegraphics[width=0.19\textwidth,clip]{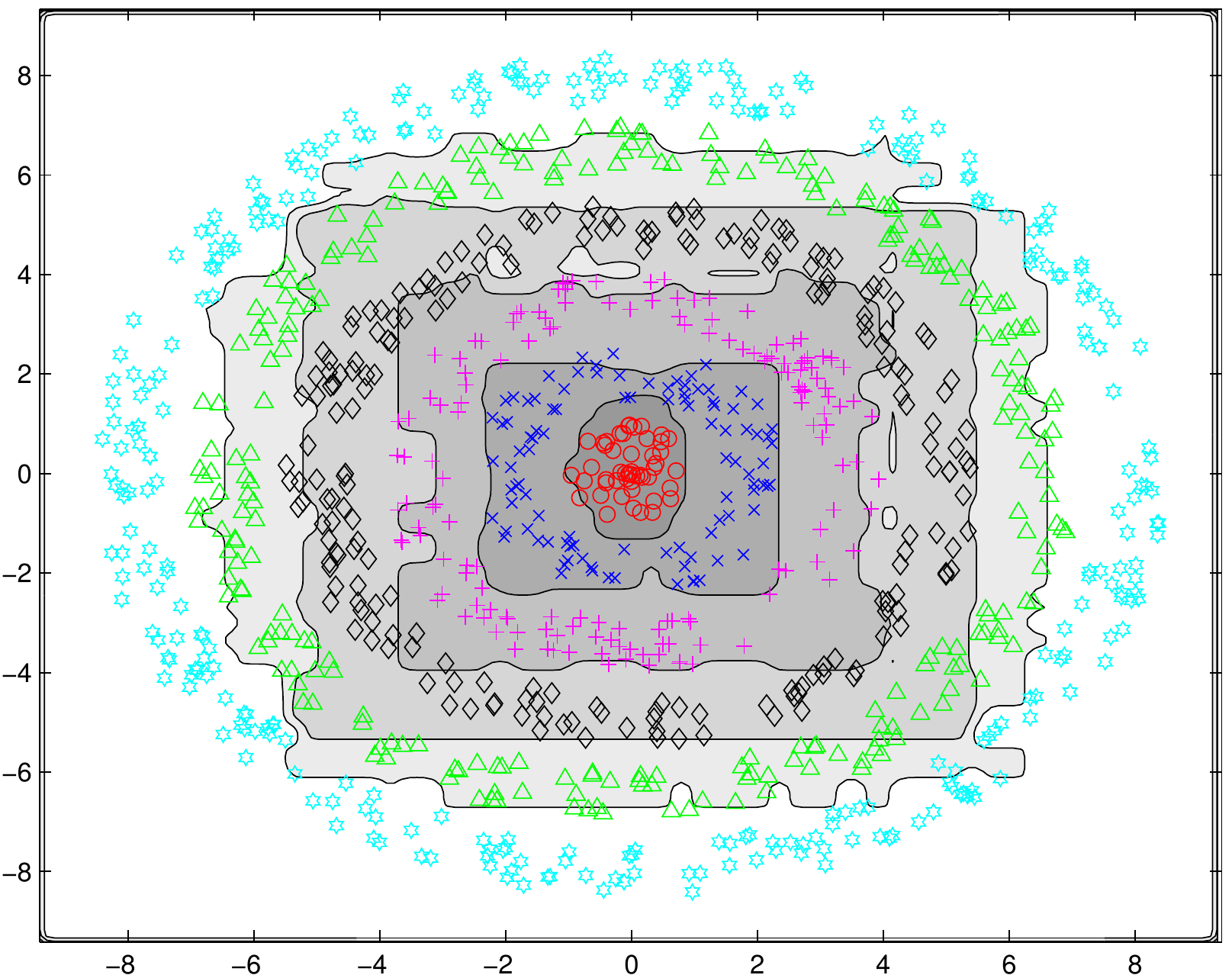}
        \includegraphics[width=0.19\textwidth,clip]{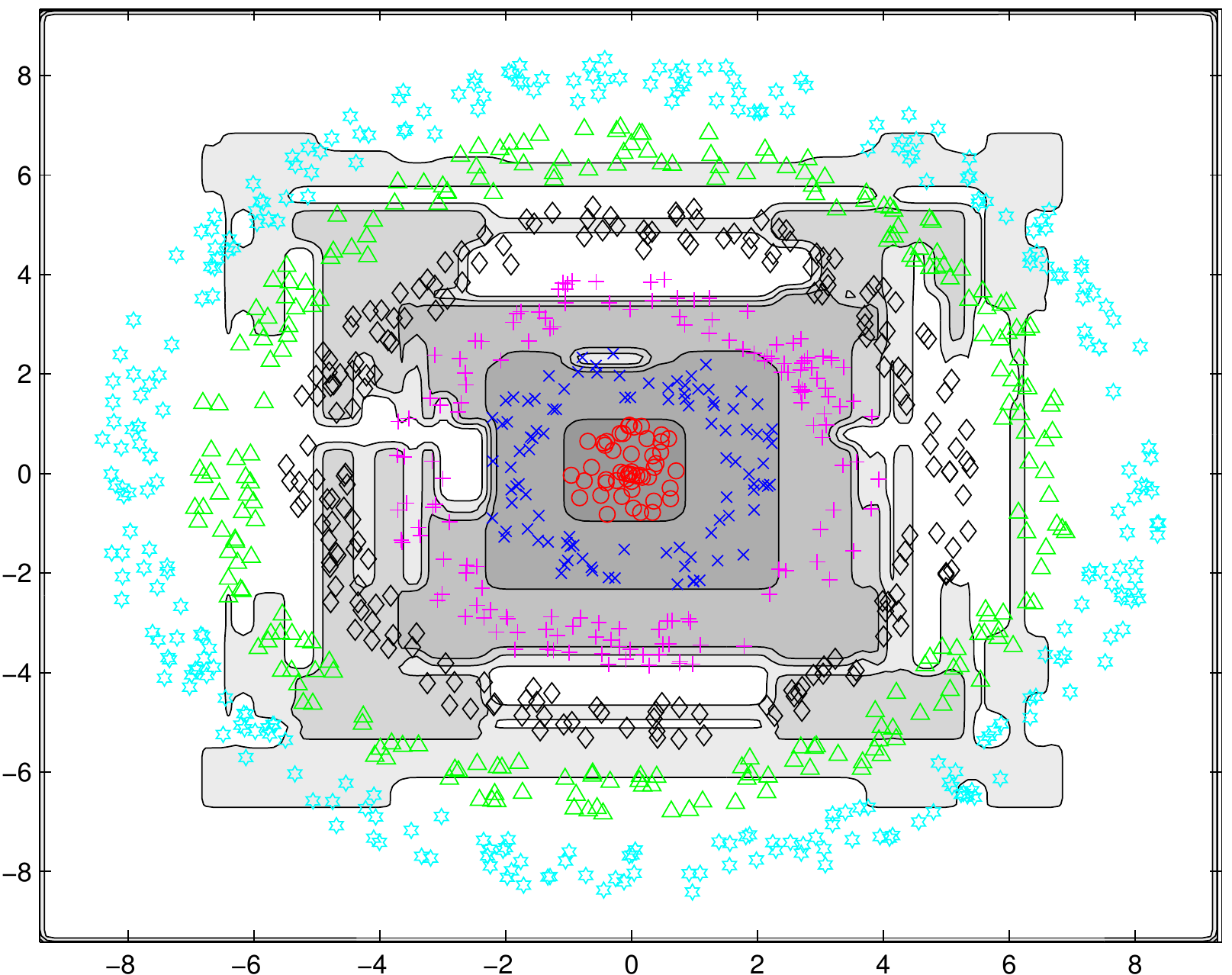}
        \includegraphics[width=0.19\textwidth,clip]{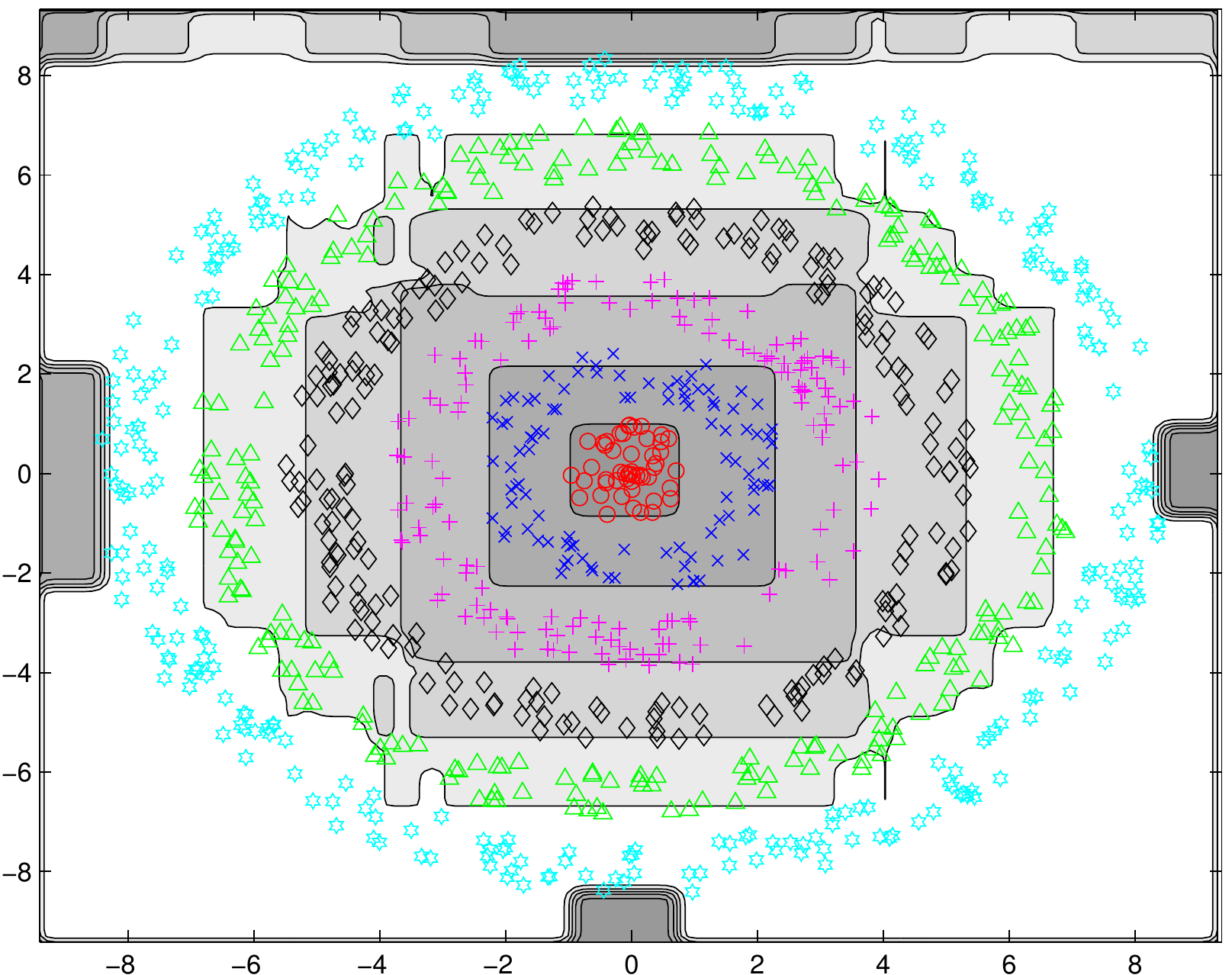}
        \includegraphics[width=0.19\textwidth,clip]{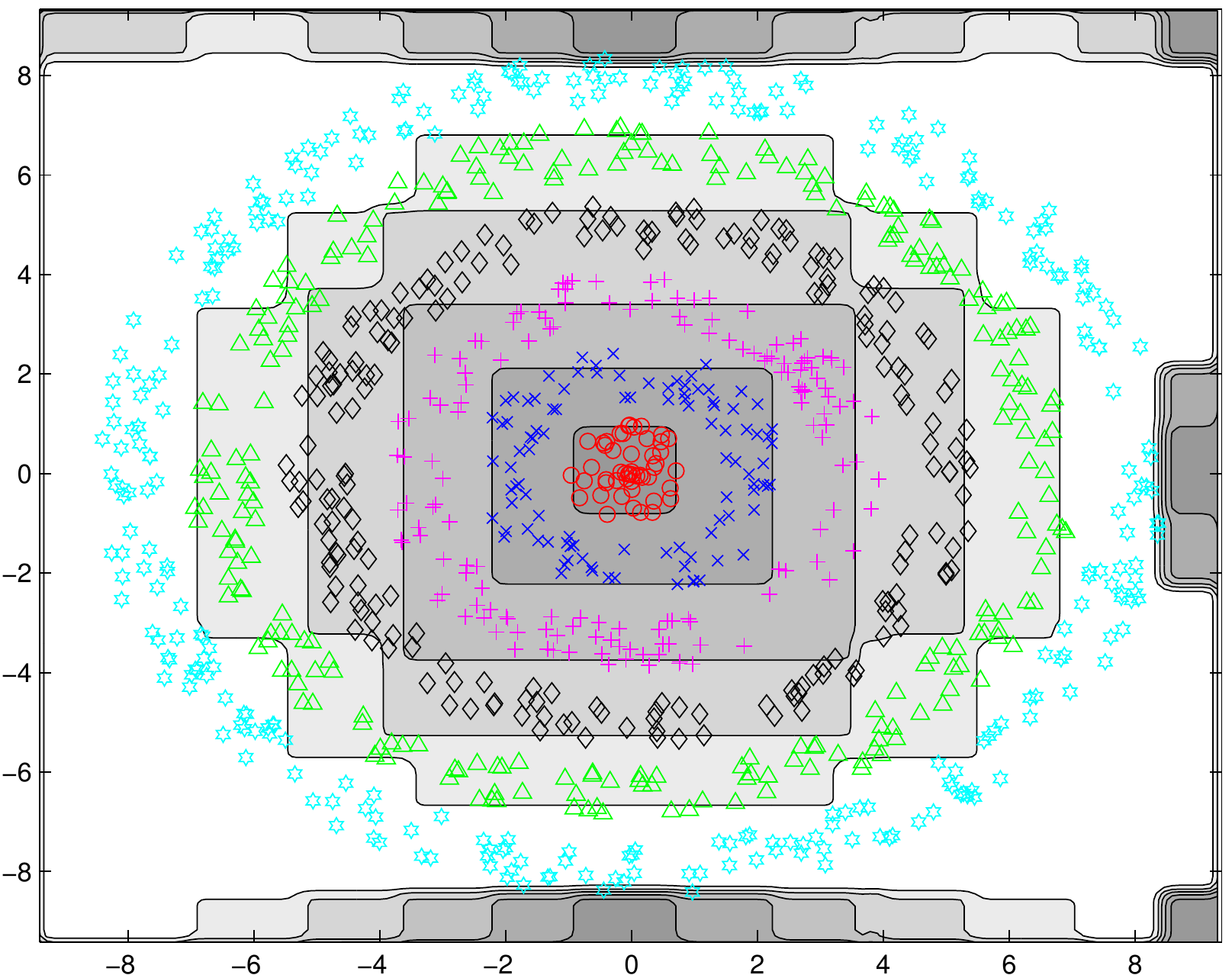}
    \caption{
        Decision boundaries on a toy data sets, with \textbf{Top row:} $20$
        weak classifiers \textbf{Middle row:} $100$ weak classifiers and
        \textbf{Bottom row:} $500$ weak classifiers.
    Note that some multi-class algorithms end up with very complicated and
    multi-modal decision boundaries.
    }
    \label{figCVPR12:toy}
    \end{figure*}

\paragraph{UCI data sets}
    The second experiment is carried out on some UCI machine learning  data sets.
    Since we are more interested in the performance of multi-class algorithms when the number of
    classes is large, we evaluate our algorithm on `segment' ($7$
    classes), `USPS' ($10$ classes),
    `pendigits' ($10$ classes), `vowel' ($11$ classes) and `isolet' ($26$ classes).
    All data instances
    from `segment' and `vowel' are used in our experiment.  For
    USPS,
    pendigits and isolet we randomly select $100$ samples from each class.  We use the original
    attributes for USPS ($256$ attributes) and isolet ($617$ attributes).
    For the rest, we increase
    the number of attributes by multiplying pairs of attributes.  Each data set is then randomly
    split into two groups: $75\%$ samples for training and $25\%$ for evaluation.
    In this experiment, we compare {\multistruct} (logistic loss) to AdaBoost.MH
    \citep{Schapire1999Improved}, AdaBoost.ECC \citep{Guruswami1999Multiclass} and GradBoost ($\ell_1/\ell_2$-regularized)
    \citep{Duchi2009Boosting}.  The regularization
    parameter is first determined by $5$-fold cross validation.

    For GradBoost, we choose the regularization parameter from $\{10^{-4}, 5
    \cdot 10^{-4}, 10^{-3}, 5 \cdot
    10^{-3}, 10^{-2}, 5 \cdot 10^{-2}, 10^{-1}, 5 \cdot 10^{-1}\}$.
    For our algorithm, we choose the regularization parameter from
    $\{ 10^{-7}, 5 \cdot 10^{-7}, 10^{-6}, 5 \cdot
    10^{-6}, 10^{-5}, 5 \cdot 10^{-5}, 10^{-4},
    5 \cdot 10^{-4}, 10^{-3} \}$.
    All experiments are
    repeated $10$ times using the same regularization parameter.
    The maximum number of
    boosting iterations is set to $500$.
    We observe that almost all the algorithms converge earlier than
    $500$ in this experiment.
    We plot the mean of test errors versus proportion of features used
    in Figure \ref{figCVPR12:figUCI}.
    These results show that our proposed approach consistently
    outperforms its competitors.  On the `segment' and `vowel' data sets
    we observe that both \multiboost and \multistruct perform similarly.
    We suspect that this is because the number of attributes
    in both data sets is quite small, and thus that there is little
    advantage to be gained through feature sharing on these data sets.
    Our approach often has the fastest convergence rate (note,
    however, that GradBoost converges faster on the USPS data sets
    but ends up with a larger test error).

\paragraph{Comparison between GradBoost and our algorithm}
    GradBoost with mixed-norm regularization \citep{Duchi2009Boosting} is
    similar to the method presented here.  The distinction, however, is
    that our method minimizes the original convex loss function rather
    than quadratic bounds on this function.
    The result is that our method is not only more effective, but also more general,
    as it can be applied not only to the logistic loss function but also to any convex loss function.
    In addition, our approach shares a similar formulation to standard boosting algorithms, \ie,
    the way we generate weak learners or update sample weights (dual variables in our
    algorithm).  The algorithm of \cite{Duchi2009Boosting} is rather heuristic and it is not
    known when the algorithm will converge.  Furthermore, GradBoost is more similar to
    FloatBoost \citep{Li2004Float} where the authors introduce a backward pruning step to remove
    less discriminative weak classifiers.  The drawback of pruning
    is 1) being heuristic and 2) a prolonged  training process.

    \begin{figure*}[t]
    \begin{center}
        \includegraphics[width=0.3\textwidth,clip]{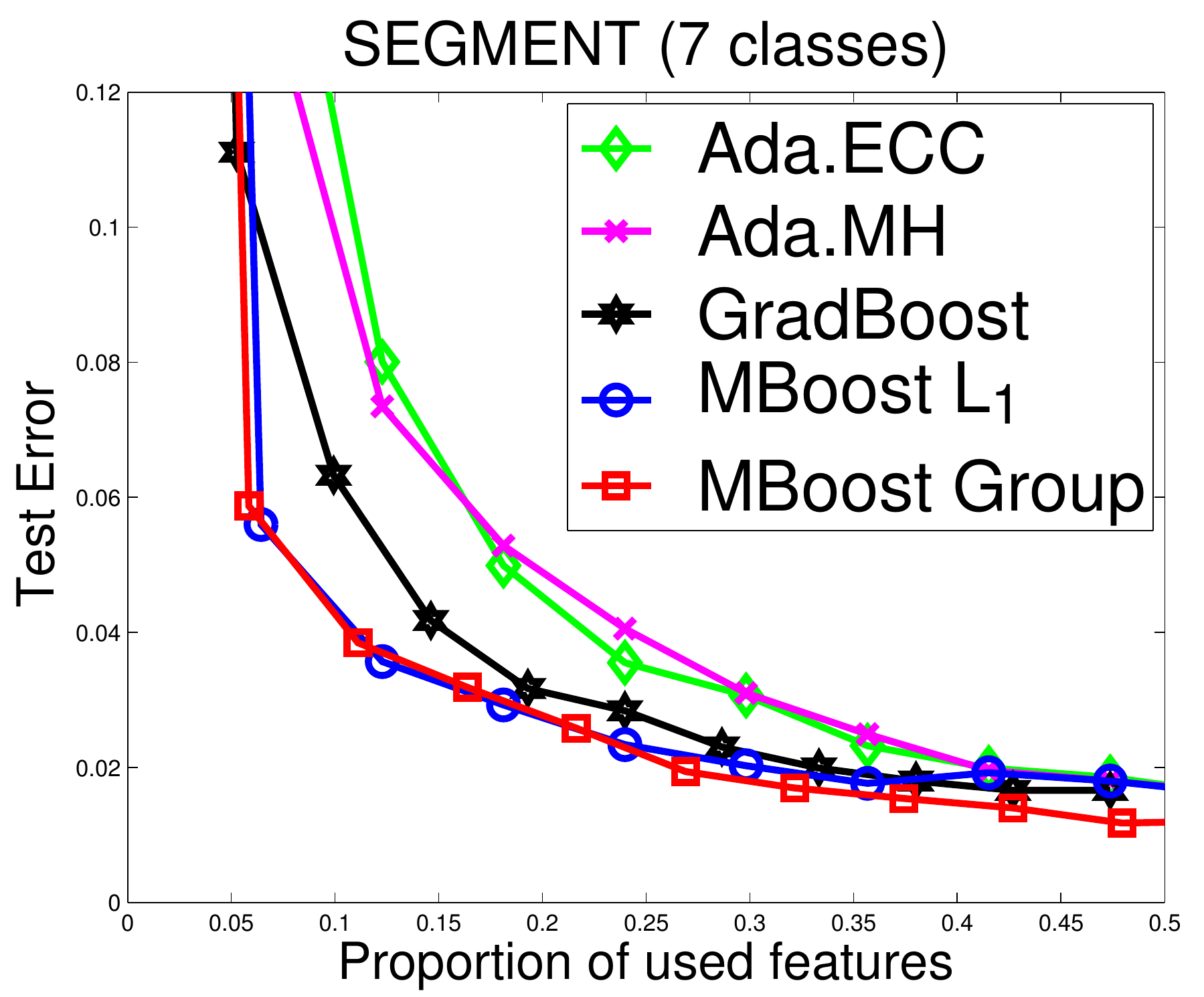}
        \includegraphics[width=0.3\textwidth,clip]{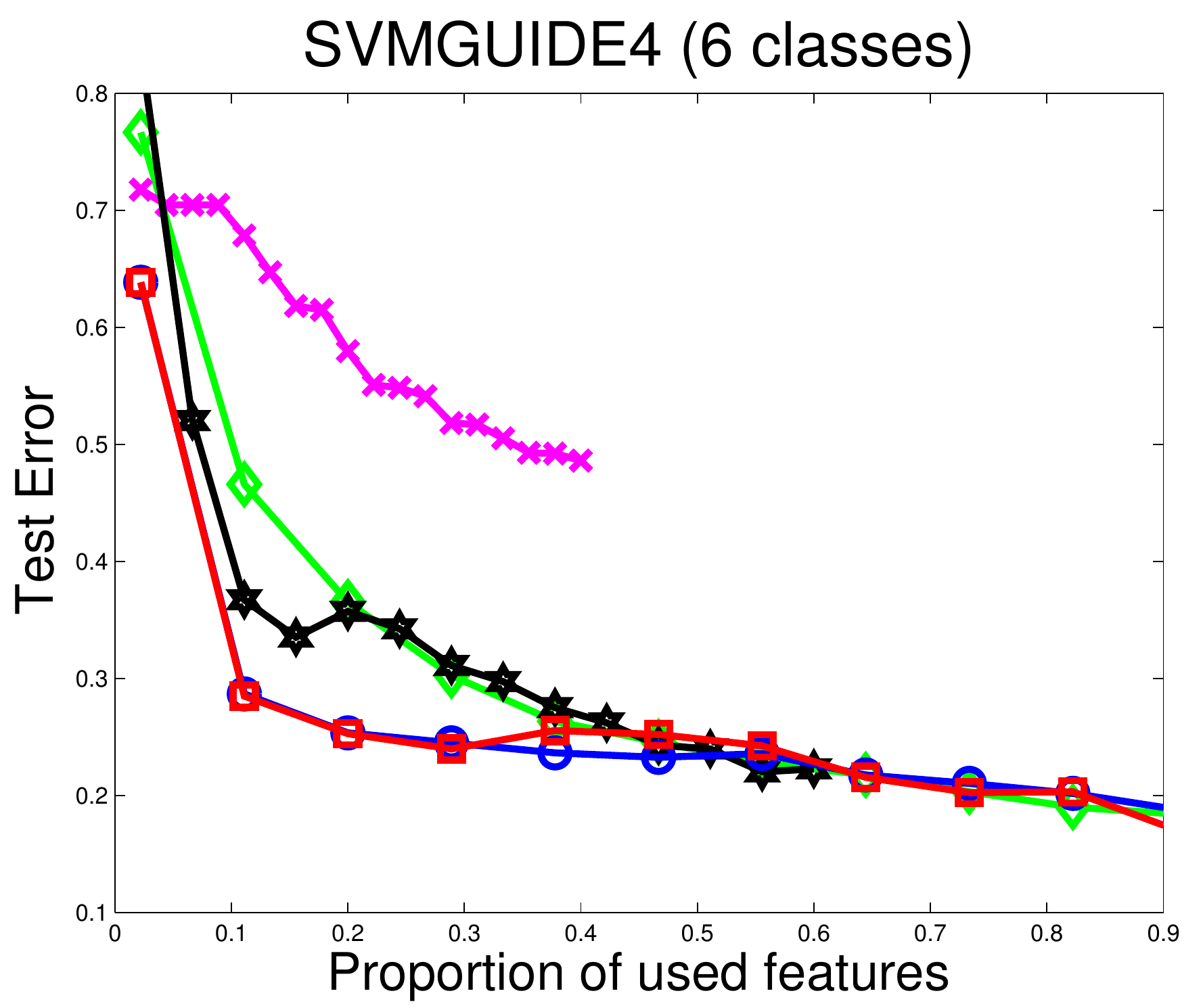}
        \includegraphics[width=0.3\textwidth,clip]{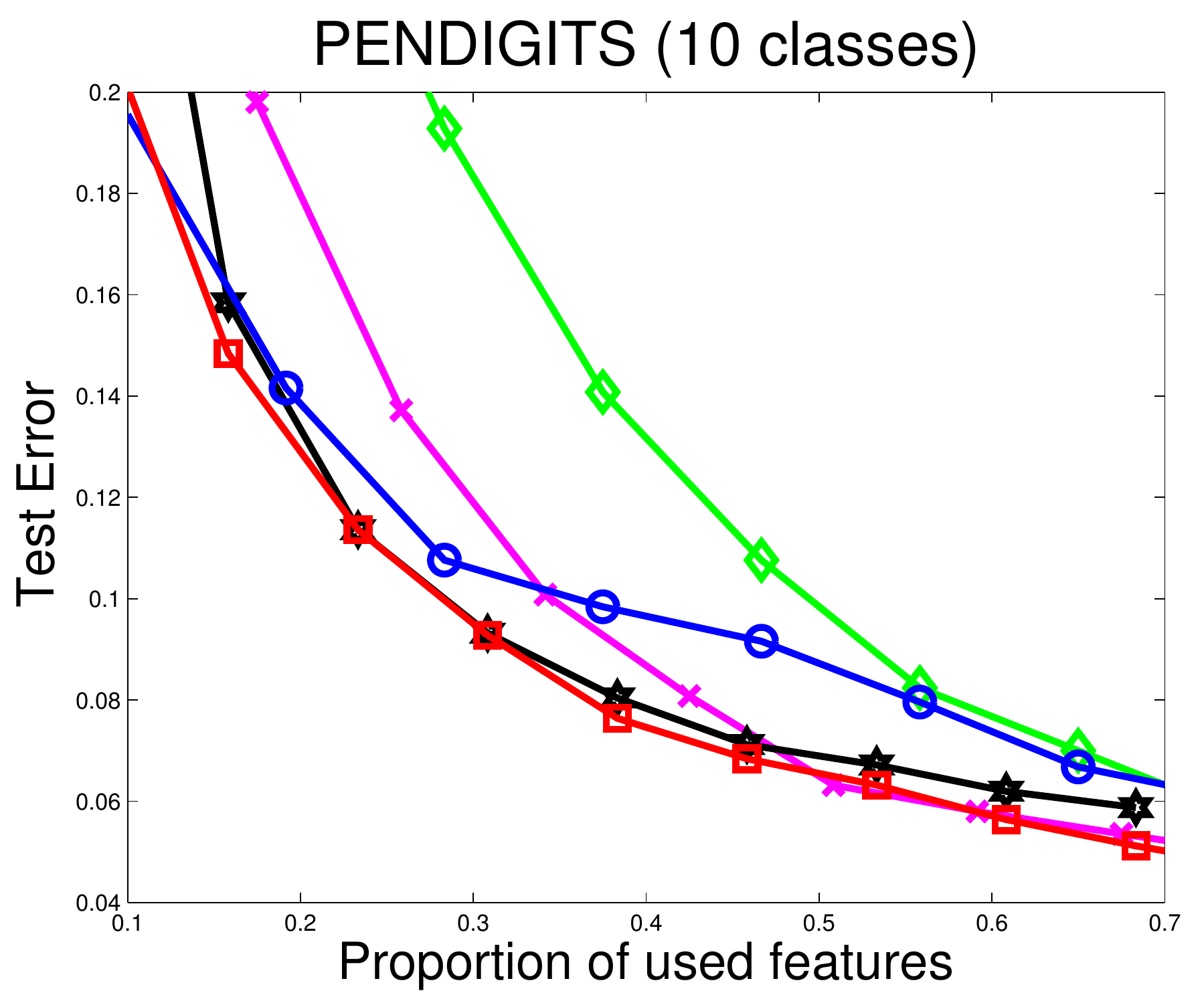}
        \includegraphics[width=0.3\textwidth,clip]{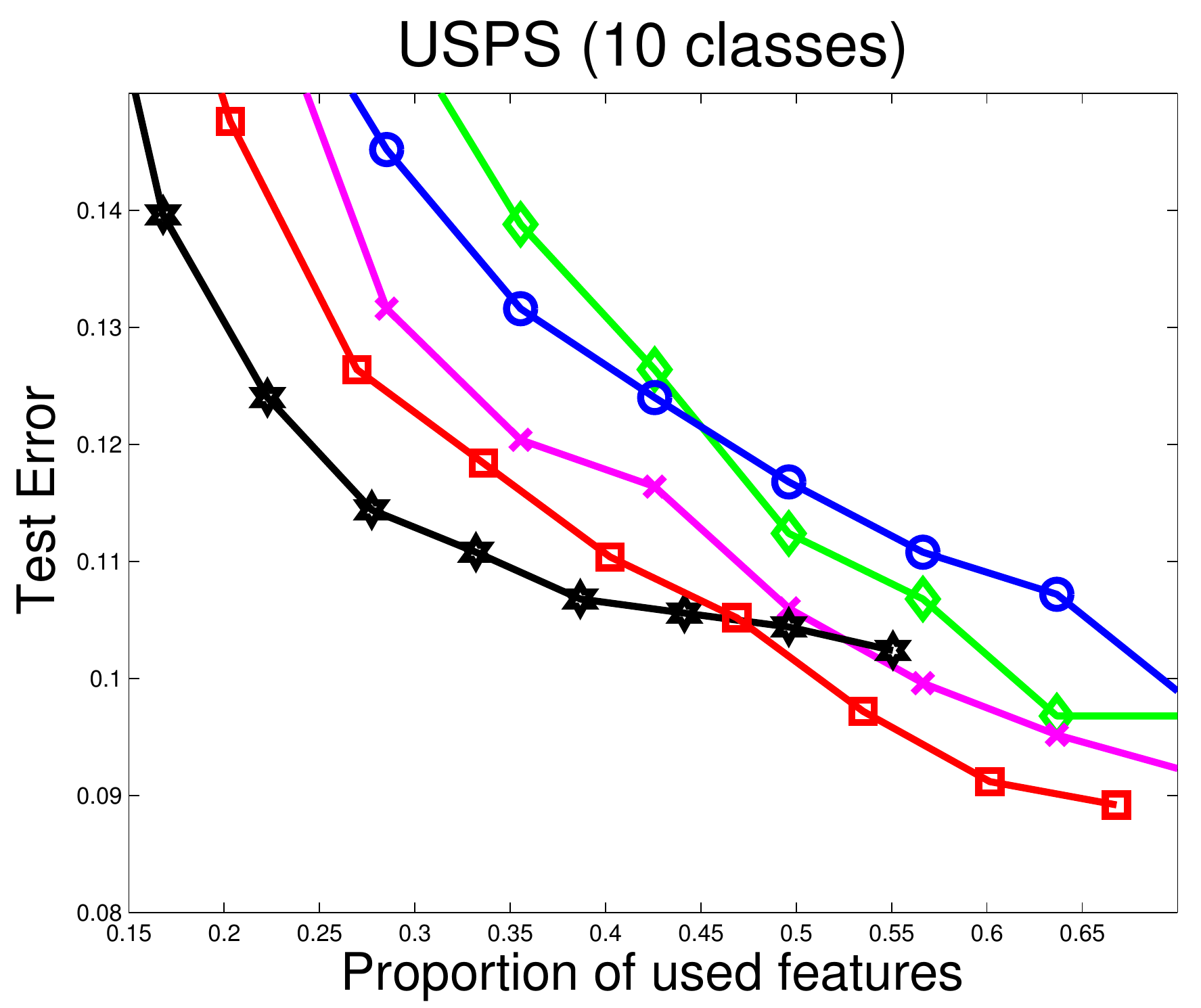}
        \includegraphics[width=0.3\textwidth,clip]{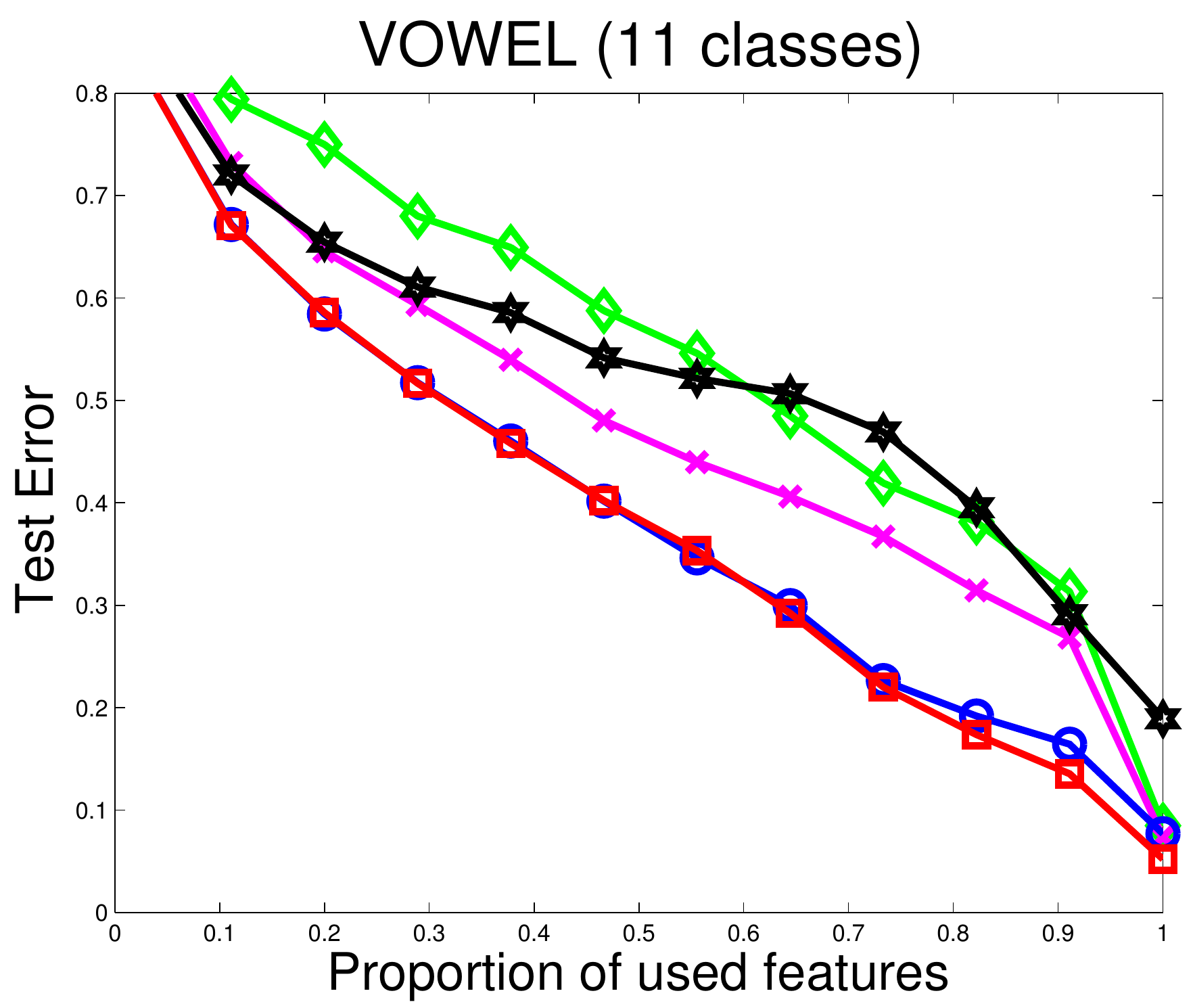}
        \includegraphics[width=0.3\textwidth,clip]{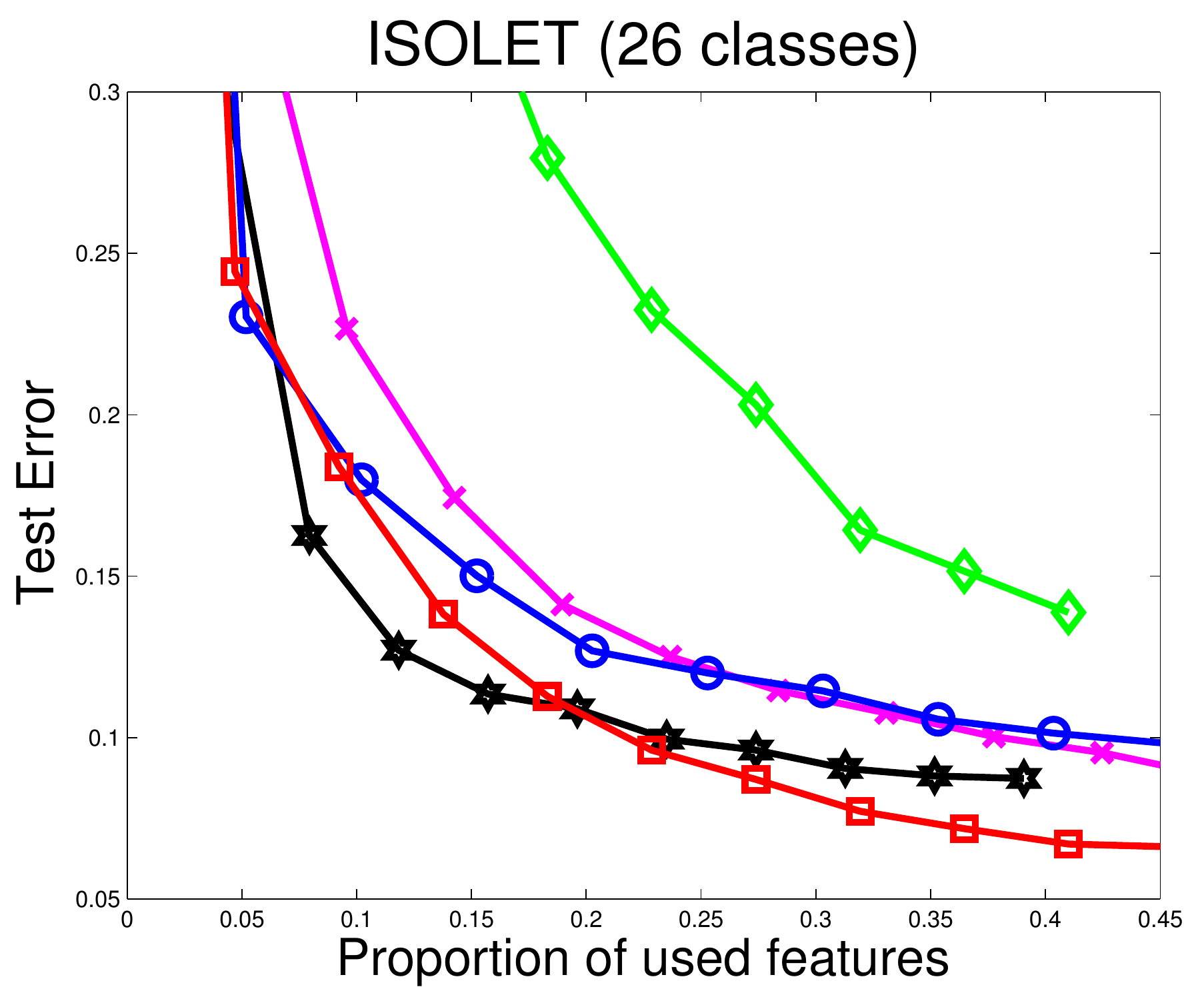}
    \end{center}
    \caption{
    The performance of our algorithm ({\multistruct}) compared with various
    boosting algorithms on several machine learning data sets.
    The horizontal axis is the fraction of used features and the vertical axis is the test error rate.
    We observe that group sparsity-based approaches (ours and GradBoost) generally converge faster than other algorithms.
    }
    \label{figCVPR12:figUCI}
    \end{figure*}

\paragraph{ABCDETC and MNIST handwritten data}
    The NEC Lab ABCDETC sets
    consist of $72$ classes (digits, letters and symbols).  For this experiment, we only use
    digits and letters ($10$ digits, $26$ lower cases and $26$ upper cases).
    We first resize the original images to a resolution of $28
    \times 28$ pixels and apply a de-skew pre-processing.  We then
    apply a spatial pyramid and extract $3$ levels of HOG features
    with $50\%$ block overlap.  The block size in each level is $4
    \times 4$, $7 \times 7$ and $14 \times 14$ pixels,
    respectively.
    Extracted HOG features from all levels are concatenated.
    In total, there are $2,172$ HOG features.
    For ABCDETC, we randomly select $5$ samples from each class as training
    sets and $120$ samples from each class as test sets.  For
    MNIST,
    we randomly select $100$ samples from each class as training
    sets and used the original test sets of $10,000$ samples.
    In this experiment, we also compare the performance of
    \multistruct\ with a fast training variant,
    {\multistructonevsall}.  All experiments are run $10$ times with $500$ boosting
    iterations and the results are briefly summarized in Table~\ref{tab:ABCDETC_error}.  From
    the table, both {\multistruct} and {\multistructonevsall} perform best compared to
    other evaluated algorithms, especially on ABCDETC test sets where the number of classes is
    large.  We observe the \ova approach to perform slightly better than {\multistruct}.
    In our work, the advantage of the \ova
    approach compared to {\multistruct} is that the training time can be further reduced by
    exploiting parallelism in \ADMM, as previously mentioned.
    Table~\ref{tab:ABCDETC_shared} illustrates the feature
    sharing property of our algorithms.
    Clearly we can see that the group sparsity regularization
    indeed encourages sharing features.

    \begin{table}[t]
      \begin{center}
          \footnotesize
      \begin{tabular}{r|cc}
      \hline
             &  MNIST  &  ABCDETC \\
      \hline
      \hline
      Ada.MH \citep{Schapire1999Improved}& \textbf{3.0} ($0.2$)& $63.4$ ($1.8$)  \\
      Ada.ECC \citep{Guruswami1999Multiclass}&  $3.1$ ($0.2$)& $70.5$ ($1.1$)  \\
      Ada.SIP \citep{Zhang2009Finding}& $4.4$ ($1.3$) & $62.7$ ($1.2$)  \\
      GradBoost \citep{Duchi2009Boosting} & $5.3$ ($0.3$) & $73.9$ ($1.3$)  \\
      \multiboost & $3.7$ ($0.2$)  & $73.2$ ($0.7$)  \\
      \multistruct &  $3.1$ ($0.2$)  & $59.1$ ($1.1$)  \\
      \multistructonevsall & \textbf{3.0} ($0.3$)  & \textbf{58.2} ($0.9$) \\
      \hline
      \end{tabular}
      \end{center}
      \vspace{-.31cm}
      \caption{Test errors (\%) of
      a few multi-class boosting methods
      on the MNIST and ABCDETC handwritten data sets.
      All experiments are run $10$ times with $500$ boosting iterations.
      The average error mean and standard deviation (in percentage) are reported.
      }
      \label{tab:ABCDETC_error}
    \end{table}

    \begin{table}
    \centering
          {
      \begin{tabular}{ r | l l l l}
      \hline
      MNIST &  `$0-3$'  &  `$4-5$' & `$6-7$' & `$8-10$'  \\
      \hline
      \hline
      \multiboost  &  $99.8\%$ & $0.2\%$ & $0\%$ & $0\%$  \\
      \multistruct &  $4.5\%$ & $48.8\%$ & $40.9\%$ & $5.8\%$  \\
      \multistructonevsall   &  $10.1\%$ & $69.9\%$ & $19.7\%$ & $0.3\%$  \\
      \hline
      ABCDETC &  `$0-15$'  &  `$16-30$' & `$31-45$' & `$46-62$'  \\
      \hline
      \hline
      \multiboost  &  $99.8\%$ & $0.2\%$ & $0\%$ & $0\%$  \\
      \multistruct &  $0\%$ & $81.3\%$ & $18.7\%$ & $0\%$  \\
      \multistructonevsall   &  $0\%$ & $65.7\%$ & $33.5\%$ & $0.7\%$  \\
      \hline
      \end{tabular}
      }
      \caption{The distribution of shared weak classifiers.
      For example,
      `$8-10$' indicates that the weak classifier is being shared among $8$ to $10$ classes.
      The table illustrates the feature sharing property of our
      algorithms, \ie, one weak classifier is being shared among multiple
      classes.
      }
      \label{tab:ABCDETC_shared}
    \end{table}

\paragraph{Scene recognition}
    In the next experiment, we compare our approach on the
    $15$-scene data set used in
    \cite{Lazebnik2006Beyond}.  The set consists of $9$ outdoor scenes and 6 indoor
    scenes.  There are $4,485$ images in total. For each run, the available data are
    randomly split into a training set and a test set based on published protocols.
    This is repeated $5$ times and the average accuracy is reported.  In each train/test
    split, a visual codebook is generated using only training images.  Both training and
    test images are then transformed into histograms of code words.
    We use CENTRIST of \cite{Wu2011CENTRIST} as our feature descriptors.  $200$ visual code words
    are built using the histogram intersection kernel (HIK), which has
    been
    shown to outperform $k$-means and
    $k$-median \citep{Wu2011CENTRIST}. We represent each image in a spatial hierarchy
    manner \citep{Bosch2008Scene}.  Each image consists of $31$ sub-windows.  An image is
    represented by the concatenation of histograms of code words from all $31$ sub-windows.
    Hence, in total there are $6,200$ dimensional histogram.

    Figure~\ref{figCVPR12:scene} shows the average classification errors.
    We observe that both {\multistruct} and \multiboost converge quickly in
    the beginning.
    However, \multistruct has a better overall convergence rate.
    We also observe that
    both ({\multistruct} and  {\multistructonevsall}), have the lowest test error
    compared to other algorithms evaluated.  We also apply a multi-class SVM to the above data
    set using the LIBSVM package \citep{Chang2011LIBSVM} and report the recognition results in
    Table~\ref{tab:scene}.  SVM with $6,200$ features achieves an average
    accuracy of $76.30\%$ (linear) and $81.47\%$ (non-linear).  Our results indicate that both
    proposed approaches achieve a comparable accuracy to non-linear SVM while requiring less
    number of features ($77.8\%$ accuracy for {\multistruct} with
    $1000$ features and $79.2\%$
    accuracy for {\multistructonevsall}).

    \begin{table}
      \centering
      \scalebox{1}
      {
      \begin{tabular}{r|c|c}
      \hline
       methods      &  $\#$ features used  &  accuracy ($\%$) \\
      \hline
      \hline
      SAMME$^\dag$ \citep{Zhu2009Multi} & $1000$ & $70.9$ ($0.40$) \\
      JointBoost$^\dag$ \citep{Torralba2007Sharing} & $1000$ & $72.2$ ($0.70$) \\
      \multiboost & $1000$ & $76.0$ ($0.48$) \\
      AdaBoost.SIP \citep{Zhang2009Finding}& $1000$ & $75.7$ ($0.10$)  \\
      AdaBoost.ECC \citep{Guruswami1999Multiclass}& $1000$ & $76.5$ ($0.67$) \\
      AdaBoost.MH \citep{Schapire1999Improved}& $1000$ & $77.6$ ($0.59$) \\
      {\multistruct}  &  $1000$ & $77.8$ ($0.77$)\\
      \multistructonevsall & \textbf{1000} & \textbf{79.2} ($0.82$)\\
      Linear SVM & $6200$ & $76.3$ ($0.88$) \\
      Nonlinear SVM (HIK) & \textbf{6200} &  \textbf{81.4} ($0.60$)\\
      \hline
      \end{tabular}
      }
      \caption{Recognition rate of various algorithms on Scene$15$ data sets.
      All experiments are run $5$ times.
      The average accuracy mean and standard deviation (in percentage) are reported.
      Results marked by $\dag$
      were reported in \cite{Zhang2009Finding}.
      }
      \label{tab:scene}
    \end{table}

    \begin{figure}[t]
    \centering
        \includegraphics[width=0.45\textwidth,clip]{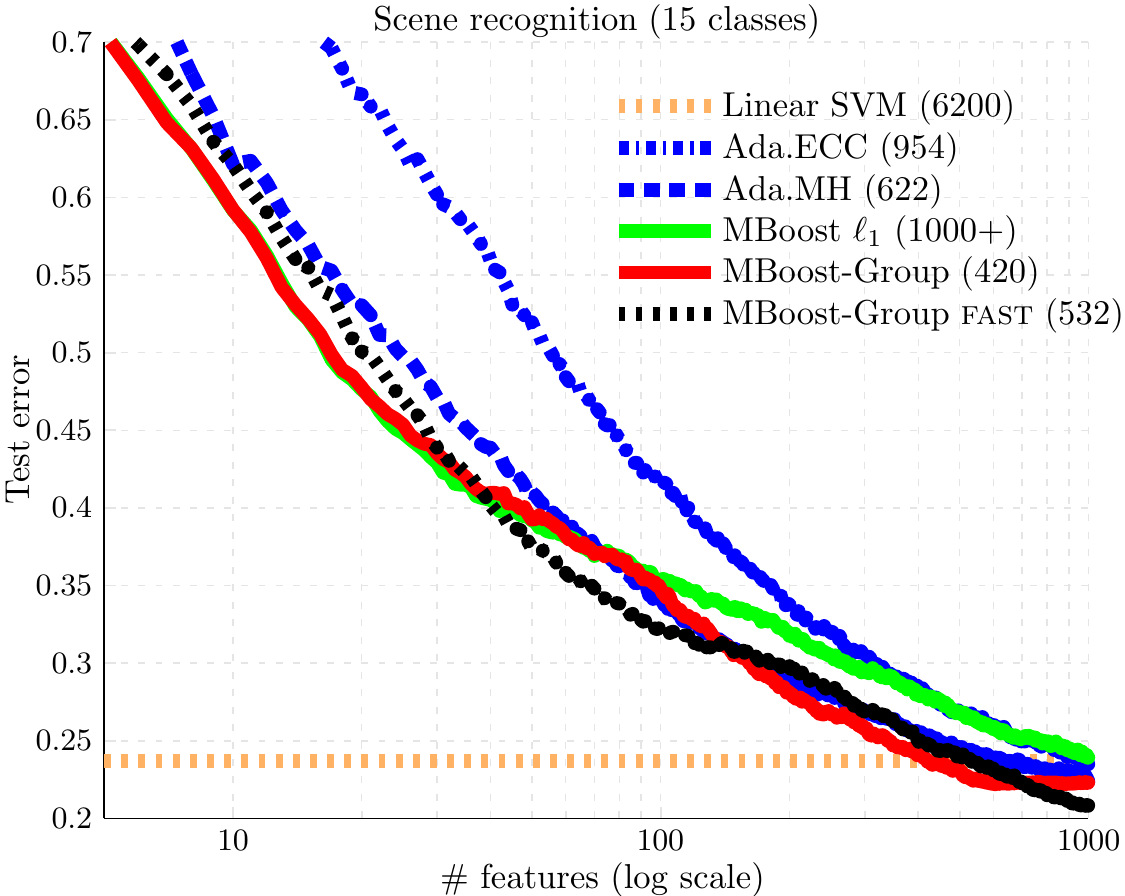}
    \caption{
    Performance of different classifiers on the scene recognition data set.
    We also report the number of features required to achieve similar results
    to linear multi-class SVM.
    Both of our methods ({\multistruct} and
    \multistructonevsall) outperform other evaluated boosting algorithms.
    }
    \label{figCVPR12:scene}
    \end{figure}

\paragraph{Traffic sign recognition}
    We evaluate our approach on the recent German traffic sign recognition
    benchmark\footnote{\url{http://benchmark.ini.rub.de/}}.  Data sets consist of $43$
    classes with more than $50,000$ images in total.  We randomly select $100$ samples from
    each class to train our classifier.  We use the provided test set to evaluate the
    performance of our classifiers ($12,569$ images).  All training images are scaled to $40
    \times 40$ pixels using bilinear interpolation.  Three different types of pre-computed
    HOG features are provided ($6,052$ features).  We
    combine all three types together.  We also make use of histogram of hue values ($256$
    bins).
    Hence, there is a total of $6,308$ features.  The results of different classifiers are
    shown in Figure~\ref{figCVPR12:traffic}.  Our proposed classifier outperforms other evaluated
    classifiers.  As a baseline, we train a multi-class SVM using LIBSVM
    \citep{Chang2011LIBSVM}.  SVM achieves $93.05\%$ (using $6,308$ features) while our
    classifier achieves $95.62\%$ for {\multistruct} and $95.42\%$ for {\multistructonevsall} with a
    much smaller set of features ($500$ features).  Note that
    an overfitting behavior is
    observed for \multiboost.

    \begin{figure}[t]
    \centering
        \includegraphics[width=0.45\textwidth,clip]{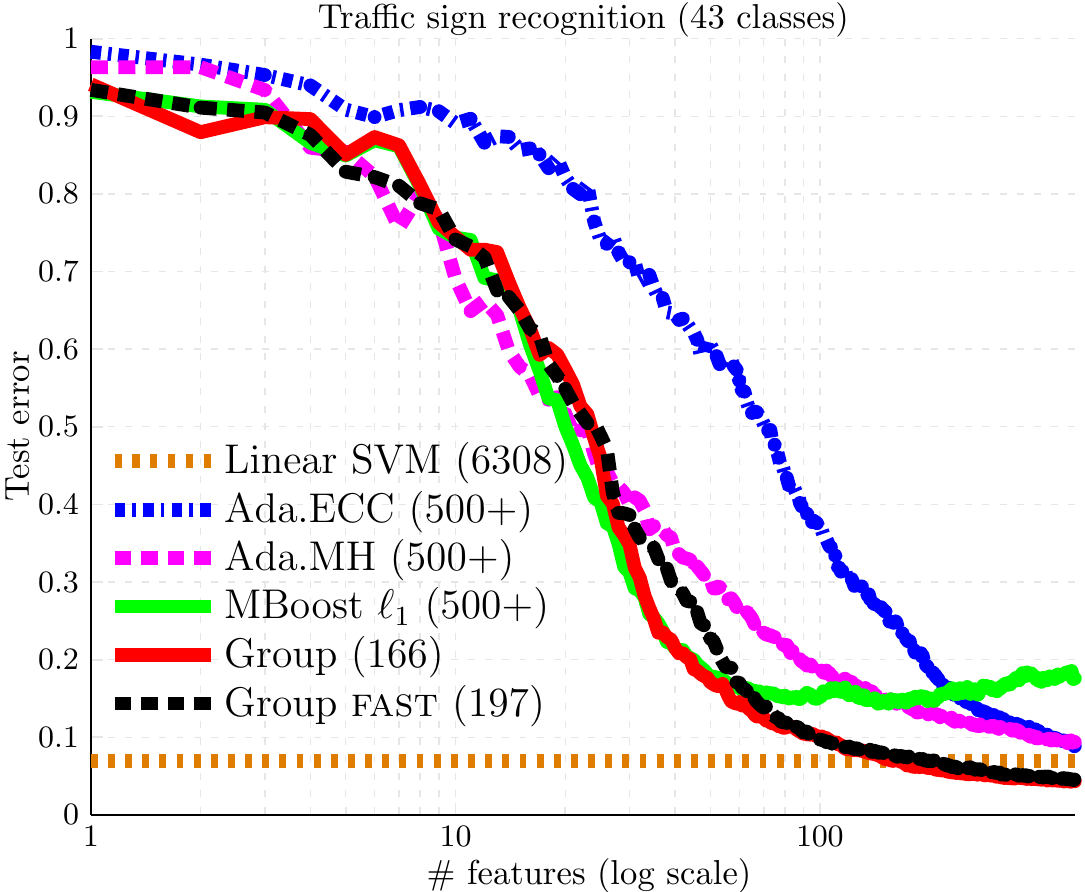}
    \caption{
    Performance of different classifiers on traffic sign recognition data sets.
    We also report the number of features needed to achieve a similar
    accuracy to the linear SVM. Both of our methods outperform other
    multi-class methods in terms of the test error.
    }
    \label{figCVPR12:traffic}
    \end{figure}

\section{Conclusion}

\label{sec:conc}

    In this work, we have presented a direct formulation for
    multi-class boosting.
    We derive the Lagrange dual of the formulated
        primal optimization problem.
    Based on the dual problem, we are able to
    design fully-corrective boosting
    using the column generation technique.
    At each iteration, all weak classifiers' weights are updated.
    We then generalize our approach and propose a new feature-sharing
    multi-class boosting method.
    The proposed boosting is based on the primal-dual view of the group
    sparsity regularized optimization.
    Various experiments on a few different data sets demonstrate
    that our direct
    multi-class boosting achieves competitive  test
    accuracy compared with other existing multi-class boosting.

    Future research topics include how to efficiently
    solve the convex optimization problems of the proposed
    multi-class boosting. Conventional multi-class boosting
    do not need to solve convex optimization at each step and
    thus much faster.
    We also want to explore the possibility of structural
    learning with boosting by extending the proposed
    multi-class boosting framework.


\end{document}